\documentclass[twoside,11pt]{article}

\usepackage{jmlr2e}
\usepackage{hyperref}
\usepackage{courier}
\usepackage{epsfig}
\usepackage{graphicx}
\usepackage{amsmath}
\usepackage{mathtools}
\usepackage{amssymb}
\usepackage{verbatim}
\usepackage{booktabs}
\usepackage{algpseudocode}
\usepackage{amscd}
\usepackage{times}
\usepackage{subfigure} 
\usepackage{algorithm}
\usepackage{color}
\usepackage{enumerate}
\usepackage{tikz}
\usepackage{multirow}

\algrenewcommand\textproc{}

\newtheorem{thm}{Theorem}

\newtheorem{lem}[thm]{Lemma}

\newtheorem{dfn}[thm]{Definition}

\def\A{\mathcal{A}}
\def\bA{\boldsymbol{A}}
\def\bB{\boldsymbol{B}}
\def\C{\boldsymbol{\mathcal{C}}}
\def\bC{\boldsymbol{C}}

\def\bI{\boldsymbol{I}}
\def\bL{\boldsymbol{L}}

\def\N{\mathcal{N}}
\def\bO{\boldsymbol{\mathcal{O}}}

\def\Y{\mathcal{Y}}
\def\bY{\boldsymbol{\Y}}
\def\bX{\boldsymbol{\mathcal{X}}}

\def\cS{\mathbb{S}}
\def\Sp{\mathbb{S}}
\def\V{\mathcal{V}}
\def\H{\mathcal{H}}
\def\I{\boldsymbol{I}}
\def\J{\mathcal{J}}

\def\bP{\boldsymbol{P}}

\def\Re{\mathbb{R}}
\def\R{\mathcal{R}}

\def\U{\mathcal{U}}
\def\bU{\boldsymbol{U}}

\def\0{\boldsymbol{0}}
\def\1{\boldsymbol{1}}

\def\b{\boldsymbol{b}}

\def\e{\boldsymbol{e}}

\def\bzeta{\boldsymbol{\zeta}}
\def\boldeta{\boldsymbol{\eta}}
\def\bxi{\boldsymbol{\xi}}

\def\bchi{\boldsymbol{\chi}}
\def\f{\boldsymbol{f}}
\def\h{\boldsymbol{h}}
\def\bn{\boldsymbol{n}}

\def\y{\boldsymbol{y}}
\def\x{\boldsymbol{x}}

\def\w{\boldsymbol{w}}

\def\e{\boldsymbol{e}}

\def\transpose{\top}
\def\st{\text{s.t.}}

\DeclareMathOperator*{\Sgn}{Sgn}
\DeclareMathOperator*{\Span}{Span}
\DeclareMathOperator*{\Sign}{Sign}

\DeclareMathOperator*{\argmin}{argmin}

\DeclareMathOperator*{\Trace}{Trace}
\DeclareMathOperator*{\Card}{Card}

\newcommand{\myparagraph}[1]{\noindent\textbf{#1.}}

\newcommand{\ra}[1]{\renewcommand{\arraystretch}{#1}}

\begin{document}
	
\title{Hyperplane Clustering Via Dual Principal Component Pursuit}

\author{Manolis C. Tsakiris and Ren\'e Vidal \email m.tsakiris,rvidal@jhu.edu \\
	Center for Imaging Science \\
	Johns Hopkins University \\
	Baltimore, MD, 21218, USA}

\editor{TBD}

\maketitle

\begin{abstract}
State-of-the-art methods for clustering data drawn from a union of subspaces are based on sparse and low-rank representation theory. Existing results guaranteeing the correctness of such methods require the dimension of the subspaces to be small relative to the dimension of the ambient space. When this assumption is violated, as is, for example, in the case of hyperplanes, existing methods are either computationally too intense (e.g., algebraic methods) or lack theoretical support (e.g., K-hyperplanes or RANSAC). The main theoretical contribution of this paper is to extend the theoretical analysis of a recently proposed single subspace learning algorithm, called Dual Principal Component Pursuit (DPCP), to the case where the data are drawn from of a union of hyperplanes. 
To gain insight into the expected properties of the non-convex $\ell_1$ problem associated with DPCP (\emph{discrete problem}), we develop a geometric analysis of a closely related \emph{continuous} optimization problem. Then transferring this analysis to the discrete problem, our results state that as long as the hyperplanes are sufficiently separated, the dominant hyperplane is sufficiently dominant and the points are uniformly distributed (in a deterministic sense) inside their associated hyperplanes, then the non-convex DPCP problem has a unique (up to sign) global solution, equal to the normal vector of the dominant hyperplane. This suggests a sequential 
hyperplane learning algorithm, which first learns the dominant hyperplane by applying DPCP to the data. In order to avoid hard thresholding of the points which is sensitive to the choice of the thresholding parameter, all points are weighted according to their distance to that hyperplane, and a second hyperplane is computed by applying DPCP to the weighted data, and so on. Experiments on
corrupted synthetic data show that this DPCP-based sequential algorithm dramatically improves over similar sequential algorithms, which learn the dominant hyperplane via state-of-the-art single subspace learning methods (e.g., with RANSAC or REAPER). Finally, 3D plane clustering experiments on real 3D point clouds show that a K-Hyperplanes DPCP-based scheme, which computes the normal vector of each cluster via DPCP, instead of the classic SVD, is very competitive to state-of-the-art approaches (e.g., RANSAC or SVD-based K-Hyperplanes).
\end{abstract}

\section{Introduction}\label{section:Introduction} 

\indent \myparagraph{Subspace Clustering} Over the past fifteen years the model of a union of linear subspaces, also called a \emph{subspace arrangement} \citep{Derksen:JPAA07}, has gained significant popularity in pattern recognition and computer vision \citep{Vidal:SPM11-SC}, often replacing the classical model of a single linear subspace, associated to the well-known Principal Component Analysis (PCA) \citep{Hotelling-1933,Pearson-1901,Jolliffe-2002}. This has led
to a variety of algorithms that attempt to cluster a collection of data drawn from a subspace arrangement, giving rise to the challenging field of subspace clustering \citep{Vidal:SPM11-SC}. Such techniques can be iterative \citep{Bradley:JGO00,Tseng:JOTA00,Zhang:WSM09}, statistical \citep{Tipping-mixtures:99,Gruber-Weiss:CVPR04}, information-theoretic \citep{Ma:PAMI07}, algebraic \citep{Vidal:CVPR03-gpca,Vidal:PAMI05,Tsakiris:SIAM17}, spectral \citep{Boult:WMU91,Costeira:IJCV98,Kanatani:ICCV01,Chen:IJCV09}, or based on sparse  \citep{Elhamifar:CVPR09,Elhamifar:ICASSP10,Elhamifar:TPAMI13} and low-rank \citep{Liu:ICML10,Favaro:CVPR11,Liu:TPAMI13,Vidal:PRL14} representation theory. 

\indent \myparagraph{Hyperplane Clustering} A special class of subspace clustering is that of \emph{hyperplane clustering}, which arises when the data are drawn from a union of hyperplanes, i.e., a 
\emph{hyperplane arrangement}. Prominent applications include projective motion segmentation \citep{Vidal:IJCV06-multibody,Vidal:PAMI08}, 3D point cloud analysis \citep{Sampath:2010segmentation} and hybrid system identification \citep{Bako:Automatica11,Ma-Vidal:HSCC05}. Even though in some ways hyperplane clustering is simpler than general subspace clustering, since, e.g., the dimensions of the subspaces are equal and known a-priori, modern \emph{self-expressiveness-based} subspace clustering methods, such as \cite{Liu:TPAMI13,Lu:ECCV12,Elhamifar:TPAMI13}, in principle do not apply in this case, because they require small \emph{relative} subspace dimensions. \footnote{The relative dimension of a linear subspace is the ratio $d/D$, where $d$ is the dimension of the subspace and $D$ is the ambient dimension.}

From a theoretical point of view, one of the most appropriate methods for hyperplane clustering is Algebraic Subspace Clustering (ASC) \citep{Vidal:CVPR03-gpca,Vidal:PAMI05,Vidal:IJCV08,Tsakiris:Asilomar14,Tsakiris:FSASCICCV15,Tsakiris:SIAM17,Tsakiris:AffinePAMI17}, which gives closed-form solutions by means of factorization \citep{Vidal:CVPR03-gpca} or differentiation \citep{Vidal:PAMI05} of polynomials. However, the main drawback of ASC is its exponential complexity\footnote{The issue of robustness to noise for ASC has been recently addressed in \cite{Tsakiris:FSASCICCV15,Tsakiris:SIAM17}.} in the number $n$ of hyperplanes and the ambient dimension $D$, which makes it impractical in many settings. Another method that is theoretically justifiable for clustering hyperplanes is Spectral Curvature Clustering (SCC) \citep{Chen:IJCV09}, which is based on computing a $D$-fold affinity between all $D$-tuples of points in the dataset. As in the case of ASC, SCC is characterized by combinatorial complexity and becomes cumbersome for large $D$; even though it is possible to reduce its complexity, this comes at the cost of significant performance degradation. On the other hand, the intuitive classical method of $K$-hyperplanes (KH) \citep{Bradley:JGO00}, which alternates between assigning clusters and fitting a new normal vector to each cluster with PCA, is perhaps the most practical method for hyperplane clustering, since it is simple to implement, it is robust to noise and its complexity depends on the maximal allowed number of iterations. However, KH is sensitive to outliers and is guaranteed to converge only to a local minimum; hence multiple restarts are in general required. Median $K$-Flats (MKF) \citep{Zhang:WSM09} shares a similar objective function as KH, but uses the $\ell_1$-norm instead of the $\ell_2$-norm, in an attempt to gain robustness to outliers. MKF minimizes its objective function via a stochastic gradient descent scheme, and searches directly for a basis of each subspace, which makes it slower to converge for hyperplanes. Finally, we note that any single subspace learning method, such as RANSAC \citep{RANSAC} or REAPER \citep{Lerman:FCM15}, can be applied in a sequential fashion to learn a union of hyperplanes, by first learning the first dominant hyperplane, then removing the points lying close to it, then learning a second dominant hyperplane, and so on.

\indent \myparagraph{DPCP: A single subspace method for high relative dimensions} Recently, an $\ell_1$ method was introduced in the context of single subspace learning with outliers, called Dual Principal Component Pursuit (DPCP) \citep{Tsakiris:DPCPICCV15,Tsakiris:DPCP-ArXiv17}, which aims at recovering the orthogonal complement of a subspace in the presence of outliers. Since the orthogonal complement of a hyperplane is one-dimensional, DPCP is particularly suited for hyperplanes. DPCP searches for the normal to a hyperplane by solving a non-convex $\ell_1$ minimization problem on the sphere, or alternatively a recursion of linear programming relaxations. Assuming the dataset is normalized to unit $\ell_2$-norm and consists of points uniformly distributed on the great circle defined by a hyperplane (inliers), together with arbitrary points uniformly distributed on the sphere (outliers), \cite{Tsakiris:DPCPICCV15,Tsakiris:DPCP-ArXiv17} gave conditions under which the normal to the hyperplane is the unique global solution to the non-convex $\ell_1$ problem, as well as the limit point of a recursion of linear programming relaxations, the latter being reached after a finite number of iterations. 

\indent \myparagraph{Contributions} Motivated by the robustness of DPCP to outliers--DPCP was shown to be the only method capable of recovering the normal to the hyperplane in the presence of about $70\%$ outliers inside a $30$-dimensional ambient space \footnote{We note that the problem of robust PCA or subspace clustering for subspaces of large relative dimension becomes very challenging as the ambient dimension increases; see Section \ref{subsection:ExperimentsSynthetic}.} --one could naively use it for hyperplane clustering by recovering 
the normal to a hyperplane one at a time, while treating points from other hyperplanes as outliers. However, such a scheme is not a-priori guaranteed to succeed, since the outliers are now clearly structured, contrary to the theorems of correctness of \cite{Tsakiris:DPCPICCV15,Tsakiris:DPCP-ArXiv17} that assume that the outliers are uniformly distributed on the sphere. It is precisely this theoretical gap that we bridge in this paper: we show that as long as the hyperplanes are sufficiently separated, the dominant hyperplane is sufficiently dominant and the points are uniformly distributed (in a deterministic sense) inside their associated hyperplanes, then the non-convex DPCP problem has a unique (up to sign) global solution, equal to the normal vector of the dominant hyperplane. This suggests a sequential 
hyperplane learning algorithm, which first learns the dominant hyperplane, and weights all points according to their distance to that hyperplane. Then DPCP applied on the weighted data yields the second dominant hyperplane, and so on. Experiments on
corrupted synthetic data show that this DPCP-based sequential algorithm dramatically improves over similar sequential algorithms, which learn the dominant hyperplane via state-of-the-art single subspace learning methods (e.g., with RANSAC). Finally, 3D plane clustering experiments on real 3D point clouds show that a K-Hyperplanes DPCP-based scheme, which computes the normal vector of each cluster via DPCP, instead of the classic SVD, is very competitive to state-of-the-art approaches (e.g., RANSAC or SVD-based K-Hyperplanes).

\indent \myparagraph{Notation} For a positive integer $n$,  $[n]:=\left\{1,\dots,n\right\}$. For a vector $\w \in \Re^D$ we let $\hat{\w} = \w / \left\| \w \right\|_2$ if $\w \neq \0$, and $\hat{\w}=\0$ otherwise. $\Sp^{D-1}$ is the unit sphere of $\Re^D$. For two vectors $\w_1,\w_2 \in \Sp^{D-1}$, the \emph{principal angle} between $\w_1$ and $\w_2$ is the unique angle $\theta \in [0,\pi/2]$ with $\cos(\theta) = |\w_1^\transpose \w_2|$. $\1$ denotes the vector of all ones, LHS stands for \emph{left-hand-side} and RHS stands for \emph{right-hand-side}. Finally, for a set $\bX$ we denote by $\Card(\bX)$ the cardinality of $\bX$.

\indent \myparagraph{Paper Organization} 
The rest of the paper is organized as follows. In \S \ref{section:PriorArt} we review the prior-art in generic hyperplane clustering. In \S \ref{section:TheoreticalContributions} we discuss the theoretical contributions of this paper; proofs are given in \S \ref{section:Proofs}. \S \ref{section:Algorithms} describes the algorithmic contributions of this paper, and \S \ref{section:Experiments} contains experimental evaluations of the proposed methods. 

\section{Prior Art} \label{section:PriorArt}
Suppose given a dataset $\bX = [\x_1,\dots, \x_N ]$ of $N$ points of $\Re^D$, 
such that $N_i$ points of $\bX$, say $\bX_i \in \Re^{D \times N_i}$, lie close to a hyperplane $\H_i = \left\{\x: \, \b_i^\transpose \x = 0 \right\}$, where $\b_i$ is the normal vector to the hyperplane. Then the goal of hyperplane clustering is to identify the underlying hyperplane arrangement $\bigcup_{i=1}^n \H_i$ and cluster the dataset $\bX$ to the subsets (clusters) $\bX_1,\dots,\bX_n$.\footnote{We are assuming here that there is a unique hyperplane arrangement associated to the data $\bX$, for more details see \S \ref{subsection:DataModel}.} 

\indent \myparagraph{RANSAC} A traditional way of clustering points lying close to a hyperplane arrangement is by means of the \emph{RANdom SAmpling Consensus} algorithm (RANSAC) \citep{RANSAC}, which attempts to identify a single hyperplane $\H_i$ at a time. More specifically, RANSAC alternates between randomly selecting $D-1$ points from $\bX$ and counting the number of points in the dataset that are within distance $\delta$ from the hyperplane generated by the selected $D-1$ points. After a certain number of trials is reached, a first hyperplane $\hat{\H}_1$ is selected as the one that admits the largest number of points in the dataset within distance $\delta$. These points are then removed and a second hyperplane $\hat{\H}_2$ is obtained from the reduced dataset in a similar fashion, and so on. Naturally, RANSAC is sensitive to the thresholding parameter $\delta$. In addition, its efficiency depends on how big the probability is, that $D-1$ randomly selected points lie close to the same underlying hyperplane $\H_i$, for some $i \in [n]$. This probability depends on how large $D$ is as well as how balanced or unbalanced the clusters are. If $D$ is small, then RANSAC is likely to succeed with few trials. The same is true if one of the clusters, say $\bX_1$, is highly dominant, i.e., $N_1 >> N_i, \, \forall i \ge 2$, since in such a case, identifying $\H_1$ is likely to be achieved with only a few trials. On the other hand, if $D$ is large and the $N_i$ are of the same order of magnitude, then exponentially many trials are required (see \S \ref{section:Experiments} for some numerical results), and RANSAC becomes inefficient. 

\indent \myparagraph{K-Hyperplanes (KH)} Another very popular method for hyperplane clustering is the so-called
$K$-hyperplanes (KH), which was proposed by \cite{Bradley:JGO00}. KH attempts to minimize the non-convex objective function
\begin{align}
\J_{\text{KH}}(\f_1,\dots,\f_n;s_1,\dots,s_N):= \sum_{i=1}^n \left[\sum_{j=1}^N s_j(i) \left(\f_{i}^\transpose \x_j\right)^2 \right], \label{eq:KH}
\end{align} where $s_j : [n] \rightarrow \left\{0,1\right\}$ is the hyperplane assignment of point $\x_j$, i.e., $s_j(i)=1$ if and only if point $\x_j$ has been assigned to hyperplane $i$, and $\f_i \in \Sp^{D-1}$ is the normal vector to the estimated $i$-th hyperplane. Because of the non-convexity of \eqref{eq:KH}, the typical way to perform the optimization is by alternating between assigning clusters, i.e., given the $\f_i$ assigning $\x_j$ to its closest hyperplane (in the euclidean sense), and fitting hyperplanes, i.e., given the segmentation $\left\{s_j\right\}$, computing the best $\ell_2$ hyperplane for each cluster by means of PCA on each cluster. Because of this iterative refinement of hyperplanes and clusters, this method is sometimes also called \emph{Iterative Hyperplane Learning (IHL)}. The theoretical guarantees of KH are limited to convergence to a local minimum in a finite number of steps. Even though the alternating minimization in KH is computationally efficient, in practice several restarts are typically used, in order to select the best among multiple local minima. In fact, the higher the ambient dimension $D$ is the more restarts are required, which significantly increases the computational burden of KH. Moreover, KH is robust to noise but not to outliers, since the update of the normal vectors is done by means of standard ($\ell_2$-based) PCA.

\indent \myparagraph{Median K Flats (MKF)} It is precisely the sensitivity to outliers of KH that \emph{Median K Flats} (MKF) or \emph{Median K Hyperplanes} \citep{Zhang:WSM09} attempts to address, by minimizing the non-convex and non-smooth objective function
\begin{align}
\J_{\text{MKF}}(\f_1,\dots,\f_n;s_1,\dots,s_N):= \sum_{i=1}^n \left[\sum_{j=1}^N s_j(i) \left|\f_{i}^\transpose \x_j\right| \right]. \label{eq:MKF}
\end{align} Notice that \eqref{eq:MKF} is almost identical to 
the objective \eqref{eq:KH} of KH, except that the distances of the points to their assigned hyperplanes now appear without the square. This makes the optimization problem harder, and \cite{Zhang:WSM09} propose to solve it by means of a stochastic gradient approach, which requires multiple restarts, as KH does. Even though conceptually MKF is expected to be more robust to outliers than KH, we are not aware of any theoretical guarantees surrounding MKF that corroborate this intuition. Moreover, MKF is considerably slower than KH, since MKF searches directly for a basis of the hyperplanes, rather than the normals to the hyperplanes. We note here that MKF was not designed specifically for hyperplanes, rather for the more general case of unions of equi-dimensional subspaces. In addition, it is not trivial to adjust MKF to search for the orthogonal complement of the subspaces, which would be the efficient approach for hyperplanes. 

\indent\myparagraph{Algebraic Subspace Clustering (ASC)} ASC was originally proposed in \cite{Vidal:CVPR03-gpca} precisely for the purpose of provably clustering hyperplanes, a problem which at the time was handled either by the intuitive RANSAC or K-Hyperplanes. The idea behind ASC is to fit a polynomial $p(x_1,\dots,x_D) \in \Re[x_1,\dots,x_D]$ of degree $n$ to the data, where $n$ is the number of hyperplanes, and $x_1,\dots,x_D$ are polynomial indeterminates. In the absence of noise, this polynomial can be shown to have up to scale the form
\begin{align}
p(x) = (\b_1^\transpose x)\cdots (\b_n^\transpose x), \, \, \, x:=\begin{bmatrix} x_1 & \cdots x_D \end{bmatrix}^\transpose,
\end{align} where $\b_i \in \Sp^{D-1}$ is the normal vector to hyperplane $\H_i$. This reduces the problem to that of factorizing $p(x)$ to the product of linear factors, which was elegantly done in \cite{Vidal:CVPR03-gpca}. When the data are contaminated by noise, the fitted polynomial need no longer be factorizable; this apparent difficulty was circumvented in \cite{Vidal:PAMI05}, where it was shown that the gradient of the polynomial evaluated at point $\x_j$ is a good estimate for the normal vector of the hyperplane $\H_i$ that $\x_j$ lies closest to. Using this insight, one may obtain the hyperplane clusters by applying standard spectral clustering \citep{vonLuxburg:StatComp2007} on the \emph{angle-based} affinity matrix 
\begin{align}
\bA_{jj'} =  \left| \langle \frac{\nabla p|_{\x_j}}{||\nabla p|_{\x_j} ||_2}, \frac{\nabla p|_{\x_{j'}}}{||\nabla p|_{\x_{j'}} ||_2} \rangle \right| , \, \, \, j,j' \in [N] \label{eq:ABA}.
\end{align} The main bottleneck of ASC is computational: at least ${n+D-1 \choose n} -1$ points are required in order to fit the polynomial, which yields prohibitive complexity in many settings when $n$ or $D$ are large. A second issue with ASC is that it is sensitive to outliers; this is because the polynomial is fitted in an $\ell_2$ sense through SVD (notice the similarity with KH).

\indent \myparagraph{Spectral Curvature Clustering (SCC)} Another yet conceptually distinct method from the ones discussed so far is SCC, whose main idea is to build a $D$-fold tensor as follows. For each $D$-tuple of distinct points in the dataset, say $\x_{j_1},\dots,\x_{j_D}$, the value of the tensor is set to 
\begin{align}
\A(j_1,\dots,j_D) = \exp\left(-\frac{\left(c_p(\x_{j_1},\dots,\x_{j_D})\right)^2}{2\sigma^2}\right),
\end{align} where $c_p(\x_{j_1},\dots,\x_{j_D})$ is the polar curvature of the points $\x_{j_1},\dots,\x_{j_D}$ (see \cite{Chen:IJCV09} for an explicit formula) and $\sigma$ is a tuning parameter. Intuitively, the polar curvature is a multiple of the \emph{volume} of the simplex of the $D$ points, which becomes zero if the points lie in the same hyperplane, and the further the points lie from any hyperplane the larger the volume becomes. SCC obtains the hyperplane clusters by unfolding the tensor $\A$ to an affinity matrix, upon which spectral clustering is applied. As with ASC, the main bottleneck of SCC is computational, since in principle all ${N \choose D}$ entries of the tensor need to be computed. Even though the combinatorial complexity of SCC can be reduced, this comes at the cost of significant performance degradation. 

\indent \myparagraph{RANSAC/KH Hybrids} Generally speaking, any single subspace learning method that is robust to outliers and can handle
subspaces of high relative dimensions, can be used to perform hyperplane clustering, either
via a RANSAC-style or a KH-style scheme or a combination of both. For example, if $\mathcal{M}$ is a method that takes a dataset and fits to it a hyperplane, then one can use $\mathcal{M}$ to compute the first dominant hyperplane, remove the points in the dataset lying close to it, compute a second dominant hyperplane and so on (RANSAC-style). Alternatively, one can start with a random guess for $n$ hyperplanes, cluster the data according to their distance to these hyperplanes, and then use $\mathcal{M}$ (instead of the classic SVD) to fit a new hyperplane to each cluster, and so on (KH-style). Even though a large variety of single subspace learning methods exist, e.g., see references in \cite{Lerman:CA14}, only few such methods are potentially able to handle large relative dimensions and in particular hyperplanes. In addition to RANSAC, in this paper we will consider two other possibilities, i.e., REAPER and DPCP, which are described next.\footnote{Regression-type methods such as the one proposed in \cite{WangCamps:CVPR15} are also a possibility.}

\indent \myparagraph{REAPER}  A recently proposed single subspace learning method that admits an interesting theoretical analysis is the so-called \emph{REAPER} \citep{Lerman:FCM15}.
REAPER is inspired by the non-convex optimization problem
\begin{align}
\min \sum_{j=1}^{L} \left\|(\bI_D - \boldsymbol{\Pi})  \x_j \right\|_2, \, \, \, \text{s.t.} \, \, \, \boldsymbol{\Pi} \, \, \, \text{is an orthogonal projection}, \, \, \,  \Trace\left(\boldsymbol{\Pi} \right) = d, \label{eq:ReaperNonConvex}
\end{align} whose principle is to minimize the sum of the euclidean distances of all points to a single $d$-dimensional linear subspace $\U$; the matrix $\boldsymbol{\Pi}$ appearing in \eqref{eq:ReaperNonConvex} can be thought of as the product 
$\boldsymbol{\Pi} = \bU \bU^\transpose$, where $\bU \in \Re^{D \times d}$ contains in its columns an orthonormal basis for $\U$. 
As \eqref{eq:ReaperNonConvex} is non-convex, \cite{Lerman:FCM15} relax it to the convex semi-definite program 
\begin{align}
\min \sum_{j=1}^{L} \left\|(\bI_D - \boldsymbol{\bP})  \x_j \right\|_2, \, \, \, \text{s.t.} \, \, \, \0 \le \boldsymbol{\bP} \le \bI_D, \, \, \,  \Trace\left(\boldsymbol{\bP} \right) = d, \label{eq:ReaperConvex}
\end{align} which is the optimization problem that is actually solved by REAPER; the orthogonal projection matrix $\boldsymbol{\Pi}^*$ associated to $\U$, is obtained by projecting the solution $\bP^*$ of \eqref{eq:ReaperConvex} onto the space of rank $d$ orthogonal projectors. A limitation of REAPER is that the semi-definite program \eqref{eq:ReaperConvex} may become prohibitively large even for moderate values of $D$. This difficulty can be circumvented by solving \eqref{eq:ReaperConvex} in an \emph{Iteratively Reweighted Least Squares} (IRLS) fashion, for which convergence of the objective value to a neighborhood of the optimal value has been established in \cite{Lerman:FCM15}. 

\indent \myparagraph{Dual Principal Component Pursuit (DPCP)} Similarly to RANSAC \citep{RANSAC} or REAPER \citep{Lerman:FCM15}, DPCP \citep{Tsakiris:DPCPICCV15,Tsakiris:DPCP-ArXiv17} is another, recently proposed, single subspace learning method that can be applied to hyperplane clustering. The key idea of DPCP is to identify a single
hyperplane $\H$ that is \emph{maximal} with respect to the data $\bX$. Such a 
\emph{maximal hyperplane} is defined by the property that it must contain a maximal number of points $N_{\H}$ from the dataset, i.e., $N_{\H'} \le N_{\H}$ for any other hyperplane $\H'$ of $\Re^D$. 
Notice that such a maximal hyperplane can be characterized as a solution to the problem
\begin{align}
\min_{\b} \, \left\|\bX^\transpose \b \right\|_0 \, \, \st \, \, \, \b \neq \0, \label{eq:ell0}
\end{align} since $\left\|\bX^\transpose \b \right\|_0$ counts precisely the number of points in $\bX$ that are orthogonal to $\b$, and hence, contained in the hyperplane with normal vector $\b$. Problem \eqref{eq:ell0} is naturally relaxed to
\begin{align}
\min_{\b} \, \left\|\bX^\transpose \b \right\|_1 \, \, \st \, \, \, \left\|\b\right\|_2 = 1, \label{eq:ell1}
\end{align} which is a non-convex non-smooth optimization problem on the sphere. In the case where there is no noise and the dataset consists of $N_1 \ge D$ \emph{inlier} points $\bX_1$ drawn from a single hyperplane $\H_1 \cap \Sp^{D-1}$ with normal vector $\b_1 \in \Sp^{D-1}$, together with $M$ \emph{outlier} points $\bO \subset \Sp^{D-1}$ in general position, i.e. $\bX = \left[ \bX_1 \, \, \,   \bO \right] \boldsymbol{\Gamma}$, where $\boldsymbol{\Gamma}$ is an unknown permutation matrix, then there is a unique maximal hyperplane that coincides with $\H_1$. Under certain uniformity assumptions on the data points and abundance of inliers $\bX_1$, \cite{Tsakiris:DPCPICCV15,Tsakiris:DPCP-ArXiv17} asserted that $\b_1$ is the unique\footnote{The theorems in \cite{Tsakiris:DPCPICCV15,Tsakiris:DPCP-ArXiv17} are given for inliers lying in a proper subspace of arbitrary dimension $d$; the uniqueness follows from specializing $d=D-1$.} up to sign global solution of \eqref{eq:ell1}, i.e., the combinatorial problem \eqref{eq:ell0} and its non-convex relaxation \eqref{eq:ell1} share the same unique global minimizer. 
Moreover, it was shown that under the additional assumption that the principal angle of the initial estimate $\hat{\bn}_0$ from $\b_1$ is not large, the sequence 
$\left\{\hat{\bn}_k \right\}$ generated by the recursion of linear programs
\footnote{Notice that problem 
	\begin{align}
	\min_{\b} \left\|\bX^\transpose \b \right\|_1, \, \, \, \text{s.t.} \, \, \, \b^\transpose \hat{\bn}_k=1
	\end{align} may admit more than one global minimizer. Here and in the rest of the paper we denote by $\bn_{k+1}$ the solution obtained via the \emph{simplex method}.}
\begin{align}
\bn_{k+1} := \argmin_{\b:\, \, \, \b^\transpose \hat{\bn}_k=1} \left\|\bX^\transpose \b \right\|_1, \label{eq:ConvexRelaxations}
\end{align} converges up to sign to $\b_1$, after finitely many iterations. Alternatively, one can attempt to solve problem \eqref{eq:ell1} by means of an IRLS scheme, in a similar fashion as for REAPER. Even though no theory has been developed for this approach, experimental evidence in \cite{Tsakiris:DPCP-ArXiv17} indicates convergence of such an IRLS scheme to the global minimizer of \eqref{eq:ell1}.

\indent \myparagraph{Other Methods} In general, there is a large variety of clustering methods that can be adapted to perform hyperplane clustering, and the above list is by no means exhaustive; rather contains the methods that are intelluctually closest to the proposal of this paper. Important examples that we do not compare with in this paper are the statistical-theoretic \emph{Mixtures of Probabilistic Principal Component Analyzers} \citep{Tipping-PPCA:99}, as well as the information-theoretic \emph{Agglomerative Lossy Compression} \citep{Ma:PAMI07}. For an extensive account of these and other methods the reader is referred to \cite{Vidal:GPCAbook}.

\section{Theoretical Contributions} \label{section:TheoreticalContributions}

In this section we develop the main theoretical contributions 
of this paper, which are concerned with the properties of the non-convex $\ell_1$ minimization problem \eqref{eq:ell1}, as well as with the recursion of linear programs \eqref{eq:ConvexRelaxations} in the context of hyperplane clustering. More specifically, we are particularly interested in developing conditions under which every global minimizer of the non-convex problem \eqref{eq:ell1} is the normal vector to one of the hyperplanes of the underlying hyperplane arrangement. Towards that end, it is insightful to study an associated \emph{continuous} problem, which is obtained by replacing each finite cluster within a hyperplane by the uniform measure on the unit sphere of the hyperplane (\S \ref{subsection:ContinuousProblem}). The main result in that direction is Theorem \ref{thm:GeneralContinuous}. Next, by introducing certain uniformity parameters which measure the deviation of discrete quantities from their continuous counterparts, we adapt our analysis to the \emph{discrete} case of interest (\S \ref{subsection:DataModel}). This furnishes Theorem \ref{thm:DiscreteNonConvex}, which is the discrete analogue of Theorem \ref{thm:GeneralContinuous}, and gives global optimality conditions for the $\ell_1$ non-convex DPCP problem \eqref{eq:ell1} (\S \ref{subsection:DiscreteProblem}). Finally, Theorem \ref{thm:LPDiscrete} gives convergence guarantees for the linear programming recursion \eqref{eq:ConvexRelaxations}. The proofs of all results are deferred to \S \ref{section:Proofs}.

\subsection{Data model and the problem of hyperplane clustering} \label{subsection:DataModel}
Consider given a collection $\bX=\left[\x_1,\dots,\x_N \right] \in \Re^{D \times N}$ of $N$ points of the unit sphere $\Sp^{D-1}$ of $\Re^D$, that lie in an arrangement $\A$ of $n$ hyperplanes $\H_1,\dots,\H_n$ of $\Re^D$, i.e., $\bX \subset \A=\bigcup_{i=1}^n \H_i$, where each hyperplane $\H_i$ is the set of points of $\Re^D$ that are orthogonal to a \emph{normal vector} $\b_i \in \Sp^{D-1}$, i.e., $\H_i = \left\{\x \in \Re^D: \, \x^\transpose \b_i = 0 \right\}, \, i \in [n]:=\left\{1,\dots,n\right\}$. We assume that the data $\bX$ lie in \emph{general position} in $\mathcal{A} $, by which we mean two things. First, we mean that there are no linear relations among the points other than the ones induced by their membership to the hyperplanes. In particular, every $(D-1)$ points coming from $\H_i$ form a basis for $\H_i$
and any $D$ points of $\bX$ that come from at least two distinct $\H_i, \H_{i'}$ are linearly independent. Second, we mean that the points $\bX$ uniquely define the hyperplane arrangement $\A$, in the sense that $\A$ is the only arrangement of $n$ hyperplanes that contains $\bX$. This can be verified computationally by checking that
there is only one up to scale homogeneous polynomial of degree $n$ that fits the data, see 
\cite{Vidal:PAMI05,Tsakiris:SIAM17} for details.
We assume that for every $i \in [n]$, precisely $N_i$ points of $\bX$, denoted by $\bX_i=[\x_1^{(i)},\dots,\x_{N_i}^{(i)}]$, belong to $\H_i$, with $\sum_{i=1}^n N_i=N$. With that notation, $\bX=[\bX_1,\dots,\bX_n] \boldsymbol{\Gamma}$, where $\boldsymbol{\Gamma}$ is an unknown permutation matrix, indicating that the hyperplane membership of the points is unknown. Moreover, we assume an ordering $N_1 \ge N_2 \ge \cdots \ge N_n$, and we refer to $\H_1$ as \emph{the dominant hyperplane}.
After these preparations, the problem of hyperplane clustering can be stated as follows: given the data $\bX$, find the number $n$ of hyperplanes associated to $\bX$, a normal vector to each hyperplane, and a clustering of the data $\bX=\bigcup_{i=1}^n \bX_i$ according to hyperplane membership. 

\subsection{Theoretical analysis of the continuous problem} \label{subsection:ContinuousProblem}

 As it turns out, certain important insights regarding problem \eqref{eq:ell1} with respect to hyperplane clustering can be gained by examining an associated continuous problem. To see what that problem is, let $\hat{\H}_i=\H_i \cap \Sp^{D-1}$, and note first that for any $\b \in \Sp^{D-1}$ we have
\begin{align}
\frac{1}{N_i} \sum_{j=1}^{N_i} \left|\b^\transpose \x_j^{(i)} \right| \simeq \int_{\x \in \hat{\H}_i}\left|\b^\transpose \x\right| d \mu_{\hat{\H}_i}, \label{eq:DiscreteIntegral}
\end{align} where the LHS of \eqref{eq:DiscreteIntegral} is precisely $\frac{1}{N_i}\left\|\bX_i^\transpose \b \right\|_1$ and can be viewed as an approximation ($\simeq$) via the point set $\bX_i$ of the integral on the RHS of \eqref{eq:DiscreteIntegral}, with $\mu_{\hat{\H}_i}$ denoting the uniform measure on $\hat{\H}_i$. Letting 
$\theta_i$ be the principal angle between $\b$ and $\b_i$, i.e., the unique angle $\theta_i \in [0,\pi/2]$ such that $\cos(\theta_i)= |\b^\transpose \b_i|$, and $\pi_{\H_i}:\Re^D \rightarrow \H_i$ the orthogonal projection onto $\H_i$, for any $\x \in \H_i$ we have
\begin{align}
\b^\transpose \x = \b^\transpose \pi_{\H_i}(\x) = \left(\pi_{\H_i}(\b)\right)^\transpose \x = \h_{i,\b}^\transpose \x = \sin(\theta_i) \,  \hat{\h}_{i,\b}^\transpose \x.
\end{align} Hence, 
\begin{align}
\int_{\x \in \hat{\H}_i}\left|\b^\transpose \x\right| d \mu_{\hat{\H}_i} &= 
 \left[\int_{\x \in \hat{\H}_i}\left|\hat{\h}_{i,\b}^\transpose \x\right| d \mu_{\hat{\H}_i}\right] \sin(\theta_i) \\
  &=\left[\int_{\x \in \Sp^{D-2}}|x_1| d \mu_{\Sp^{D-2}}\right]\sin(\theta_i) \\ &= c \, \sin (\theta_i).\label{eq:JiContinutous}
\end{align} In the second equality above we made use of the rotational invariance of the sphere, as well as the fact that $\hat{\H}_i \cong \Sp^{D-2}$, which leads to (for details see the proof of Proposition $4$ and Lemma $7$ in \cite{Tsakiris:DPCP-ArXiv17})
\begin{align}
\int_{\x \in \hat{\H}_i}\left|\hat{\h}_{i,\b}^\transpose \x\right| d \mu_{\hat{\H}_i} = \int_{\x \in \Sp^{D-2}}|x_1| d \mu_{\Sp^{D-2}} =: c, \label{eq:c}
\end{align} where $x_1$ is the first coordinate of $\x$ and $c$ is the average height of the unit hemisphere in $\Re^D$. As a consequence, we can view the objective function of \eqref{eq:ell1}, which is given by
\begin{align}
\left\|\bX^\transpose \b \right\|_1 = \sum_{i=1}^n \left\|\bX_i^\transpose \b \right\|_1 = \sum_{i=1}^n N_i \left(\frac{1}{N_i} \sum_{j=1}^{N_i} \left|\b^\transpose \x_j^{(i)} \right|\right), \label{eq:Jdiscrete}
\end{align} as a discretization via the point set $\bX$ of the function 
\begin{align}
\mathcal{J}(\b) := \sum_{i=1}^n N_i \left( \int_{\x \in \hat{\H}_i}\left|\b^\transpose \x\right| d \mu_{\hat{\H}_i}\right) \stackrel{\eqref{eq:JiContinutous}}{=} \sum_{i =1}^n N_i \,  c \,  \sin(\theta_i). \label{eq:J}
\end{align} In that sense, the continuous counterpart of problem \eqref{eq:ell1} is 
\begin{align}
\min_{\b} \,  \, \, \J(\b)=N_1 \, c\,  \sin (\theta_1)+\cdots+N_n \, c \, \sin (\theta_n), \, \, \, \text{s.t.} \, \, \, \b  \in \Sp^{D-1} \label{eq:ContinuousProblem}.
\end{align} Note that $\sin(\theta_i)$ is the distance between the line spanned by $\b$ and the line spanned by $\b_i$.\footnote{Recall that $\theta_i$ is a principal angle, i.e., $\theta_i \in [0,\pi/2]$.} Thus, every global minimizer $\b^*$ of problem \eqref{eq:ContinuousProblem} minimizes the sum of the weighted distances of $\Span(\b^*)$  from $\Span(\b_1),\dots,\Span(\b_n)$, and can be thought of as representing a weighted median of these lines. Medians in Riemmannian manifolds, and in particular in the Grassmannian manifold, are an active subject of research \citep{draper2014flag,ghalieh199680}. However, we are not aware of any work in the literature that defines a median by means of \eqref{eq:ContinuousProblem}, nor any work that studies \eqref{eq:ContinuousProblem}.  

The advantage of working with \eqref{eq:ContinuousProblem} instead of \eqref{eq:ell1}, is that the solution set of the continuous problem \eqref{eq:ContinuousProblem} depends solely on the \emph{weights} $N_1 \ge N_2 \ge \cdots \ge N_n$ assigned to the hyperplane arrangement, as well as on the geometry of the arrangement, captured by the principal angles $\theta_{ij}$ between $\b_i$ and $\b_j$. In contrast, the solutions of the discrete problem \eqref{eq:ell1} may also depend on the distribution of the points $\bX$. From that perspective, understanding when problem \eqref{eq:ContinuousProblem} has a unique solution that coincides with the normal $\pm \b_1$ to the dominant hyperplane $\H_1$ is essential for understanding the potential of \eqref{eq:ell1} for hyperplane clustering.
Towards that end, we next provide a series of results pertaining to \eqref{eq:ContinuousProblem}. 
The first configuration that we examine is that of two hyperplanes. In that case the weighted geometric median of the two lines spanned by the normals to the hyperplanes always corresponds to one of the two normals:\begin{thm} \label{prp:TwoPlanes}
	Let $\b_1,\b_2$ be an arrangement of two hyperplanes in $\Re^D$, with weights $N_1 \ge N_2$. 
	Then the set $\mathfrak{B}^*$ of global minimizers of \eqref{eq:ContinuousProblem} satisfies:
	\begin{enumerate}
		\item If $N_1=N_2$, then $\mathfrak{B}^*=\left\{\pm \b_1, \pm \b_2 \right\}$. 
		\item If $N_1>N_2$, then $\mathfrak{B}^*=\left\{\pm \b_1\right\}$. 
	\end{enumerate}
\end{thm} Notice that when $N_1 >N_2$, problem \eqref{eq:ContinuousProblem} recovers the normal $\b_1$ to the dominant hyperplane, irrespectively of how separated the two hyperplanes are, since, according to Proposition \ref{prp:TwoPlanes}, the principal angle $\theta_{12}$ between $\b_1,\b_2$ does not play a role. 
The continuous problem \eqref{eq:ContinuousProblem} is equally favorable in recovering normal vectors as global minimizers in the \emph{dual} situation, where the arrangement consists of up to $D$ perfectly separated (orthogonal) hyperplanes, as asserted by the next Theorem.

\begin{thm} \label{prp:Orthogonal}
	Let $\b_1,\dots,\b_n$ be an orthogonal hyperplane arrangement ($\theta_{ij}=\pi/2, \, \forall i \neq j$) of $\Re^D$, with $n \le D$, and weights $N_1 \ge N_2 \ge \cdots \ge N_n$. Then the set $\mathfrak{B}^*$ of global minimizers of \eqref{eq:ContinuousProblem} can be characterized as follows:
	\begin{enumerate}
		\item If $N_1=N_n$, then $\mathfrak{B}^* = \left\{\pm \b_1,\dots,\pm \b_{n} \right\}$.
		\item If $N_1=\cdots=N_\ell>N_{\ell+1} \ge \cdots N_n$, for some $\ell \in [n-1]$, then
		$\mathfrak{B}^* = \left\{\pm \b_1,\dots,\pm \b_{\ell} \right\}$.
	\end{enumerate}
	\end{thm} Theorems \ref{prp:TwoPlanes} and \ref{prp:Orthogonal} are not hard to prove, since for two hyperplanes the objective can be shown to be a strictly concave function, while for orthogonal hyperplanes the objective is separable. In contrast, the problem becomes considerably harder for $n>2$ non-orthogonal hyperplanes. Even when $n=3$, characterizing the global minimizers of \eqref{eq:ContinuousProblem} as a function of the geometry and the weights seems challenging. Nevertheless, when the three hyperplanes are equiangular and their weights are equal, the symmetry of the configuration allows to analytically characterize the median as a function of the angle of the arrangement.	\begin{thm}\label{prp:equiangular}
		Let $\b_1,\b_2,\b_3$ be an equiangular hyperplane arrangement of $\Re^D$, $D \ge 3$, with $\theta_{12}=\theta_{13}=\theta_{23}=\theta \in (0,\pi/2]$ and weights $N_1=N_2=N_3$. Let $\mathfrak{B}^*$ be the set of global minimizers of \eqref{eq:ContinuousProblem}. Then $\mathfrak{B}^*$ satisfies the following phase transition:
		\begin{enumerate}
			\item If $\theta> 60^\circ$, then $\mathfrak{B}^*=\left\{\pm \b_1, \pm \b_2, \pm \b_3 \right\}$.
			\item If $\theta= 60^\circ$, then $\mathfrak{B}^*=\left\{\pm \b_1, \pm \b_2, \pm \b_3, \pm \frac{1}{\sqrt{3}} \1 \right\}$.
			\item If $\theta< 60^\circ$, then $\mathfrak{B}^*=\left\{\pm \frac{1}{\sqrt{3}} \1 \right\}$.
		\end{enumerate}
	\end{thm} Proposition \ref{prp:equiangular}, whose proof uses nontrivial arguments from spherical and algebraic geometry, is particularly enlightening, since it suggests that the global minimizers of \eqref{eq:ContinuousProblem} are associated to normal vectors of the underlying hyperplane arrangement when the hyperplanes are sufficiently separated, while otherwise they seem to be capturing the \emph{median hyperplane} of the arrangement. This is in striking similarity with the results regarding the \emph{Fermat point} of planar and spherical triangles \citep{ghalieh199680}.
However, when the symmetry in Theorem \ref{prp:equiangular} is removed, by not requiring the principal angles $\theta_{ij}$ or/and the weights $N_i$ to be equal, then our proof technique no longer applies, and the problem seems even harder. Even so, one intuitively expects an interplay between the angles and the weights of the arrangement under which, if the hyperplanes are \emph{sufficiently separated} and $\H_1$ is \emph{sufficiently dominant}, then \eqref{eq:ContinuousProblem} should have a unique global minimizer equal to $\b_1$. Our next theorem formalizes this intuition.

	\begin{thm} \label{thm:GeneralContinuous} Let $\b_1,\dots,\b_n$ be an arrangement of $n \ge 3$ hyperplanes in $\Re^D$, with pairwise principal angles $\theta_{ij}$. Let $N_1 \ge N_2\ge \cdots \ge N_n$ be positive integer weights assigned to the arrangement. Suppose that $N_1$ is large enough, in the sense that 
		\begin{align}
		N_1 > \sqrt{\alpha^2+\beta^2}, \label{c:ContinuousDominance}
		\end{align} where	
		\begin{align}		 
		\alpha &:=\sum_{i>1}N_i \sin(\theta_{1,i})-
		\sqrt{\sum_{i>1}N_i^2-\sigma_{\max}\left(\left[N_iN_j\cos(\theta_{ij})\right]_{i,j>1}\right)} \ge 0,\\
		\beta &:=\sqrt{\sum_{i>1} N_i^2 + 2 \sum_{i \neq j, \,  i,j >1} N_i N_j\cos(\theta_{ij})}, \label{eq:alphabeta}
		\end{align}  with $\sigma_{\max}\left(\left[N_iN_j\cos(\theta_{ij})\right]_{i,j>1}\right)$ denoting the maximal eigenvalue of the $(n-1)\times(n-1)$ matrix, whose 
		$(i-1,j-1)$ entry is $N_i N_j \cos(\theta_{ij})$ and $1 < i,j \le n$. Then any global minimizer $\b^*$ of problem \eqref{eq:ContinuousProblem} must satisfy $\b^*=\pm \b_i$, for some $i \in [n]$. If in addition,
		\begin{align}
		   \gamma:=\min_{i'\neq 1} \sum_{i \neq i'} N_i \sin(\theta_{i',i}) - \sum_{i>1} N_i \sin(\theta_{1,i})>0, \label{eq:NormalAvoidanceGeneralFirst}
		\end{align} then problem \eqref{eq:ContinuousProblem} admits a unique up to sign global minimizer $\b^* = \pm \b_1$.
	\end{thm} Let us provide some intuition about the meaning of the quantities $\alpha,\beta$ and $\gamma$ in Theorem \ref{thm:GeneralContinuous}. To begin with, the first term in $\alpha$ is precisely equal to $\J(\b_1)$, while the second term in $\alpha$ is a lower bound on the objective function $N_2 \sin(\theta_2) + \cdots + N_n \sin(\theta_n)$, if one discards hyperplane $\H_1$. Moving on, under the hypothesis that $N_1>\sqrt{\alpha^2+\beta^2}$, the quantity $\frac{\beta}{N_1}$ admits a nice geometric interpretation: $\cos^{-1} \left(\frac{\beta}{N_1} \right)$ is a lower bound on how small the principal angle of a critical point $\b^{\dagger}$ from $\b_1$ can be, if $\b^{\dagger} \neq \pm \b_1$. Interestingly, this means that the larger $N_1$ is, the larger this minimum angle is, which shows that critical hyperplanes $\H^{\dagger}$ (i.e., hyperplanes defined by critical points $\b^\dagger$) that are distinct from $\H_1$, must be sufficiently separated from $\H_1$. Finally, the first term in $\gamma$ is $\J(\b_1)$, while the second term is the smallest objective value that corresponds to $\b = \b_i, \, i>1$, and so \eqref{eq:NormalAvoidanceGeneralFirst} simply guarantees that $\J(\b_1)< \J(\b_i), \, \forall i>1$.
	
	Next, notice that condition $N_1 > \sqrt{\alpha^2+\beta^2}$ of Theorem \ref{thm:GeneralContinuous} is easier to satisfy when $\H_1$ is close to the rest of the hyperplanes (which leads to small $\alpha$), while the rest of the hyperplanes are sufficiently separated (which leads to small $\alpha$ and small $\beta$). Here the notion of \emph{close} and \emph{separated} is to be interpreted relatively to $\H_1$ and its assigned weight $N_1$. Regardless, one can show that	
	\begin{align}	
	\sqrt{2} \sum_{i>1} N_i \ge \sqrt{\alpha^2+\beta^2},
	\end{align} and so if 
	\begin{align}
	N_1 > \sqrt{2} \sum_{i>1} N_i,
	\end{align} then any global minimizer of \eqref{eq:ContinuousProblem} has to be one of the normals, irrespectively of the $\theta_{ij}$. Finally, condition \eqref{eq:NormalAvoidanceGeneralFirst} is consistent with 
	 condition \eqref{c:ContinuousDominance} in that it requires $\H_1$ to be close to $\H_i, i >1$ and $\H_i,\H_j$ to be sufficiently separated for $i, j >1$. Once again, \eqref{eq:NormalAvoidanceGeneralFirst} can always be satisfied irrespectively of the $\theta_{ij}$, by choosing $N_1$ sufficiently large, since only the positive term in the definition of $\gamma$ depends on $N_1$, once again manifesting that the terms \emph{close} and \emph{separated} are used relatively to $\H_1$ and its assigned weight $N_1$. 
	 
Removing the term $N_1 \sin(\theta_1)$ from the objective function, which corresponds to having identified $\H_1$ and removing its associated points, one may re-apply Theorem \ref{thm:GeneralContinuous} to the remaining hyperplanes $\H_2,\dots,\H_n$ to obtain conditions for recovering $\b_2$ and so on. Notice that these conditions will be independent of $N_1$, rather they will be relative to $\H_2$ and its assigned weight $N_2$, and can always be satisfied for large enough $N_2$. Finally, further recursive application of Theorem \ref{thm:GeneralContinuous} can furnish conditions for sequentially recovering all normals $\b_1,\dots,\b_n$. However, we should point out that the conditions of Theorem \ref{thm:GeneralContinuous} are readily seen to be stronger than necessary. For example, we already know from Theorem \ref{prp:Orthogonal} that when the arrangement is orthogonal, i.e., $\theta_{ij}=\pi/2, \, \forall i \neq j$, then problem \eqref{eq:ContinuousProblem} has a unique minimizer $\pm \b_1$ as soon as $N_1 >N_i, \forall i>1$. On the contrary, Theorem \ref{thm:GeneralContinuous} applied to that case requires $N_1$ to be unnecessarily large, i.e., condition \eqref{c:ContinuousDominance} becomes
\begin{align}
N_1^2 > \left(\sum_{i>1} N_i \right)^2 + 2 \sum_{i>1} N_i^2 - 2 \left(\sum_{i>1}N_i \right) \left(\sum_{i>1} N_i^2 \right)^{1/2},
\end{align} which in the special case $N_2=\cdots=N_n$ reduces to $N_1> (n-1) N_2$.
 This is clearly an artifact of the techniques used to prove Theorem \ref{thm:GeneralContinuous}, and not a weakness of problem \eqref{eq:ContinuousProblem} in terms of its global optimality properties. Improving the proof technique of Theorem \ref{thm:GeneralContinuous} is an open problem.

\subsection{Theoretical analysis of the discrete problem} \label{subsection:DiscreteProblem}

We now turn our attention to the discrete problem of hyperplane clustering via DPCP, i.e., to problems \eqref{eq:ell1} and \eqref{eq:ConvexRelaxations}, for the case where $\bX = [\bX_1,\dots,\bX_n] \boldsymbol{\Gamma}$, with $\bX_i$ being $N_i$ points in $\H_i$, as described in \S \ref{subsection:DataModel}. As a first step of our analysis we define certain \emph{uniformity parameters} $\epsilon_i$, which serve as link between the continuous and discrete domains. Towards that end, note that
for any $i \in [n]$ and $\b \in \Sp^{D-1}$, the quantity $||\bX_i^\transpose \b ||_1$ can be written as 
\begin{align}
\left\|\bX_i^\transpose \b \right\|_1 =\sum_{j=1}^{N_i}\left|\b^\transpose \x_j^{(i)}\right| = \b^\transpose \sum_{j=1}^{N_i} \Sign\left(\b^\transpose \x_j^{(i)}\right) \x_j^{(i)}=N_i \, \b^\transpose \bchi_{i,\b},  \label{eq:Ji}
\end{align} where 
\begin{align}
\bchi_{i,\b} :=  \frac{1}{N_i} \sum_{j=1}^{N_i} \Sign\left(\b^\transpose \x_{j}^{(i)}\right) \x_j^{(i)}
\end{align} is \emph{the average point of $\bX_i$ with respect to the} orthogonal projection $\h_{i,\b}:=\pi_{\H_i}(\b)$ of $\b$ onto $\H_i$. Notice that $\bchi_{i,\b}$ can be seen as an average of the function $\y \in \hat{\H}_i \mapsto \Sign\left(\b^\transpose \y\right) \y \in \hat{\H}_i$ through the point sent $\bX_i \subset \hat{\H}_i$. In other words, 
$\bchi_{i,\b}$ can be seen as an approximation of the integral\footnote{For details regarding the evaluation of this integral see Lemma $9$ and its proof in \cite{Tsakiris:DPCP-ArXiv17}.}
\begin{align}
\int_{\x \in \hat{\H}_i} \Sign\left(\b^\transpose \x\right) \x \,  d \mu_{\hat{\H}_i} = c\, \hat{\h}_{i,\b},
\end{align} where $c$ was defined in \eqref{eq:c}. To remove the dependence on $\b$ we define the approximation error $\epsilon_i$ associated to hyperplane $\H_i$ as
\begin{align}
\epsilon_i :=  \max_{\b \in \Sp^{D-1}} \left\|\bchi_{i,\b} - c \, \widehat{\h}_{i,\b} \right\|_2.\label{eq:epsilon}
\end{align} Then it can be shown (see \cite{Tsakiris:DPCP-ArXiv17} \S $4.2$) that
when the points $\bX_i$ are uniformly distributed in a deterministic sense
\citep{Grabner:MR97,Grabner:MathComp93}, $\epsilon_i$ is small and in particular $\epsilon_i \rightarrow 0$ as $N_i \rightarrow \infty$. 

We are now almost ready to state our main results, before doing so though we need a rather technical, yet necessary, definition.
\begin{dfn} \label{dfn:DPCPmultiple_circumradius}	
For a set $\boldsymbol{\mathcal{Y}} = [\y_1,\dots,\y_L] \subset \Sp^{D-1}$ and positive integer $K\le L$,
	define $\R_{\boldsymbol{\mathcal{Y}},K}$ to be the maximum circumradius among the circumradii of all polytopes $\left\{\sum_{i=1}^K \alpha_{j_i} \y_{j_i}: \, \alpha_{j_i} \in [-1,1] \right\}$,
where $j_1,\dots,j_K$ are distinct integers in $[L]$, and the circumradius of a closed bounded set is the minimum radius among all spheres that contain the set. We now define the quantity of interest as
	\begin{align}
	\mathcal{R} := \max_{\substack{K_1+\cdots+K_n = D-1 \\ 0 \le K_i \le D-2}} \sum_{i=1}^n \mathcal{R}_{\bX_i,K_i}. \label{eq:R}
	\end{align}
\end{dfn} \noindent We note that it is always the case that $\mathcal{R}_{\bX_i,K_i} \le K_i$, with this upper bound achieved when $\bX_i$ contains $K_i$ colinear points. Combining this fact with the constraint $\sum_i K_i = D-1$ in \eqref{eq:R}, we get that $\mathcal{R} \le D-1$, and the more uniformly distributed are the points $\bX$ inside the hyperplanes, the smaller $\R$ is (even though $\R$ does not go to zero).	 
	 
Recalling the definitions of $c$, $\epsilon_i$ and $\mathcal{R}$ in \eqref{eq:c}, \eqref{eq:epsilon} and \eqref{eq:R}, respectively, we have the following result about the non-convex problem \eqref{eq:ell1}.
\begin{thm} \label{thm:DiscreteNonConvex}
	Let $\b^*$ be a solution of \eqref{eq:ell1} with $\bX=[\bX_1,\dots,\bX_n] \boldsymbol{\Gamma}$, and suppose that $c> \sqrt{2} \epsilon_1$. If
	\begin{align}
	N_1 &> \sqrt{\bar{\alpha}^2+\bar{\beta}^2}, \label{c:thmDiscreteNonConvexDominance} \, \, \, \text{where} \\ 
	\bar{\alpha} &:=  \alpha+c^{-1}\left(\epsilon_1N_1+2\sum_{i>1} \epsilon_i N_i\right),  \, \, \, \text{and} \\ 
	\bar{\beta} &:=\beta + c^{-1}\left(\mathcal{R} + \sum \epsilon_i N_i \right),  
	\end{align} with $\alpha, \beta$ as in Theorem \ref{thm:GeneralContinuous}, then $\b^* =\pm \b_i$ for some $i \in [n]$. Furthermore, $\b^* = \pm \b_1$, if  
	\begin{align}
	\bar{\gamma}:=\gamma - c^{-1}\left(\epsilon_1 N_1 + \epsilon_2 N_2 + 2 \sum_{i>2} \epsilon_i N_i\right) \label{eq:gammaBar} >0. 
	\end{align} 
\end{thm} Notice the similarity of conditions $N_1 > \sqrt{\bar{\alpha}^2+\bar{\beta}^2},\bar{\gamma}>0$ of Theorem \ref{thm:DiscreteNonConvex} with conditions $N_1 > \sqrt{\alpha^2+\beta^2},\gamma>0$ of Theorem \ref{thm:GeneralContinuous}. In fact $\bar{\alpha}> \alpha, \bar{\beta}> \beta, \bar{\gamma}<\gamma$, which implies that the conditions of Theorem \ref{thm:DiscreteNonConvex} are strictly stronger than those of Theorem \ref{thm:GeneralContinuous}. This is no surprise since, as we have already remarked, the solution set of \eqref{eq:ell1} depends not only on the geometry ($\theta_{ij}$) and the weights $(N_i)$ of the arrangement, but also on the distribution of the data points (parameters $\epsilon_i$ and $\mathcal{R}$). 

We note that in contrast to condition \eqref{c:ContinuousDominance} of Theorem \ref{thm:GeneralContinuous}, $N_1$ now appears in both sides of condition \eqref{c:thmDiscreteNonConvexDominance} of Theorem \ref{thm:DiscreteNonConvex}. Nevertheless, under the assumption $c > \sqrt{2} \epsilon_1$, \eqref{c:thmDiscreteNonConvexDominance} is equivalent to the positivity of a quadratic polynomial in $N_1$, whose leading coefficient is positive, and hence it can always be satisfied for sufficiently large $N_1$.

Another interesting connection of Theorem \ref{thm:GeneralContinuous} to Theorem \ref{thm:DiscreteNonConvex}, is that the former can be seen as a limit version of the latter : dividing \eqref{c:thmDiscreteNonConvexDominance} and \eqref{eq:gammaBar} by $N_1$, letting $N_1,\dots,N_n$ go to infinity while keeping each ratio $N_i/N_1$ fixed, and recalling that $\epsilon_i \rightarrow 0$ as $N_i \rightarrow \infty$ and $\mathcal{R} \le D-1$, we recover the conditions of Theorem \ref{thm:GeneralContinuous}. 

Next, we consider the linear programming recursion \eqref{eq:ConvexRelaxations}. At a conceptual level, the main difference between the linear programming recursion in \eqref{eq:ConvexRelaxations} and the continuous and discrete problems \eqref{eq:ContinuousProblem} and \eqref{eq:ell1}, respectively, is that 
the behavior of \eqref{eq:ConvexRelaxations} depends highly on the initialization $\hat{\bn}_0$. Intuitively, the closer $\hat{\bn}_0$ is to $\b_1$, the more likely the recursion will converge to $\b_1$, with this likelihood becoming larger for larger $N_1$. The precise technical statement is as follows.

\begin{thm} \label{thm:LPDiscrete}
	Let $\left\{\hat{\bn}_k\right\}$ be the sequence generated by the linear programming recursion
	 \eqref{eq:ConvexRelaxations} by means of the simplex method, where $\hat{\bn}_0 \in \Sp^{D-1}$ is an initial estimate for $\b_1$, 
	with principal angle from $\b_i$ equal to $\theta_{i,0}$. Suppose that $c> \sqrt{5} \epsilon_1$, and let $\theta_{\min}^{(1)}=\min_{i>1} \left\{\theta_{1i} \right\}$. If $\theta_{1,0}$ is small enough, i.e.,
	\begin{align}
	\sin(\theta_{1,0})<\min \left\{\sin\left(\theta_{\min}^{(1)}\right)-2\epsilon_1, \sqrt{1-(c^{-1}\epsilon_1)^2}-2c^{-1}\epsilon_1 \right\}, \label{c:theta10}
	\end{align} and $N_1$ is large enough in the sense that
	\begin{align}
	N_1 > \max \left\{\mu,\frac{\nu+\sqrt{\nu^2+4\rho \tau }}{2 \tau} \right\}, \label{c:thmLPDiscreteDominance} \, \, \, \text{where}
	\end{align}	
	\begin{align}
	\mu &:= \max_{j\neq 1} \left\{
	\frac{\sum_{i>1}N_i \sin(\theta_{i,0})+c^{-1}\epsilon_j N_j+\sum_{i\neq 1,j}N_i\left[2c^{-1}\epsilon_i-\sin(\theta_{ij}) \right]}{\sin(\theta_{1j})-\sin(\theta_{1,0})-2c^{-1}\epsilon_1}\right\}, \\
	\nu &:= 2c^{-1}\epsilon_1\left(\beta+c^{-1}\mathcal{R}+c^{-1}\sum_{i>1}\epsilon_i N_i \right) + 2 \left[\sin(\theta_{1,0})+2c^{-1}\epsilon_1\right]\left(\alpha+2c^{-1}\sum_{i>1}\epsilon_i N_i\right), \\	
	\rho &:= \left(\alpha+2c^{-1}\sum_{i>1}\epsilon_i N_i\right)^2 + \left(\beta+c^{-1}\mathcal{R}+c^{-1}\sum_{i>1} \epsilon_i N_i\right)^2, \label{eq:rho} \\
	\tau &:= \cos^2(\theta_{1,0})-4c^{-1}\epsilon_1 \sin(\theta_{1,0})-5(c^{-1}\epsilon_1)^2, \label{eq:sigma}
	\end{align} \normalsize with $\alpha,\beta$ as in Theorem \ref{thm:GeneralContinuous}, then $\left\{\bn_k\right\}$ converges to either $\b_1$ or $-\b_1$ in a finite number of steps.
\end{thm} The quantities appearing in Theorem \ref{thm:LPDiscrete} are harder to interpret than those of Theorem \ref{thm:DiscreteNonConvex}, but we can still give some intuition about their meaning. To begin with, the two inequalities in \eqref{c:thmLPDiscreteDominance} represent two distinct requirements that we enforced in our proof, which when combined, guarantee that the limit point of \eqref{eq:ConvexRelaxations} is $\pm \b_1$. 

The first requirement is that no $\pm \b_i$ can be the limit point of \eqref{eq:ConvexRelaxations} for $i>1$; this is captured by a linear inequality of the form
\begin{align}
\mu \, N_1 + (\text{terms not depending on} \, \, N_1) >0, 
\end{align} which is satisfied either for $N_1$ sufficiently large (if $\mu>0$) or for $N_1$ sufficiently small (if $\mu<0$). To avoid pathological situations where $N_1$ is required to be negative or less than $D-1$, it is natural to enforce $\mu$ to be positive. This is precisely achieved by inequality $\sin(\theta_{1,0})< \sin\left(\theta_{\min}^{(1)}\right)-2\epsilon_1$ in \eqref{c:theta10}, which is a quite natural condition itself: the initial estimate $\hat{\bn}_0$ needs to be closer to $\b_1$ than any other normal $\b_i$ for $i>1$, and the more well-distributed the data $\bX_1$ are inside $\H_1$  (smaller $\epsilon_1$), the further $\hat{\bn}_0$ can be from $\b_1$.

The second requirement that we employed in our proof is that the limit point of \eqref{eq:ConvexRelaxations} is one of the $\pm \b_1,\dots,\pm \b_n$; this is captured by requiring that a certain quadratic polynomial 
\begin{align}
p(N_1):=\tau \, N_1^2 - \nu \, N_1 - \rho 
\end{align}
in $N_1$ is positive. To avoid situations where the positivity of this polynomial contradicts the relation $N_1 > \mu$, it is important that we ask its leading coefficient $\tau$ to be positive, so that the second requirement is satisfied for $N_1$ large enough, and thus is compatible with $N_1 > \mu$. As it turns out, $\tau$ is positive only if the data $\bX_1$ are sufficiently well distributed in $\H_1$, which is captured by condition $c> \sqrt{5} \epsilon_1$ of Theorem \ref{thm:LPDiscrete}. Even so, this latter condition is not sufficient; instead $\sin(\theta_{1,0})< \sqrt{1-(c^{-1}\epsilon_1)^2}-2c^{-1}\epsilon_1$ is needed (as in \eqref{c:theta10}), which is once again very natural: the more well-distributed the data $\bX_1$ are inside $\H_1$  (smaller $\epsilon_1$), the further $\hat{\bn}_0$ from $\b_1$ can be.

Next, notice that the conditions of Theorem \ref{thm:LPDiscrete} are not directly comparable to those of Theorem \ref{thm:DiscreteNonConvex}. Indeed, it may be the case that $\pm \b_1$ is not a global minimizer of the non-convex problem \eqref{eq:ell1}, yet the recursions \eqref{eq:ConvexRelaxations} do converge to $\b_1$, simply because $\hat{\bn}_0$ is \emph{close} to $\b_1$. In fact, by \eqref{c:theta10} $\hat{\bn}_0$ must be closer to $\b_1$ than $\b_i$ to $\b_1$ for any $i>1$, i.e., $\theta_{\min}^{(1)}>\theta_{1,0}$. Similarly to Theorems \ref{thm:GeneralContinuous} and \ref{thm:DiscreteNonConvex}, the more separated the hyperplanes $\H_i,\H_j$ are for $i,j>1$, the easier it is to satisfy condition \eqref{c:thmLPDiscreteDominance}. In contrast, $\H_{1}$ needs to be sufficiently separated from $\H_i$ for $i>1$, since otherwise $\mu$ becomes large. This has an intuitive explanation: the less separated $\H_1$ is from the rest of the hyperplanes, the \emph{less resolution} the linear program \eqref{eq:ConvexRelaxations} has in distinguishing $\b_1$ from $\b_i, \, i>1$. To increase this \emph{resolution}, one needs to either select $\hat{\bn}_0$ very close to $\b_1$, or select $N_1$ very large. The acute reader may recall that the quantity $\alpha$ appearing in \eqref{eq:rho} becomes larger when $\H_1$ becomes separated from $\H_i, \, i>1$. Nevertheless, there are no inconsistency issues in controlling the size of $\mu$ and $\rho$. This is because $\alpha$ is always bounded from above by $\sum_{i>1} N_i$, i.e., $\alpha$ does not increase arbitrarily as the $\theta_{1i}$ increase. Another way to look at the consistency of condition \eqref{c:thmLPDiscreteDominance}, is that its RHS does not depend on $N_1$; hence one can always satisfy \eqref{c:thmLPDiscreteDominance} by selecting $N_1$ large enough.
 
\section{Proofs} \label{section:Proofs}
 
In this Section we prove Theorems \ref{prp:TwoPlanes}-\ref{prp:equiangular} associated to the continuous problem \eqref{eq:ContinuousProblem}, as well as Theorems \ref{thm:DiscreteNonConvex} and \ref{thm:LPDiscrete} associated to the discrete non-convex $\ell_1$ minimization problem \eqref{eq:ell1} and the recursion of linear programs \eqref{eq:ConvexRelaxations} respectively.

\subsection{Preliminaries on the continuous problem} \label{subsection:PreliminariesContinuous}

We start by noting that the objective function \eqref{eq:ContinuousProblem} is everywhere differentiable except at the points $\pm \b_1,\dots, \pm \b_n$, where its partial derivatives do not exist. For any $\b \in \Sp^{D-1}$ distinct from $\pm \b_i$, the gradient at $\b$ is given by
\begin{align}
\nabla_{\b} \J = -\sum_{i=1}^n \frac{\b_i^\transpose \b}{\left(1 - (\b_i^\transpose \b)^2\right)^{\frac{1}{2}}} \b_i.
\end{align} Now let $\b^*$ be a global solution of \eqref{eq:ContinuousProblem} and suppose that $\b^* \neq \pm \b_i, \, \forall i \in [n]$. Then $\b^*$ must satisfy 
the first order optimality condition 
\begin{align}
\nabla_{\b} \J|_{\b^*} + \lambda^* \, \b^* = \0, 
\end{align} where $\lambda^*$ is a Lagrange multiplier. Equivalently, we have
\begin{align}
-\sum_{i=1}^n N_i \, \left(\b_i^\transpose \b^*\right) 
\left(1-\left(\b_i^\transpose \b^*\right)^2  \right)^{-\frac{1}{2}} \b_i + \lambda^* \, \b^* = \0, \label{eq:Optimality}
\end{align} which implies that
\begin{align}
\sum_{i=1}^n N_i \, \left(\b_i^\transpose \b^*\right)^2 
\left(1-\left(\b_i^\transpose \b^*\right)^2  \right)^{-\frac{1}{2}} \b^* = \sum_{i=1}^n N_i \,  \left(\b_i^\transpose \b^*\right) 
\left(1-\left(\b_i^\transpose \b^*\right)^2  \right)^{-\frac{1}{2}} \b_i, \label{eq:bStar}
\end{align} from which the next Lemma follows.
\begin{lem} \label{lem:bSpan}
	Let $\b^*$ be a global solution of \eqref{eq:ContinuousProblem}. Then $\b^* \in \Span \left(\b_1,\dots,\b_n \right)$.
\end{lem}
\begin{proof}
	If $\b^*$ is equal to some $\pm \b_i$, then the statement of the Lemma is certainly true. If $\b^* \neq \pm \b_i, \, \forall i \in [n]$, then $\b^*$ satisfies \eqref{eq:bStar}, from which again the statement is true.
\end{proof} 

\subsection{Proof of Theorem \ref{prp:TwoPlanes}} \label{subsection:ProofTwoPlanes}

By Lemma \ref{lem:bSpan} any global solution must lie in the plane $\Span(\b_1,\b_2)$, and so our problem becomes planar, i.e., we may as well assume that the hyperplane arrangement $\b_1,\b_2$ is a line arrangement of $\Re^2$. Note that $\b_1,\b_2 \in \Sp^1$ partition $\Sp^1$ in two arcs, and among these, only one arc has length $\theta$ strictly less than $\pi$; we denote this arc by $\mathfrak{a}$. 	
Next, recall that the continuous objective function for two hyperplanes can be written as 
\begin{align}
\J(\b) =N_1 \left(1 - (\b_1^\transpose \b)^2\right)^{\frac{1}{2}} + N_2 \left(1 - (\b_2^\transpose \b)^2\right)^{\frac{1}{2}}, \, \, \, \b \in \Sp^1.
\end{align} Let $\b^*$ be a global solution, and suppose that $\b^* \not\in \mathfrak{a}$. If $-\b^* \in \mathfrak{a}$, then we can replace $\b_1,\b_2$ by $-\b_1,-\b_2$, an operation that does not change neither the arrangement nor the objective. After this replacement, we have that $\b^* \in \mathfrak{a}$. Finally suppose that neither $\b^*$ nor $-\b^*$ are inside $\mathfrak{a}$. Then replacing either $\b_1$ with $-\b_1$ or $\b_2$ with $-\b_2$, leads to $\b^* \in \mathfrak{a}$.  Consequently, without loss of generality we may assume that $\b^*$ lies in $\mathfrak{a}$. Moreover, subject to a rotation and perhaps exchanging $\b_1$ with $\b_2$, we can assume that $\b_1$ is aligned with the positive $x$-axis and that the angle $\theta$ between $\b_1$ and $\b_2$, measured counter-clockwise, lies in $(0, \pi)$.	
Then $\b^*$ is a global solution to 
\begin{align}
\J(\b) =N_1 \left(1 - (\b_1^\transpose \b)^2\right)^{\frac{1}{2}} + N_2 \left(1 - (\b_2^\transpose \b)^2\right)^{\frac{1}{2}}, \, \, \, \b \in \Sp^1 \cap \mathfrak{a}.
\end{align} 		 	
Now, for any vector $\b \in \Sp^1 \cap \mathfrak{a}$, let $\theta_1,\theta_2=\theta - \theta_1$ be the angle between $\b$ and $\b_1,\b_2$ respectively. Then our objective can be written as 	
\begin{align}
\J(\b) = \tilde{\J}(\theta_1) = N_1 \sin(\theta_1) + N_2 \sin(\theta - \theta_1), \, \, \, \theta_1 \in [0,\theta].
\end{align} Taking first and second derivatives, we have
\begin{align}
\frac{\partial \tilde{\J}}{\partial \theta_1} &= N_1 \cos(\theta_1) - N_2 \cos(\theta-\theta_1) \\
\frac{\partial^2 \tilde{\J}}{\partial \theta_1^2} &=- N_1 \sin(\theta_1) - N_2 \sin(\theta-\theta_1).
\end{align} Since the second derivative is everywhere negative on $[0,\theta]$, $\tilde{\J}(\theta_1)$ is strictly concave on $[0,\theta]$ and so its minimum must be achieved at the boundary $\theta_1=0$ or $\theta_1=\theta$. This means that either $\b^* = \b_1$ or $\b^* = \b_2$. 

\subsection{Proof of Theorem \ref{prp:Orthogonal}} \label{subsection:ProofOrthogonal}

For the sake of simplicity we assume $n=3$, the general case follows in a similar fashion. Letting $x_i := \b_i^\transpose \b$ and $y_i := \sqrt{1-x_i^2}$, \eqref{eq:bStar} can be 
written as 
\begin{align}
\left( N_1\frac{x_1^2}{y_1} + N_2\frac{x_2^2}{y_2} + N_3\frac{x_3^2}{y_3} \right) \b^* = N_1\frac{x_1}{y_1}\b_1 + N_2\frac{x_2}{y_2}\b_2 + N_3\frac{x_3}{y_3} \b_3. \label{eq:bStarxy}
\end{align} Taking inner products of \eqref{eq:bStarxy} with $\b_1,\b_2,\b_3$ we respectively obtain
\begin{align}
\left( N_1\frac{x_1^2}{y_1} + N_2\frac{x_2^2}{y_2} + N_3\frac{x_3^2}{y_3} \right) x_1 &= N_1\frac{x_1}{y_1} + N_2\frac{x_2}{y_2} (\b_1^\transpose\b_2) + N_3\frac{x_3}{y_3} (\b_1^\transpose \b_3), \label{eq:xyBegin}\\
\left( N_1\frac{x_1^2}{y_1} + N_2\frac{x_2^2}{y_2} + N_3\frac{x_3^2}{y_3} \right) x_2 &= N_1\frac{x_1}{y_1}(\b_2^\transpose \b_1) + N_2\frac{x_2}{y_2} + N_3\frac{x_3}{y_3} (\b_2^\transpose \b_3), \\
\left( N_1\frac{x_1^2}{y_1} + N_2\frac{x_2^2}{y_2} + N_3\frac{x_3^2}{y_3} \right) x_3 &= N_1\frac{x_1}{y_1}(\b_3^\transpose \b_1) + N_2\frac{x_2}{y_2}(\b_3^\transpose \b_2) + N_3\frac{x_3}{y_3}. \label{eq:xyEnd}
\end{align} Since by Lemma \ref{lem:bSpan} $\b^*$ is a linear combination of $\b_1,\b_2,\b_3$, we can assume that $D=3$. Suppose that $\b^* \neq \pm \b_i, \, \forall i \in [n]$. Now, suppose that $x_3=0$. Then we can not have either $x_1=0$ or $x_2=0$, otherwise $\b^*= \b_2$ or $\b^*=\b_1$ respectively. Hence $x_1,x_2 \neq 0$. Then equations 
\eqref{eq:xyBegin}-\eqref{eq:xyEnd} imply that 
\begin{align}
\frac{N_1}{y_1} = \frac{N_2}{y_2} \, \, \, \text{and} \, \, \, x_1^2+x_2^2=1. 
\end{align} Taking into consideration the relations $x_i^2+y_i^2=1$, we deduce that 
\begin{align}
y_1 = \frac{N_1}{\sqrt{N_1^2+N_2^2}}, \, \, \, y_2 = \frac{N_2}{\sqrt{N_1^2+N_2^2}}.
\end{align} Then 
\begin{align}
\J(\b^*) = N_1y_1+N_2y_2+N_3y_3 = \sqrt{N_1^2+N_2^2}+N_3 > \J(\b_1) = N_2+N_3,
\end{align}  which is a contradiction on the optimality of $\b^*$. Similarly, none of the $x_1,x_2$ can be zero, i.e. $x_1,x_2,x_3 \neq 0$. Then equations \eqref{eq:xyBegin}-\eqref{eq:xyEnd} imply that
\begin{align}
x_1^2+x_2^2+x_3^2 = 1, \, \, \,  \frac{N_1}{y_1} = \frac{N_2}{y_2}=\frac{N_3}{y_3}, 
\end{align} which give
\begin{align}
y_i = \frac{N_i \sqrt{2}}{\sqrt{N_1^2+N_2^2+N_3^2}}, \, \, \, i=1,2,3.
\end{align} But then $\J(\b^*) = \sqrt{2 \left(N_1^2+N_2^2+N_3^2 \right)}> \J(\b_1) = N_2+N_3$. This contradiction shows that our hypothesis $\b^* \neq \pm \b_i, \, \forall i \in [n]$ is not valid, i.e., 
$\mathfrak{B}^* \subset \left\{\pm \b_1,\pm \b_2,\pm \b_3 \right\}$. The rest of the theorem follows by comparing the values $\J(\b_i), \, i \in [3]$.

\subsection{Proof of Theorem \ref{prp:equiangular}}
Without loss of generality, we can describe an equiangular arrangement of three hyperplanes of $\Re^D$, with an equiangular arrangement of three planes of $\Re^3$, with normals $\b_1,\b_2,\b_3$ given by 
\begin{align}
\b_1 &:= \mu \begin{bmatrix} 1+\alpha & \alpha & \alpha \end{bmatrix}^\transpose \\ 
\b_2 &:= \mu \begin{bmatrix} \alpha & 1+\alpha & \alpha \end{bmatrix}^\transpose \\
\b_3 &:= \mu \begin{bmatrix} \alpha & \alpha & 1+\alpha \end{bmatrix}^\transpose \\ \label{eq:b-equiangular}
\mu &:=\left[(1+\alpha)^2+2\alpha^2 \right]^{-\frac{1}{2}},
\end{align} with $\alpha$ a positive real number that determines the angle $\theta \in (0,\pi/2]$ of the arrangement, given by 
\begin{align}
\cos (\theta) := \frac{2\alpha(1+\alpha)+\alpha^2}{(1+\alpha)^2+2\alpha^2}= 
\frac{2\alpha+3\alpha^2}{1+2\alpha+3\alpha^2}.
\end{align} Since $N_1=N_2=N_3$, so our objective function essentially becomes
\begin{align}
\J(\b) &= \left(1 - (\b_1^\transpose \b)^2\right)^{\frac{1}{2}} +  \left(1 - (\b_2^\transpose \b)^2\right)^{\frac{1}{2}}+  \left(1 - (\b_3^\transpose \b)^2\right)^{\frac{1}{2}}, \, \, \, \b \in \Sp^2 \\
& = \sin(\theta_1) + \sin(\theta_2) + \sin(\theta_3), \label{eq:ContinuousEquiangular}
\end{align} where $\theta_i$ is the principal angle of $\b$ from $\b_i$. The next Lemma shows that any global minimizer $\b^*$ must have equal principal angles from at least two of the $\b_1,\b_2,\b_3$.

\begin{lem} \label{lem:gPolynomial}
	Let $\b_1,\b_2,\b_3$ be an arrangement of equiangular planes in $\Re^3$, with angle $\theta$ and weights $N_1=N_2=N_3$. Let $\b^*$ be a global minimizer of \eqref{eq:ContinuousProblem} and let $x_i :=\b_i^\transpose \b^*, \, y_i = \sqrt{1-x_i^2} \, i=1,2,3$. Then either $y_1=y_2$ or $y_1=y_3$ or $y_2=y_3$.
\end{lem}
\begin{proof}
	If $\b^*$ is one of $\pm\b_1,\pm\b_2,\pm\b_3$, then the statement clearly holds, since if say $\b^*=\b_1$, then $y_2 = y_3 = \sin(\theta)$. So suppose that $\b^* \neq \pm \b_i, \, \forall i \in [3]$. Then $x_i,y_i$ must satisfy equations \eqref{eq:xyBegin}-\eqref{eq:xyEnd}, together with $x_i^2+y_i^2=1$. Allowing for $y_i$ to take the value zero, the $x_i,y_i$ must satisfy
	\begin{align}
	p_1 &:=x_1y_1y_2y_3+x_2y_3[z-x_1x_2]+x_3y_2[z-x_1x_3] = 0, \\
	p_2 &:=x_1y_3[z-x_1x_2]+x_2y_1y_2y_3+x_3y_1[z-x_2x_3] = 0, \\
	p_3 &:=x_1y_2[z-x_1x_3]+x_2y_1[z-x_2x_3]+x_3y_1y_2y_3 = 0, \\
	q_1 &:= x_1^2+y_1^2-1, \\
	q_2 &:= x_2^2+y_2^2-1, \\
	q_3 &:= x_3^2+y_3^2-1,
	\end{align} where $z:= \cos(\theta)$. Viewing the above system of equations as polynomial equations in the variables 
	$x_1,x_2,x_3,y_1,y_2,y_3,z$, standard Groebner basis \citep{Cox:2007} computations reveal that the polynomial
	\begin{align}
	g:=(1-z)(y_1^2-y_2^2)(y_1^2-y_3^2)(y_2^2-y_3^2)(y_1+y_2+y_3)
	\end{align} \emph{lies in the ideal generated by} $p_i,q_i, \, i=1,2,3$. In simple terms, this means that $\b^*$ must satisfy
	$g(x_i,y_i,z=\cos(\theta)) = 0$. However, the $y_i$ are by construction non-negative and can not be all zero. Moreover, $\theta>0$ so $1-z \neq 0$. This implies that 
	\begin{align}
	(y_1^2-y_2^2)(y_1^2-y_3^2)(y_2^2-y_3^2) = 0,
	\end{align} which in view of the non-negativity of the $y_i$ implies
	\begin{align}
	(y_1-y_2)(y_1-y_3)(y_2-y_3) = 0.
	\end{align}
\end{proof} \noindent The next Lemma says that a global minimizer of $\J(\b)$ is \emph{not far} from the arrangement.

\begin{lem} \label{prp:C1C2C3}
	Let $\b_1,\b_2,\b_3$ be an arrangement of equiangular planes in $\Re^3$, with angle $\theta$ and weights $N_1=N_2=N_3$. Let $\C_i$ be the spherical cap with center $\b_i$ and radius $\theta$. Then any global minimizer of \eqref{eq:ContinuousEquiangular} must lie (up to a sign) either on the boundary or the interior of $\C_1 \cap \C_2 \cap \C_3$.
\end{lem}
\begin{proof} First of all notice that $\b_1,\b_2,\b_3$ lie on the boundary of $\C_1 \cap \C_2 \cap \C_3$.
	Let $\b^*$ be a global minimizer. 
	If $\theta = \pi/2$, we have already seen in Proposition \ref{prp:Orthogonal} that $\b^*$ has to be  one of the vertices $\b_1,\b_2,\b_3$ (up to a sign); so suppose that $\theta<\pi/2$. Let $\theta_i^*$ be the principal angle of $\b^*$ from $\b_i$. Then at least two of $\theta_1^*,\theta_2^*,\theta_3^*$ must be less or equal to $\theta$; for if say $\theta_1^*,\theta_2^* > \theta$, then $\b_3$ would give a smaller objective than $\b^*$. Hence, 
	without loss of generality we may assume that $\theta_1^*,\theta_2^* \le \theta$. In addition, because of Lemma \ref{lem:gPolynomial}, we can further assume without loss of generality that $\theta_1^* = \theta_2^*$. Let $\bzeta$ be the vector in the small arc that joins $\frac{1}{\sqrt{3}} \1$ and $\b_3$ and has angle from $\b_1,\b_2$ equal to $\theta_1^*$. Since $\J(\b^*) \le \J(\bzeta)$, it must be the case that the principal angle $\theta_3^*$ is less or equal to $\theta$ (because the angle of $\bzeta$ from $\b_3$ is $\le \theta$). We conclude that $\theta_1^*,\theta_2^*,\theta_3^* \le \theta$. Consequently, there exist $i \neq j$ such that up to a sign $\b^* \in \C_i \cap \C_j$. Let us assume without loss of generality that $\b^* \in \C_1 \cap \C_2$, i.e., $\theta_1^*, \theta_2^*$ are the angles of $\b^*$ from $\b_1,\b_2$ (notice that now it may no longer be the case that $\theta_1^* = \theta_2^*$).
	
	Notice that the boundaries of $\C_1$ and $\C_2$ intersect at two points: $\b_3$ and its reflection $\tilde{\b}_3$ with respect to the plane $\H_{12}$ spanned by $\b_1,\b_2$.
	In fact, $\H_{12}$ divides $\C_1 \cap \C_2$ in two halves, $\Y,\tilde{\Y}$, with $\Y$ being the reflection of $\tilde{\Y}$ with respect to $\H_{12}$. Letting $\tilde{\C}_3$ be the spherical cap of radius $\theta$ around $\tilde{\b}_3$, we can write
	\begin{align}
	\C_1 \cap \C_2 = (\C_1 \cap \C_2 \cap \C_3) \cup (\C_1 \cap \C_2 \cap \tilde{\C}_3).
	\end{align} If $\b^* \in \C_1 \cap \C_2 \cap \C_3$ we are done, so let us assume that $\b^* \in \C_1 \cap \C_2 \cap \tilde{\C}_3$. 
	Let $\tilde{\b}^*$ be the reflection of $\b^*$ with respect to $\H_{12}$. This reflection preserves the angles from $\b_1$ and $\b_2$. We will show that $\tilde{\b}^*$ has a smaller principal angle $\tilde{\theta}^*_3$ from $\b_3$ than $\b^*$. In fact the spherical angle of $\tilde{\b}^*$ from $\b_3$ is  $\tilde{\theta}_3^*$ itself, and this is precisely the angle of $\b^*$ from $\tilde{\b}_3$. Denote by $H_{3,\tilde{3}}$ the plane spanned by $\b_3$ and $\tilde{\b}_3$, $\bar{\b}^*$ the spherical projection of $\b^*$ onto $H_{3,\tilde{3}}$, $\gamma$
	the angle between $\bar{\b}^*$ and $\b^*$, $\alpha$ the angle between $\bar{\b}^*$ and $\b_3$, and $\tilde{\alpha}$ the angle between $\bar{\b}^*$ and $\tilde{\b}_3$. Then the spherical law of cosines gives
	\begin{align}
	\cos(\tilde{\theta}_3^*) &= \cos(\tilde{\alpha}) \cos(\gamma), \\
	\cos(\theta_3^*) &= \cos(\alpha) \cos(\gamma).
	\end{align} Letting $2 \psi$ be the angle between $\b_3$ and $\tilde{\b}_3$, we have that 
	\begin{align}
	\alpha = \psi + (\psi - \tilde{\alpha}). 
	\end{align} By hypothesis $\tilde{\alpha} < \psi$ and so $\alpha > \psi$. If $2 \psi \le \pi/2$, then $\alpha$ is an acute angle and $\cos(\tilde{\alpha}) > \cos(\alpha) $. If $2 \psi > \pi/2$, then $\cos(\tilde{\alpha}) \le \cos(\alpha) $ only when $\pi - (2\psi-\tilde{\alpha}) \le \tilde{\alpha} \Leftrightarrow \psi \ge \pi/2$. But by construction $\psi \le \pi/2$ and equality is achieved only when $\theta = \pi/2$. Hence, we conclude that $\cos(\tilde{\alpha}) > \left| \cos(\alpha) \right|$, which implies that 
	$\cos(\tilde{\theta}_3) > \left| \cos(\theta_3) \right|$. This in turn means that $\J(\tilde{\b}^*) < \J(\b^*)$, which is a contradiction. 
\end{proof}

\begin{lem} \label{lem:Signs}
	Let $\b_1,\b_2,\b_3$ be an arrangement of equiangular planes in $\Re^3$, with angle $\theta$ and weights $N_1=N_2=N_3$. Let $\b^*$ be a global minimizer of \eqref{eq:ContinuousProblem} and let $x_i :=\b_i^\transpose \b^*, \, i=1,2,3$. Then either $x_1,x_2,x_3$ are all non-negative or they are all non-positive.
\end{lem} 
\begin{proof}
	By Lemma \ref{prp:C1C2C3}, we know that either $\b^* \in \C_1 \cap \C_2 \cap \C_3$ or $-\b^* \in \C_1 \cap \C_2 \cap \C_3$. In the first case, the angles of $\b^*$ from $\b_1,\b_2,\b_3$ are less or equal to $\theta \le \pi/2$.
\end{proof} Now, Lemmas \ref{lem:gPolynomial} and \ref{lem:Signs} show that $\b^*$ is a global minimizer of problem
	\begin{align}
	\min_{\b \in \Sp^2} \left\{\sqrt{1-(\b_1^\transpose \b)^2}+\sqrt{1-(\b_2^\transpose \b)^2}+\sqrt{1-(\b_3^\transpose \b)^2} \right\} \label{eq:ContinuousEquiangularEquiweight3}
	\end{align} if and only if it is a global minimizer of problem
	\begin{align}
	&\min_{\b \in \Sp^2} \left\{\sqrt{1-(\b_1^\transpose \b)^2}+\sqrt{1-(\b_2^\transpose \b)^2}+\sqrt{1-(\b_3^\transpose \b)^2}\right\} , \\
	&\text{s.t.} \, \, \, \b_i^\transpose \b = \b_j^\transpose \b, \, \, \, \text{for some} \, \, \,  i \neq j \in [3]. \label{eq:ContinuousDissection}
	\end{align} So suppose without loss of generality that $\b^*$ is a global minimizer of \eqref{eq:ContinuousDissection} corresponding to indices $i=1,j=2$. Then $\b^*$ lives in the vector space
	\begin{align}
	\V_{12} = \Span \left(\begin{bmatrix}1 \\ 1 \\0 \end{bmatrix}, \begin{bmatrix}0 \\ 0 \\1 \end{bmatrix}\right),
	\end{align} which consists of all vectors that have equal angles from $\b_1$ and $\b_2$. Taking into consideration that $\b^*$ also lies in $\Sp^2$, we have the parametrization
	\begin{align}
	\b^* = \frac{1}{\sqrt{2v^2+w^2}} \begin{bmatrix} v \\ v \\w\end{bmatrix}.
	\end{align} The choice $v=0$, corresponding to $\b^*=\e_3$ (the third standard basis vector), can be excluded, since $\b_3$ always results in a smaller objective: moving $\b$ from $\e_3$ to $\b_3$ while staying in the plane $\V_{12}$ results in decreasing angles of $\b$ from $\b_1,\b_2,\b_3$. Consequently, we can assume $v=1$, and our problem becomes an unconstrained one, with objective
	\begin{align}
	\J(w) = \frac{2\left[(2+w^2)(1+2\alpha+3\alpha^2)-(\alpha w+2\alpha +1)^2\right]^{1/2}+\sqrt{2}\left|a-aw+1\right|}{\left[(2+w^2)(1+2\alpha+3\alpha^2) \right]^{1/2}}. \label{eq:Jw}
	\end{align} Now, it can be shown that:
	\begin{itemize}
		\item The following quantity is always positive
		\begin{align}
		u:=(2+w^2)(1+2\alpha+3\alpha^2)-(\alpha w+2\alpha +1)^2. \label{eq:u}
		\end{align}
		\item The choice $w = 1+1/\alpha$ corresponds to $\b^*=\b_3$, and that is precisely the only point where $\J(w)$ is non-differentiable.
		\item The choice $w=1$ corresponds to $\b^*=\frac{1}{\sqrt{3}} \1$.
		\item The choice $\alpha=1/3$ corresponds to $\theta=60^\circ$.
		\item $\J(\b_3) = \J\left(\frac{1}{\sqrt{3}} \1\right)$ precisely for $\alpha = 1/3$.
	\end{itemize} Since for $\alpha=0$ the theorem has already been proved (orthogonal case), we will assume that $\alpha >0$. We proceed by showing that for $\alpha \in (0,1/3)$ and for $w \neq 1+1/a$, it is always the case that $\J(w)>\J(1+1/a)$. Expanding this last inequality, we obtain 
	\small
	\begin{align}
	\frac{2\left[(2+w^2)(1+2\alpha+3\alpha^2)-(\alpha w+2\alpha +1)^2\right]^{1/2}+\sqrt{2}\left|a-aw+1\right|}{\left[(2+w^2)(1+2\alpha+3\alpha^2) \right]^{1/2}}> \frac{2\sqrt{1+4\alpha+6\alpha^2}}{1+2\alpha+3\alpha^2},
	\end{align}  \normalsize which can be written equivalently as 
	\begin{align}
	p_1&<4\sqrt{2} u^{1/2}\left|\alpha-\alpha w +1 \right|(1+2\alpha+3\alpha^2) , \, \, \, \text{where} \label{eq:p1Inequality} \\
	p_1 &:= 4(2+w^2)(1+4\alpha+6 \alpha^2)-(1+2\alpha+3\alpha^2)\left[4u+2(\alpha-\alpha w+1)^2\right],
	\end{align} and $u$ has been defined in \eqref{eq:u}. Viewing $p_1$ as a polynomial in $w$, $p_1$ has two real roots given by
	\begin{align}
	r_{p_1}^{(1)} := 1+1/\alpha \, \, \, >\, \, \,  r_{p_1}^{(2)} :=\frac{-1-7\alpha+\alpha^2+15\alpha^3}{\alpha(7+22\alpha+15\alpha^2)}.
	\end{align}  Since the leading coefficient of $p_1$ is always a negative function of $\alpha$ (for $\alpha>0$), \eqref{eq:p1Inequality} will always be true for $w \not\in [r_{p_1}^{(2)},r_{p_1}^{(1)} ]$, in which interval $p_1$ is strictly negative. Consequently, we must show that as long as $\alpha \in (0,1/3)$, \eqref{eq:p1Inequality} is true for every $w \in [r_{p_1}^{(2)},r_{p_1}^{(1)} )$. For such $w$, $p_1$ is non-negative and by squaring \eqref{eq:p1Inequality}, we must show that 
	\begin{align}
	p_2>&0, \, \, \, \forall w \in [r_{p_1}^{(2)},r_{p_1}^{(1)} ), \, \, \, \forall \alpha \in (0,1/3), \\
	p_2:=& 32u(\alpha-\alpha w +1)^2(1+2\alpha +3 \alpha^2)-p_1^2.
	\end{align} Interestingly, $p_2$ admits the following factorization
	\begin{align}
	p_2 =& 
	-4 (-1 - \alpha + \alpha w)^2 p_3, \\
	p_3:=&-7 - 18 \alpha - 49 \alpha^2 - 204 \alpha^3 - 441 \alpha^4 - 
	162 \alpha^5 + 81 \alpha^6 \nonumber \\
	&+ (30 \alpha  + 238 \alpha^2  + 612 \alpha^3  + 468 \alpha^4  - 
	162 \alpha^5  - 162 \alpha^6 )w \nonumber \\
	& +(- 8  - 48 \alpha  - 111 \alpha^2  - 
	12 \alpha^3  + 270 \alpha^4  + 324 \alpha^5  + 81 \alpha^6) w^2
	\end{align} The discriminant of $p_3$ is the following $10$-degree polynomial in $\alpha$:
	\begin{align}
	\Delta(p_3)&=32 (-7 - 60 \alpha - 226 \alpha^2 - 312\alpha^3 + 782 \alpha^4 + 5160 \alpha^5 + 13500 \alpha^6 + \nonumber \\
	&+ 21816 \alpha^7 
	+ 22761 \alpha^8 + 14580 \alpha^9 + 4374 \alpha^{10}).
	\end{align} By Descartes rule of signs, $\Delta(p_3)$ has precisely one positive root. In fact this root is equal to $1/3$. Since the leading coefficient of $\Delta(p_3)$ is positive, we must have that 
	$\Delta(p_3)<0, \, \forall \alpha \in (0,1/3)$, and so for such $\alpha$, $p_3$ has no real roots, i.e. it will be either everywhere negative or everywhere positive. Since $p_3(\alpha=1/4,w=1)=-80327/4096$, we conclude that as long as $\alpha \in (0,1/3)$, $p_3$ is everywhere negative and as long as
	$w \neq 1+1/\alpha$, $p_2$ is positive, i.e. we are done. 
	
	Moving on to the case $\alpha =1/3$, we have
	\begin{align}
	p_2(\alpha=1/3,w) = \frac{128}{9}(w-4)^2(w-1)^2,
	\end{align} which shows that for such $\alpha$ the only global minimizers are $\pm \b_3$ and $\pm \frac{1}{\sqrt{3}} \1$.
	
	In a similar fashion, we can proceed to show that $\J(w)>\J\left(\frac{1}{\sqrt{3}} \1\right)$, for all $w \neq 1$ and all $\alpha \in (1/3,\infty)$. However, the roots of the polynomials that arise are more complicated functions of $\alpha$ and establishing the inequality $\J(w)>\J\left(\frac{1}{\sqrt{3}} \1\right)$ analytically, seems intractable; instead this can be done if one allows for numeric computation of polynomial roots.

\subsection{Proof of Theorem \ref{thm:GeneralContinuous}}

We begin with two Lemmas. 

\begin{lem} \label{lem:Jdagger}
	Let $\b_1,\dots,\b_n$ be vectors of $\Sp^{D-1}$, with pairwise principal angles $\theta_{ij}$. Then
	\begin{align}
	\max_{\b^\transpose \b=1}\, \, \, \left[ N_1\left|\b_1^\transpose \b \right|+\cdots+N_n\left|\b_n^\transpose \b \right|\right] \le \left[\sum_{i} N_i^2 + 2 \sum_{i \neq j} N_i N_j\cos(\theta_{ij}) \right]^{1/2}.
	\end{align} 
\end{lem}
\begin{proof}
	Let $\b^{\dagger}$ be a maximizer of $ N_1\left|\b_1^\transpose \b \right|+\cdots+N_n\left|\b_n^\transpose \b \right|$. Then $\b^\dagger$ must satisfy the first order optimality condition, which is 
	\begin{align}
	\lambda^{\dagger} \b^\dagger = \sum_i N_i \Sgn(\b_i^\transpose \b^\dagger) \b_i ,
	\end{align} where $\lambda^\dagger$ is a Lagrange multiplier and $\Sgn(\b_i^\transpose \b^\dagger)$ is the subdifferential of $\left|\b_i^\transpose \b^\dagger \right|$. Then 
	\begin{align}
	\lambda^\dagger \b^\dagger =  N_1 s_1^\dagger \b_1+\cdots+N_n s_n^\dagger \b_n, \label{eq:bDaggerGeneric}
	\end{align} where $s_i^\dagger = \Sign(\b_i^\transpose \b^\dagger)$, if $\b_i^\transpose \b^\dagger \neq 0$, and $s_i^\dagger \in [-1,1]$ otherwise. Recalling that $\left\| \b^\dagger \right\|_2=1$, and taking equality of $2$-norms on both sides of \eqref{eq:bDagger}, we get 
	\begin{align}
	\b^\dagger = \frac{N_1s_1^\dagger \b_1 + \cdots + N_ns_n^\dagger \b_n}{\left\|N_1s_1^\dagger \b_1 + \cdots + N_ns_n^\dagger \b_n \right\|_2}. \label{eq:bDagger}
	\end{align} Now 
	\begin{align}
	\sum_{i} N_i\left|\b_i^\transpose \b^\dagger \right| &=  \sum_{i: \b^\dagger \not\perp \b_i} N_i\left|\b_i^\transpose \b^\dagger \right| = \sum_{i: \b^\dagger \not\perp \b_i} N_is_i^\dagger \b_i^\transpose \b^\dagger = \left(\b^\dagger\right)^\transpose \left( \sum_{i: \b^\dagger \not\perp \b_i} N_is_i^\dagger \b_i\right) \nonumber \\
	&= \left(\b^\dagger\right)^\transpose \left( \sum_{i: \b^\dagger \not\perp \b_i} N_is_i^\dagger \b_i + \sum_{i: \b^\dagger \perp \b_i} N_is_i^\dagger \b_i\right) = \left(\b^\dagger\right)^\transpose \left( \sum_{i} N_is_i^\dagger \b_i \right) \nonumber \\
	&\stackrel{\eqref{eq:bDagger}}{=} \left\|N_1s_1^\dagger \b_1 + \cdots + N_ns_n^\dagger \b_n \right\|_2 = \left[\sum_i \left(s_i^\dagger N_i\right)^2 +2 \sum_{i \neq j} N_i N_j s_i^\dagger s_j^\dagger \b_i^\transpose \b_j \right]^{1/2} \nonumber \\
	&\le \left[\sum_{i} N_i^2 + 2 \sum_{i \neq j} N_i N_j\cos(\theta_{ij}) \right]^{1/2}.
	\end{align}
\end{proof} 

\begin{lem} \label{lem:Jlowerbound}
	Let $\b_1,\dots,\b_n$ be a hyperplane arrangement of $\Re^D$ with integer weights $N_1,\dots,N_n$ assigned. For $\b \in \Sp^{D-1}$, let $\theta_i$ be the principal angle between $\b$ and $\b_i$. Then 
	\begin{align}
	\min_{\b \in \Sp^{D-1}} \sum N_i \sin(\theta_i) \ge \sqrt{\sum N_i^2-\sigma_{\max}\left(\left[N_iN_j\cos(\theta_{ij})\right]_{i,j}
		\right)},
	\end{align} where $\sigma_{\max}\left(\left[N_iN_j\cos(\theta_{ij})\right]_{i,j}\right)$ denotes the maximal eigenvalue of the $n \times n$ matrix, whose 
	$(i,j)$ entry is $N_i N_j \cos(\theta_{ij})$ and $1 \le i,j \le n$.
\end{lem}
\begin{proof}
	For any vector $\bxi$ we have that 
	$\left\|\bxi \right\|_1 \ge \left\|\bxi \right\|_2$. Let $\psi_i \in [0,180^\circ]$ be the angle between $\b$ and $\b_i$. Then 
	\begin{align}
	&\sum N_i \sin(\theta_i) = \sum N_i \left| \sin(\psi_i) \right| \stackrel{\left\|\cdot \right\|_1 \ge \left\|\cdot \right\|_2}{\ge } \sqrt{\sum N_i^2 \sin^2(\psi_i)} \\
	&= \sqrt{\sum N_i^2 - \sum N_i^2 \cos^2(\psi_{i})}.
	\end{align} Hence $N_i \sin(\theta_i)$ is minimized when 
	$\sum N_i^2 \cos^2(\psi_{i})$ is maximized. But
	\begin{align}
	\sum N_i^2 \cos^2(\psi_{i}) = \b^\transpose \left (\sum N_i^2   \b_i \b_i^\transpose \right) \b,
	\end{align} and the maximum value of $\sum N_i^2 \cos^2(\psi_{i})$ is equal to the maximal eigenvalue 
	of the matrix 
	\begin{align}
	\sum N_i^2   \b_i \b_i^\transpose=\begin{bmatrix} N_1 \b_1 & \cdots & N_n \b_n \end{bmatrix} \begin{bmatrix} N_1 \b_1 & \cdots & N_n \b_n \end{bmatrix}^\transpose,
	\end{align} which is the same as the maximal eigenvalue of the matrix
	\begin{align} 
	\begin{bmatrix} N_1 \b_1 & \cdots & N_n \b_n \end{bmatrix}^\transpose \begin{bmatrix} N_1 \b_1 & \cdots & N_n \b_n \end{bmatrix} = \left[N_i N_j \cos(\psi_{ij})\right]_{i,j},
	\end{align} where $\psi_{ij}$ is the angle between $\b_i,\b_j$. Now, if $\bA$ is a matrix and we denote by 
	$\left|\bA \right|$ the matrix that arises by taking absolute values of each element in the matrix $\bA$, then it is known that $\sigma_{\max}(\left|\bA \right|) \ge \sigma_{\max}(\bA)$. Hence the result follows by recalling that $\left|\cos(\psi_{ij})\right| = \cos(\theta_{ij})$.
\end{proof} Now, let $\b^*$ be a global solution of \eqref{eq:ContinuousProblem}.
	Suppose for the sake of a contradiction that $\b^* \not\perp \H_i, \forall i \in [n]$, i.e., $\b^* \neq \pm \b_i, \, \forall i \in [n]$. Consequently, $\J$ is differentiable at $\b^*$ and so $\b^*$ must satisfy \eqref{eq:Optimality}, which we repeat here for convenience:
	\begin{align}
	-\sum_{i=1}^n N_i \left(\b_i^\transpose \b^*\right) 
	\left(1-\left(\b_i^\transpose \b^*\right)^2  \right)^{-\frac{1}{2}} \b_i + \lambda^* \, \b^* = \0. \label{eq:AbstractOptimalityRepeated}
	\end{align}Projecting \eqref{eq:AbstractOptimalityRepeated}  orthogonally onto the hyperplane $\H^*$ defined by $\b^*$ we get
	\begin{align}
	-\sum_{i=1}^n N_i  \, \left(\b_i^\transpose \b^*\right) 
	\left(1-\left(\b_i^\transpose \b^*\right)^2  \right)^{-\frac{1}{2}} \pi_{\H^*} \left(\b_i\right)=\0. \label{eq:AbstractOptimalityProjected}
	\end{align} Since $\b^* \neq \b_i, \, \forall i \in [n]$, it will be the case that $\h_i:=  \pi_{\H^*} \left(\b_i\right) \neq \0, \, \forall i \in [n]$. Since 
	\begin{align}
	\left\|\pi_{\H^*} \left(\b_i\right) \right\|_2 = \left(1-\left(\b_i^\transpose \b^*\right)^2  \right)^{\frac{1}{2}}>0,
	\end{align} equation \eqref{eq:AbstractOptimalityProjected} can be written as
	\begin{align}
	\sum_{i=1}^n N_i \,  \left(\b_i^\transpose \b^*\right) \hat{\h}_i = \0, \label{eq:OptimalityProjectedNormalized}
	\end{align} which in turn gives
	\begin{align}
	N_1 \, \left|\b_1^\transpose \b^* \right| \le \left\|\sum_{i>1} N_i \left(\b_i^\transpose \b^*\right) \hat{\h}_i \right\|_2 \le  \sum_{i>1} N_i \, \left|\b_i^\transpose \b^* \right| \le \max_{\b^\transpose \b=1} \sum_{i>1} N_i \, \left|\b_i^\transpose \b \right| 
	\stackrel{Lem. \ref{lem:Jdagger}}{\le} \beta. \label{eq:theta1thetadagger}
	\end{align}  Since by hypothesis $N_1 > \beta$, we can define an angle $\theta_1^\dagger$ by
	\begin{align}
	\cos(\theta_1^\dagger):= \frac{\beta}{N_1},
	\end{align} and so \eqref{eq:theta1thetadagger} says that $\theta_1$ can not drop below $\theta_1^\dagger$. Hence $\J(\b^*)$ can be bounded from below as follows:
	\begin{align}
	\J(\b^*) &= N_1\sin(\theta_1^*) + \sum_{i>1} N_i\sin(\theta_i^*)  \ge N_1\sin(\theta_1^\dagger)+\min_{\b^\transpose \b=1} \sum_{i>1} N_i\sin(\theta_i) \\
	& \stackrel{Lem. \ref{lem:Jlowerbound}}{\ge} N_1\sin(\theta_1^\dagger) + \sqrt{\sum_{i>1}N_i^2-\sigma_{\max}\left(\left[N_iN_j\cos(\theta_{ij})\right]_{i,j>1}\right)}. \label{eq:AbstractJbstar_FirstBound}
	\end{align} By the optimality of $\b^*$, we must also have $\J(\b_1) \ge \J(\b^*)$, which in view of \eqref{eq:AbstractJbstar_FirstBound} gives
	\begin{align}
	\sum_{i>1} N_i \sin(\theta_{1i}) \ge N_1\sin(\theta_1^\dagger) + \sqrt{\sum_{i>1}N_i^2-\sigma_{\max}\left(\left[N_iN_j\cos(\theta_{ij})\right]_{i,j>1}\right)}.
	\end{align} Now, a little algebra reveals that this latter inequality is precisely the negation of hypothesis $N_1 > \sqrt{\alpha^2 + \beta^2}$. This shows that $\b^*$ has to be $\pm \b_i$, for some $i \in [n]$. For the last statement of the Theorem, notice that condition $\gamma >0$ is equivalent to saying that $\J(\b_1) < \J(\b_i), \, \forall i >1$.

\subsection{Proof of Theorem \ref{thm:DiscreteNonConvex}}

 Let us first derive an upper bound $\theta^{(1)}_{\max}$ on how large $\theta_1^*$ can be. Towards that end, we derive a lower bound on the objective function $\J(\b)$ in terms of $\theta_1$: For any vector $\b \in \Sp^{D-1}$ we can write 
	\begin{align}
	&\J(\b) =\left\| \bX^\transpose \b \right\|_1= \sum \left\|\bX_i^\transpose \b \right\|_1  = \sum N_i \b^\transpose \bchi_{i,\b} \\
	&= \sum c N_i \sin(\theta_i) + \sum N_i\b^\transpose \boldeta_{i,\b}, \, \, \, \left\|\boldeta_{i,\b} \right\|_2 \le \epsilon_i \\
	& \ge c \sum N_i \sin(\theta_i) - \sum \epsilon_i N_i \\
	& = c N_1 \sin(\theta_{1}) + c\sum_{i>1} N_i \sin(\theta_i) - \sum \epsilon_i N_i \\
	& \ge c N_1 \sin(\theta_{1}) + c \min_{\b^\transpose \b =1} \left[\sum_{i>1} N_i \sin(\theta_i)\right]- \sum \epsilon_i N_i \\
	& \stackrel{Lem. \ref{lem:Jlowerbound}}{\ge} c N_1 \sin(\theta_{1}) + c\sqrt{\sum_{i>1}N_i^2-\sigma_{\max}\left(\left[N_iN_j\cos(\theta_{ij})\right]_{i,j>1}\right)} -  \sum \epsilon_i N_i. \label{eq:DiscreteNonConvexJbLB}
	\end{align}	Next, we derive an upper bound on $\J(\b_1)$:
	\begin{align}
	&\J(\b_1) = \sum_{i>1} \left\|\bX_i^\transpose \b_1 \right\|_1 =
	\sum_{i>1} N_i \b_1^\transpose \bchi_{i,\b_1} \\
	&= \sum_{i>1} c N_i \sin(\theta_{1i}) + \sum_{i>1} N_i\b_1^\transpose \boldeta_{i,\b_1}, \, \, \, \left\|\boldeta_{i,\b_1} \right\|_2 \le \epsilon_i \\
	&\le c\sum_{i>1}  N_i \sin(\theta_{1i}) + \sum_{i>1} \epsilon_i N_i. \label{eq:DiscreteNonConvexJb1LB}
	\end{align}
	Since any vector $\b$ for which the corresponding lower bound \eqref{eq:DiscreteNonConvexJbLB} on $\J(\b)$ is strictly larger than the upper bound \eqref{eq:DiscreteNonConvexJb1LB} on $\J(\b_1)$, can not be a global minimizer (because it gives a larger objective than $\b_1$), $\theta_1^*$ must be bounded above by $\theta_{\max}^{(1)}$, where the latter is defined, in view of \eqref{c:thmDiscreteNonConvexDominance}, by
	\begin{align}
	\sin\left(\theta_{\max}^{(1)}\right) := \frac{ \alpha+c^{-1}\left(\epsilon_1N_1+2\sum_{i>1} \epsilon_i N_i\right)}{N_1}, \label{eq:thmDiscreteNonConvexTheta1MAX}
	\end{align} where $\alpha$ is as in Theorem \ref{thm:DiscreteNonConvex}. Now let $\b^*$ be a global minimizer, and suppose for the sake of contradiction that $\b^* \not\perp \H_i, \forall i \in [n]$. We will show that there exists a lower bound $\theta_{\min}^{(1)}$ on $\theta_1$, such that $\theta_{\min}^{(1)} > \theta_{\max}^{(1)}$, which is of course a contradiction.	
	Towards that end, the first order optimality condition for $\b^*$ can be written as 
	\begin{align}
	\0 \in \bX \Sgn(\bX^\transpose \b^*) + \lambda \b^*, \label{eq:FirstOptimality}
	\end{align} where $\lambda$ is a Lagrange multiplier and
	$\Sgn(\alpha) = \Sign(\alpha)$ if $\alpha \neq 0$ and $\Sgn(0) = [-1,1]$, is the subdifferential of the function $|\cdot|$.
	Since the points $\bX$ are general, any hyperplane $\H$ of $\Re^D$ spanned by $D-1$ points of $\bX$ such that at most $D-2$ points come from $\bX_i, \, \forall i \in [n]$, does not contain any of the remaining points of $\bX$. Consequently, by Lemma \ref{lem:NonConvexMaximalInterpolation} $\b^*$ will be orthogonal to precisely $ D-1$ points $\left\{\bxi_1,\dots,\bxi_{D-1}\right\} \subset \bX$, from which at most $K_i \le D-2$ lie in $\H_i$. Thus, we can write relation \eqref{eq:FirstOptimality} as
	\begin{align}
	\sum_{j=1}^{D-1} \alpha_j \bxi_j + \sum_{i=1}^n N_i \, \bchi_{i,\b^*} + \lambda \b^*=\0, \label{eq:OptimalityExpanded}
	\end{align} for real numbers $-1 \le \alpha_j \le 1, \, \forall j \in [D-1]$. Using the definition of $\epsilon_i$, we can write
	\begin{align}
	\bchi_{i,\b^*} = c \, \hat{\h}_{i,\b^*} + \boldeta_{i,\b^*}, \, \forall i \in [n],  \label{eq:Concentration}
	\end{align} with $\left\|\boldeta_{i,\b^*}\right\|_2 \le \epsilon_i$. Note that since 
	$\b^* \not\perp \H_i, \, \forall i \in [n]$, we have $\hat{\h}_{i,\b^*} \neq \0$. 
	Substituting \eqref{eq:Concentration} in \eqref{eq:OptimalityExpanded} we get
	\begin{align}
	\sum_{j=1}^{D-1} \alpha_j \bxi_j + c\sum_{i=1}^n N_i  \, \hat{\h}_{i,\b^*} + \sum_{i=1}^n N_i \, \boldeta_{i,\b^*} + \lambda \b^*=\0, \label{eq:OptimalityConcentration}
	\end{align}  and projecting \eqref{eq:OptimalityConcentration} onto the hyperplane $\H_{\b^*}$ with normal $\b^*$, we obtain
	\begin{align}
	\pi_{\H_{\b^*}}\left(\sum_{j=1}^{D-1} \alpha_j \bxi_j \right) + c\sum_{i=1}^n N_i \pi_{\H_{\b^*}}\left(\hat{\h}_{i,\b^*}\right) + \sum_{i=1}^n N_i \, \pi_{\H_{\b^*}}\left(\boldeta_{i,\b^*}\right) =\0.
	\label{eq:OptimalityConcentrationProjected}
	\end{align} Let us analyze the term $\pi_{\H_{\b^*}}\left(\hat{\h}_{i,\b^*}\right)$. We have 
	\begin{align}
	&\pi_{\H_{\b^*}}\left(\hat{\h}_{i,\b^*}\right)  = \pi_{\H_{\b^*}}\left(\frac{\pi_{\H_i}(\b^*)}{\left\|\pi_{\H_i}(\b^*) \right\|_2}\right) = \pi_{\H_{\b^*}}\left(\frac{\b^*- \left(\b_i^\transpose \b^*\right) \b_i}{\left\|\b^*-\left(\b_i^\transpose \b^*\right) \b_i\right\|_2}\right) \\
	& = \pi_{\H_{\b^*}}\left(\frac{\b^*-\left(\b_i^\transpose \b^*\right) \b_i}{\sin(\theta_i)}\right) = \frac{\b^*-\left(\b_i^\transpose \b^*\right) \b_i}{\sin(\theta_i)} - \left( \frac{1-\cos^2(\theta_i)}{\sin(\theta_i)} \right)\b^* \\
	&= \frac{\left(\b_i^\transpose \b^*\right)\left( \left(\b_i^\transpose \b^*\right)\b^* -\b_i\right)}{\sin(\theta_i)} = - \left(\b_i^\transpose \b^*\right) \hat{\bzeta}_i, \, \, \, \bzeta_i = \pi_{\H_{\b^*}}(\b_i). \label{eq:DoublePi}
	\end{align} Using \eqref{eq:DoublePi}, \eqref{eq:OptimalityConcentrationProjected} becomes 
	\begin{align}
	\pi_{\H_{\b^*}}\left(\sum_{j=1}^{D-1} \alpha_j \bxi_j \right) - \sum_{i=1}^n N_i \, c \, \left(\b_i^\transpose \b^*\right) \hat{\bzeta}_i + \sum_{i=1}^n N_i \, \pi_{\H_{\b^*}}\left(\boldeta_{i,\b^*}\right) =\0.
	\label{eq:OptimalityConcentrationProjectedZeta}
	\end{align} Isolating the term that depends on $i=1$ to the LHS and moving everything else to the RHS, and taking norms, we get
	\begin{align}
	&c \, N_1 \, \cos(\theta_1) \le \sum_{i>1}^n c \, N_i \, \cos(\theta_i) + \nonumber \\
	&+ \left\|\pi_{\H_{\b^*}}\left(\sum_{j=1}^K \alpha_j \bxi_j \right) \right\|_2 + \sum_{i=1}^n N_i \, \left\|\pi_{\H_{\b^*}}\left(\boldeta_{i,\b^*}\right) \right\|_2. \label{eq:OptimalityConcentrationProjectedZetaTheta1Bound}
	\end{align} Since $\left\| \boldeta_{i,\b^*} \right\|_2 \le \epsilon_i$, we have that $ \left\|\pi_{\H_{\b^*}}\left(\boldeta_{i,\b^*}\right) \right\|_2 \le \epsilon_i$. Next, the quantity $\sum_{j=1}^K \alpha_j \bxi_j$ can be decomposed along the index $i$, based on the hyperplane membership of the $\bxi_j$. For instance, if $\bxi_1 \in \H_1$, then replace the term $\alpha_1 \bxi_1$ with $\alpha_1^{(1)} \bxi_1^{(1)}$, where the superscript $\cdot^{(1)}$ denotes association to hyperplane $\H_1$. Repeating this for all $\bxi_j$ and after a possible re-indexing, we have
	\begin{align}
	\sum_{j=1}^{D-1} \alpha_j \bxi_j = \sum_{i=1}^n \sum_{j=1}^{K_i}  \alpha_j^{(i)} \bxi_j^{(i)}.  
	\end{align} Now, by Definition \ref{dfn:DPCPmultiple_circumradius} we have that
	\begin{align}
	\left\|\sum_{j=1}^{K_i} \alpha_j^{(i)} \bxi_j^{(i)} \right\|_2 \le \mathcal{R}_{i,K_i}, 
	\end{align} and as a consequence, the upper bound \eqref{eq:OptimalityConcentrationProjectedZetaTheta1Bound} can be extended to 
	\begin{align}
	c \, N_1 \, \cos(\theta_1) \le \sum_{i>1}^n c \, N_i \, \cos(\theta_i) + \sum_{i} \epsilon_i \, N_i + \mathcal{R}. \label{eq:OptimalityConcProjZetaTheta1BoundCompact}
	\end{align} Finally, Lemma \ref{lem:Jdagger} provides a bound 
	\begin{align}
	\sum_{i>1}^n \, N_i \, \cos(\theta_i)  \le \beta,
	\end{align} where $\beta$ is as in Theorem \ref{thm:GeneralContinuous}. In turn, this can be used to extend \eqref{eq:OptimalityConcProjZetaTheta1BoundCompact} to 
	\small
	\begin{align}
	\cos(\theta_1) \le \frac{\beta+c^{-1}\left(\mathcal{R}+\sum  \epsilon_i  N_i\right)}{N_1}=: \cos\left(\theta_{\min}^{(1)}\right) \label{eq:thmDiscreteNonConvexTheta1MIN}. 
	\end{align} \normalsize Note that the angle $\theta_{\min}^{(1)}$ of 
	\eqref{eq:thmDiscreteNonConvexTheta1MIN} is well-defined, since by hypothesis $N_1 > \bar{\beta}$, and that what \eqref{eq:thmDiscreteNonConvexTheta1MIN} effectively says, is that $\theta_1$ never drops below $\theta_{\min}^{(1)}$. It is then straightforward to check that hypothesis $N_1 > \sqrt{\bar{\alpha}^2+\bar{\beta}^2}$ implies $\theta_{\min}^{(1)} > \theta_{\max}^{(1)}$, which is a contradiction. In other words, $\b^*$ must be equal up to sign to one of the $\b_i$, which proves the first part of the Theorem. The second part follows from noting that condition $\bar{\gamma}>0$ guarantees that $\J(\b_1) < \min_{i>1} \J(\b_i)$.

\subsection{Proof of Theorem \ref{thm:LPDiscrete}}

	First of all, it follows from the theory of the simplex method, that if $\bn_{k+1}$ is obtained via the simplex method, then it will satisfy the conclusion of Lemma \ref{lem:LPMaximalInterpolation} in Appendix \ref{appendix:Spath}. Then Lemma \ref{lem:LPconvergence} guarantees that $\left\{\bn_k\right\}$ converges to a critical point of 
	problem \eqref{eq:ell1} in a finite number of steps; denote that point by $\bn_{\infty}$. In other words, $\bn_{\infty}$ will satisfy equation \eqref{eq:AbstractOptimalityRepeated} and it will have unit $\ell_2$ norm. Now, if $\bn_{\infty} = \pm \b_j$ for some $j>1$, then 
	\begin{align}
	\J(\hat{\bn}_0) \ge \J(\b_j), 
	\end{align} or equivalently 
	\begin{align}
	\sum N_i \hat{\bn}_0^\transpose \bchi_{i,\hat{\bn}_0} \ge 
	\sum_{i \neq j} N_i \b_j^\transpose \bchi_{i,\b_j}. \label{eq:LP_not_bi}
	\end{align} Substituting the concentration model 
	\begin{align}
	\bchi_{i,\hat{\bn}_0} &= c \widehat{\pi_{\H_i}\left(\bn_0\right)} + \boldeta_{i,0}, \, \, \, \left\|\boldeta_{i,0} \right\|_2 \le \epsilon_i, \\
	\bchi_{i,\b_j} &= c \widehat{\pi_{\H_i}\left(\b_j\right)} + \boldeta_{ij}, \, \, \, \left\|\boldeta_{ij} \right\|_2 \le \epsilon_i,
	\end{align} into \eqref{eq:LP_not_bi}, we get 
	\begin{align}
	\sum N_i c \sin(\theta_{i,0}) + \sum N_i \hat{\bn}_0^\transpose \boldeta_{i,0} \ge \sum_{i \neq j} N_i c \sin(\theta_{ij}) + \sum N_i \b_j^\transpose \boldeta_{ij}. \label{eq:LP_not_bi_concentration}
	\end{align} Bounding the LHS of \eqref{eq:LP_not_bi_concentration} from above and the RHS from below, we get
	\begin{align}
	\sum N_i \, c \, \sin(\theta_{i,0}) + \sum \epsilon_i N_i  \ge \sum_{i \neq j} N_i \, c \, \sin(\theta_{ij}) - \sum \epsilon_i N_i.
	\end{align} But this very last relation is contradicted by hypothesis $N_1 > \mu$, i.e., none of the $\pm \b_j$ for $j>1$ can be $\bn_{\infty}$. We will show that $\bn_{\infty}$ has to be $\pm \b_1$. So suppose for the sake of a contradiction that that $\bn_{\infty}$ is not colinear with $\b_1$, i.e., $\bn_{\infty} \not\perp \H_i, \, \forall i \in [n]$. Since $\bn_{\infty} $ satisfies \eqref{eq:AbstractOptimalityRepeated}, we can use part of the proof of Theorem \ref{thm:DiscreteNonConvex}, according to which the principal angle $\theta_{1,\infty}$ of $\bn_{\infty}$ from $\b_1$ does not become less than $\theta_{\min}^{(1)}$, where $\theta_{\min}^{(1)}$ is as in \eqref{eq:thmDiscreteNonConvexTheta1MIN}. Consequently, and using once again the concentration model, we obtain
	\begin{align}
	\sum N_i \, c \, \sin(\theta_{i,0}) + \sum \epsilon_i N_i \ge \J(\hat{\bn}_0) \ge \J(\bn_{\infty}) \ge \sum N_i \, c \, \sin(\theta_{i,\infty}) - \sum \epsilon_i N_i \nonumber \\ 
	\ge N_1 \, c \, \sin\left(\theta_{\min}^{(1)} \right) + c \, \sqrt{\sum_{i>1}N_i^2-\sigma_{\max}\left(\left[N_iN_j\cos(\theta_{ij})\right]_{i,j>1}\right)} - \sum \epsilon_i N_i. \label{eq:LP_b1}
	\end{align} Now, a little algebra reveals that the outermost inequality in \eqref{eq:LP_b1} contradicts \eqref{c:thmLPDiscreteDominance}.

\section{Algorithmic Contributions} \label{section:Algorithms}
There are at least two ways in which DPCP can be used to learn a hyperplane
arrangement; either through a sequential (RANSAC-style) scheme, or through 
an iterative (K-Subspaces-style) scheme. These two cases are described next.

\subsection{Sequential hyperplane learning via DPCP} \label{subsection:MultipleDPCPAlgorithmsRANSACstyle}

Since at its core DPCP is a single subspace learning method, we may as well use it to learn 
$n$ hyperplanes in the same way that RANSAC \citep{RANSAC} is used: learn one hyperplane from the entire dataset, remove the points close to it, then learn a second hyperplane and so on. The main weakness of this 
technique is well known, and consists of its sensitivity to the thresholding parameter, which is necessary 
in order to remove points. 
To alleviate the need of knowing a good threshold, we propose to replace the process of removing points
by a process of appropriately weighting the points. In particular, suppose we solve the DPCP problem
\eqref{eq:ell1} on the entire dataset $\bX$ and obtain a unit $\ell_2$-norm vector $\b_1$. Now, instead of
removing the points of $\bX$ that are close to the hyperplane with normal vector $\b_1$ (which would require a threshold parameter), we weight each and every point $\x_j$ of $\bX$ by its distance $\left| \b_1^\transpose \x_j \right|$ from that hyperplane. Then to compute a second hyperplane with normal $\b_2$ we apply DPCP on the weighted dataset $\left\{\left| \b_1^\transpose \x_j \right| \x_j \right\}$. To compute a third hyperplane, 
the weight of point $\x_j$ is chosen as the smallest distance of $\x_j$ from the already computed two hyperplanes, i.e., DPCP is now applied to $\left\{\min_{i=1,2} \left| \b_i^\transpose \x_j \right| \x_j \right\}$. After $n$ hyperplanes have been computed, the clustering of the points is obtained based on their distances to the $n$ hypeprlanes; see Algorithm \ref{alg:MultipleDPCPsequential}.

\begin{algorithm}[t!] \caption{Sequential Hyperplane Learning via DPCP}\label{alg:MultipleDPCPsequential} \begin{algorithmic}[1] 
		\Procedure{SHL-DPCP}{$\bX=\left[\x_1, \, \x_2, \dots, \, \x_N \right] \in \Re^{D \times N},n$}	
		\State $ i \gets 0$;
		\State $w_j \gets 1, \, j=1,\dots,N$; 		
		\For{$i=1 : n$}	
		      \State $\bY \gets \left[ w_1 \x_1 \, \cdots \, w_N \x_N \right]$;		
			  \State $\b_i \gets \argmin_{\b \in \Re^D} \, \, \left[ \left\| \bY^\transpose \b \right\|_1, \, \, \, \text{s.t.} \, \, \, \b^\transpose \b = 1\right]$; 			 
			  \State $w_j \gets \min_{k=1,\dots,i} \left| \b_k^\transpose \x_j \right|, \, j=1,\dots,N$;
			  		\EndFor			
		\State $\bC_i \gets \left\{\x_j \in \bX: \, i = \argmin_{k=1,\dots,n}  \left| \b_k^\transpose \x_j \right| \right\}, i=1,\dots,n$;
		\State \Return $\left\{(\b_i,\bC_i)\right\}_{i=1}^n$;
		\EndProcedure 				
	\end{algorithmic} 
\end{algorithm}

\subsection{Iterative hyperplane learning via DPCP} \label{subsection:AlgorithmsKHstyle}
Another way to do hyperplane clustering via DPCP, is to modify the classic K-Subspaces \citep{Bradley:JGO00,Tseng:JOTA00,Zhang:WSM09} by computing the normal vector of each cluster by DPCP. We call the resulting method IHL-DPCP; see Algorithm \ref{alg:IHL-DPCP}. It is worth noting
that since DPCP minimizes the $\ell_1$-norm of the distances of the points to a hyperplane, consistency dictates that the stopping criterion for IHL-DPCP be governed by the sum over all points of the distance of each point to its
assigned hyperplane (instead of the traditional sum of squares \citep{Bradley:JGO00,Tseng:JOTA00}); in other words the global objective function minimized by
IHL-DPCP is the same as that of Median K-Flats \citep{Zhang:WSM09}. 

\begin{algorithm}[t!] \caption{Iterative Hyperplane Learning via Dual Principal Component Pursuit}\label{alg:IHL-DPCP} \begin{algorithmic}[1] 
		\Procedure{IHL-DPCP}{$\bX=\left[\x_1, \, \x_2, \dots, \, \x_N \right],\b_1,\dots,\b_n,\varepsilon,T_{\max}$}	
		\State $\J_{\text{old}} \gets \infty, \,  \Delta \J \gets \infty, \, t \gets 0$;
		\While{ $t < T_{\max}$ and $\Delta \J > \varepsilon$} 
		\State $\J_{\text{new}} \gets 0, \, t = t+1$;		
			\State $\bC_i \gets \left\{\x_j \in \bX: \, i = \argmin_{k=1,\dots,n}  \left| \b_k^\transpose \x_j \right| \right\}, i=1,\dots,n$;
			\State $\J_{\text{new}} = \sum_{i=1}^n \sum_{\x_j \in \C_i} \left|\b_i^\transpose \x_j \right|$;			
		\State $\Delta \J \gets 
		                      (\J_{\text{old}} - \J_{\text{new}}) / (\J_{\text{old}} + 10^{-9}), \, \J_{\text{old}} \gets \J_{\text{new}}$ ;		
				\State $\b_i \gets \argmin_{\b } \, \, \left[ \left\| \C_i^\transpose \b \right\|_1, \, \, \, \text{s.t.} \, \, \, \b^\transpose \b = 1\right], \, i=1,\dots,n$; 		
		\EndWhile				
		\State \Return $\left\{(\b_i,\bC_i)\right\}_{i=1}^n$;
		\EndProcedure 				
	\end{algorithmic} 
\end{algorithm}	

\subsection{Solving the DPCP problem} \label{subsection:DPCPcomputation}

\begin{algorithm}[t!] \caption{Relaxed Dual Principal Component Pursuit}\label{alg:DPCP-r} \begin{algorithmic}[1] 
		\Procedure{DPCP-r}{$\bX,\varepsilon,T_{\max}$}				
		\State $k \gets 0; \Delta \mathcal{J} \gets \infty$; 
		\State $\hat{\bn}_0 \gets \argmin_{ \left\|\b\right\|_2=1} \left\|\bX^\transpose \b \right\|_2$;
		\While{$k < T_{\max}$ and $\Delta \mathcal{J} > \varepsilon$}
		\State $k \gets k+1$;
		\State $\bn_k \gets \argmin_{\b^\transpose \hat{\bn}_{k-1}=1} \left\| \bX^\transpose \b \right\|_1$;
		\State $\Delta \mathcal{J} \gets \left(\left\| \bX^\transpose \hat{\bn}_{k-1} \right\|_1-\left\| \bX^\transpose \hat{\bn}_{k} \right\|_1 \right) / \left(\left\| \bX^\transpose \hat{\bn}_{k-1} \right\|_1+10^{-9}\right)$;
		\EndWhile				
		\State \Return $\hat{\bn}_{k}$;
		\EndProcedure 				
	\end{algorithmic} 
\end{algorithm} 

\begin{algorithm}[t!] \caption{Dual Principal Component Pursuit via Iteratively Reweighted Least Squares}\label{alg:DPCP-IRLS} \begin{algorithmic}[1] 
		\Procedure{DPCP-IRLS}{$\bX,c,\varepsilon,T_{\max},\delta$}				
		\State $k \gets 0; \Delta \mathcal{J} \gets \infty$; 
		\State $\bB_0 \gets \argmin_{ \bB \in \Re^{D \times c}} \, \, \left\|\bX^\transpose \bB \right\|_2, \, \, \, \text{s.t.} \, \, \, \bB^\transpose \bB = \I_c$;
		\While{$k < T_{\max}$ and $\Delta \mathcal{J} > \varepsilon$}
		\State $k \gets k+1$;	
		\State $w_{\x} \gets 1/\max\left\{\delta, \left\|\bB_{k-1}^\transpose \x\right\|_2 \right\}, \, \x \in \bX$;
				\State $\bB_k \gets \argmin_{\bB \in \Re^{D \times c}} \, \, \sum_{\x \in \bX}  w_{\x}\left\|\bB^\transpose \x\right\|_2^2   \, \, \, \text{s.t.} \, \, \, \bB^\transpose \bB = \I_c$;
		\State $\Delta \mathcal{J} \gets \left(\left\| \bX^\transpose \bB_{k-1} \right\|_1-\left\| \bX^\transpose \bB_{k} \right\|_1\right) / \left(\left\| \bX^\transpose \bB_{k-1} \right\|_1+10^{-9}\right)$;
		\EndWhile				
		\State \Return $\bB_{k}$;
		\EndProcedure 				
	\end{algorithmic} 
\end{algorithm}

\begin{algorithm} \caption{Denoised Dual Principal Component Pursuit 
 }\label{alg:DPCP-d} \begin{algorithmic}[1] 
		\Procedure{DPCP-d}{$\bX,\varepsilon,T_{\max},\delta,\tau$}				
		\State Compute a Cholesky factorization $\bL \bL^\transpose  = \bX \bX^\transpose+\delta \bI_{D}$;			\State $k \gets 0; \Delta \mathcal{J} \gets \infty$; 
		\State $\b \gets \argmin_{\b \in \Re^{D}: \, \left\|\b\right\|_2=1} \, \, \, \left\|\bX^\transpose \b \right\|_2$;	
		\State $\J_0 \gets \tau \left\|\bX^\transpose \b \right\|_1$;		
		\While{$k < T_{\max}$ and $\Delta \mathcal{J} > \varepsilon$}
		\State $k \gets k+1$;	
		\State $\y \gets \cS_{\tau} \left(\bX^\transpose \b \right)$;
		\State $\b \gets$ solution of $\bL \bL^\transpose \bxi = \bX \y$ by backward/forward propagation;
		\State $\b \gets \b / \left\| \b \right\|_2$;
		\State $\J_k \gets \tau \left\| \y \right\|_1 + \frac{1}{2}  \left\|\y - \bX^\transpose \b \right\|_2^2$;
		\State $\Delta \J \gets \left(\J_{k-1} - \J_k \right) / \left(\J_{k-1} + 10^{-9} \right)$;				
		\EndWhile				
		\State \Return $(\y,\b)$;
		\EndProcedure 				
	\end{algorithmic} 
\end{algorithm}

 Recall that the DPCP problem \eqref{eq:ell1} that appears in Algorithms \ref{alg:MultipleDPCPsequential} and \ref{alg:IHL-DPCP} (with data matrices $\bY$ and $\bC_i$, respectively)
is non-convex. In \cite{Tsakiris:DPCP-ArXiv17} we described four distinct methods for solving it,
which we briefly review here. 

The first method, which was first proposed in \cite{Spath:Numerische87}, consists of solving 
the recursion of linear programs \eqref{eq:ConvexRelaxations} using any standard solver, 
such as Gurobi \citep{gurobi}; we refer to such a method as DPCP-r, standing for \emph{relaxed DPCP} (see Algorithm \ref{alg:DPCP-r}). A second approach, called DPCP-IRLS, is to solve \eqref{eq:ell1} using a standard 
 \emph{Iteratively Reweighted Least-Squares (IRLS)} technique (\citep{Candes:JFAA08,Daubechies:CPAM10,Chartrand:ICASSP08}) as in Algorithm \ref{alg:DPCP-IRLS}.
 A third method, first proposed in \cite{Qu:NIPS14}, is to solve \eqref{eq:ell1} approximately by applying alternative minimization on 
 its \emph{denoised} version
 \begin{align}
\min_{\b,\y: \, ||\b||_2=1} \, \left[\tau  \, \left\|\y\right\|_1 + \frac{1}{2}\left\|\y - \bX^\transpose \b \right\|_2^2  \right] \label{eq:DPCP-d}.
\end{align} We refer to such a method as \emph{DPCP-d}, standing for \emph{denoised DPCP}; see Algorithm 
\ref{alg:DPCP-d}. Finally, the fourth method is \emph{relaxed and denoised DPCP (DPCP-r-d)}, which 
replaces each problem of recursion \eqref{eq:ConvexRelaxations} with its denoised version
\begin{align}
\min_{\y,\b} \left[ \tau \left\|\y \right\|_1 + \frac{1}{2}\left\|\y - \bX^\transpose \b \right\|_2^2 \right], \, \, \, \text{s.t.} \, \, \, \b^\transpose \hat{\bn}_{k-1}=1; \label{eq:LassoFirst}
\end{align} which is in turn solved via alternating minimization; see \cite{Tsakiris:DPCP-ArXiv17} for details.
 
\section{Experimental evaluation} \label{section:Experiments}
In this section we evaluate experimentally Algorithms \ref{alg:MultipleDPCPsequential}
and \ref{alg:IHL-DPCP} using both synthetic (\S \ref{subsection:ExperimentsSynthetic})
 and real data (\S \ref{subsection:ExperimentsReal}).

\subsection{Synthetic data} \label{subsection:ExperimentsSynthetic}

\begin{figure}[t!]
	\centering
	\subfigure[RANSAC, $\alpha=1$]{\label{figure:RANSACstyle_D_vs_n_Ns300_T50_RANSAC1}\includegraphics[width=0.29\linewidth]{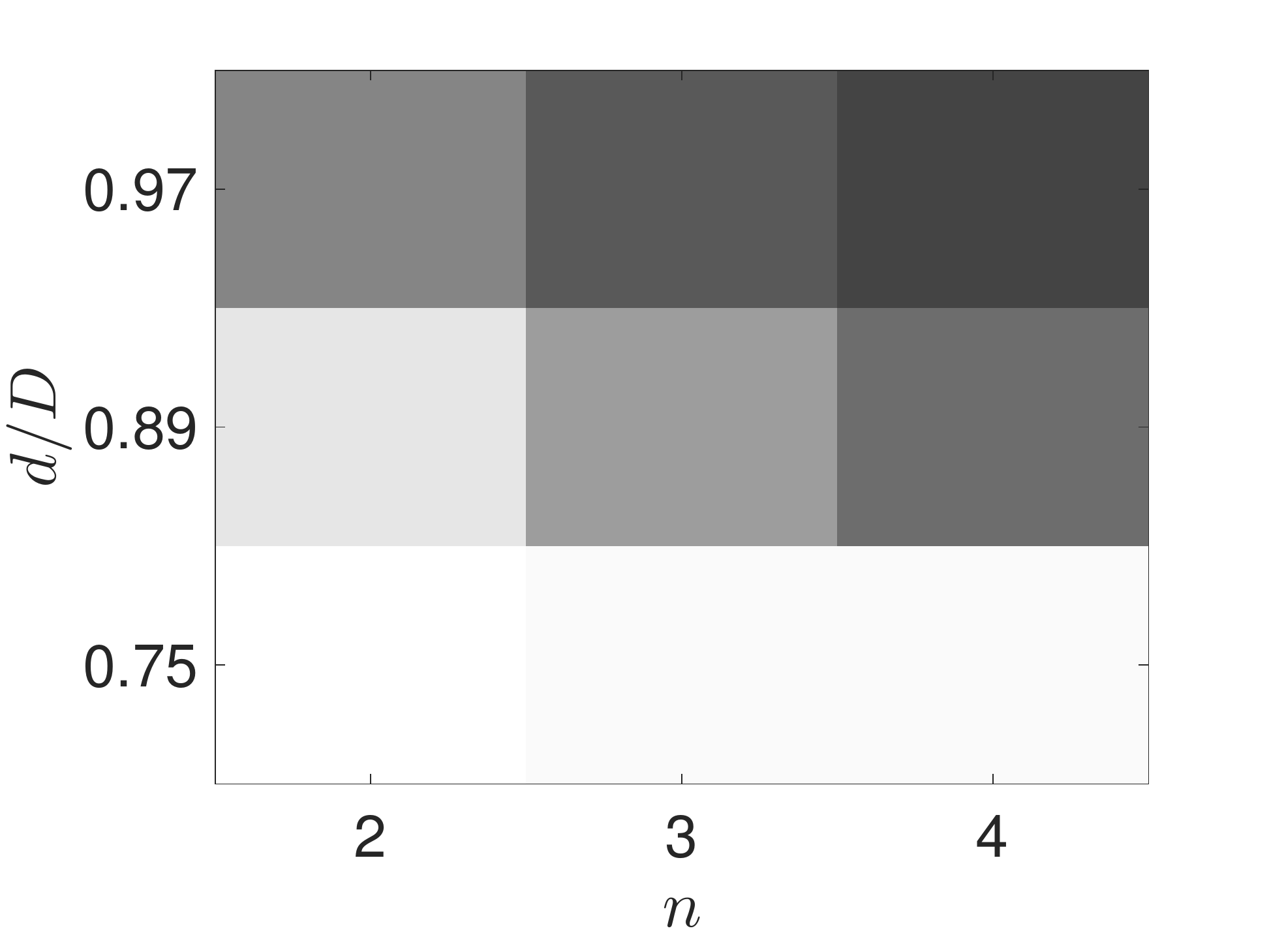}} 
	\subfigure[RANSAC, $\alpha=0.8$]{\label{figure:RANSACstyle_D_vs_n_Ns300_T50_RANSAC08}\includegraphics[width=0.29\linewidth]{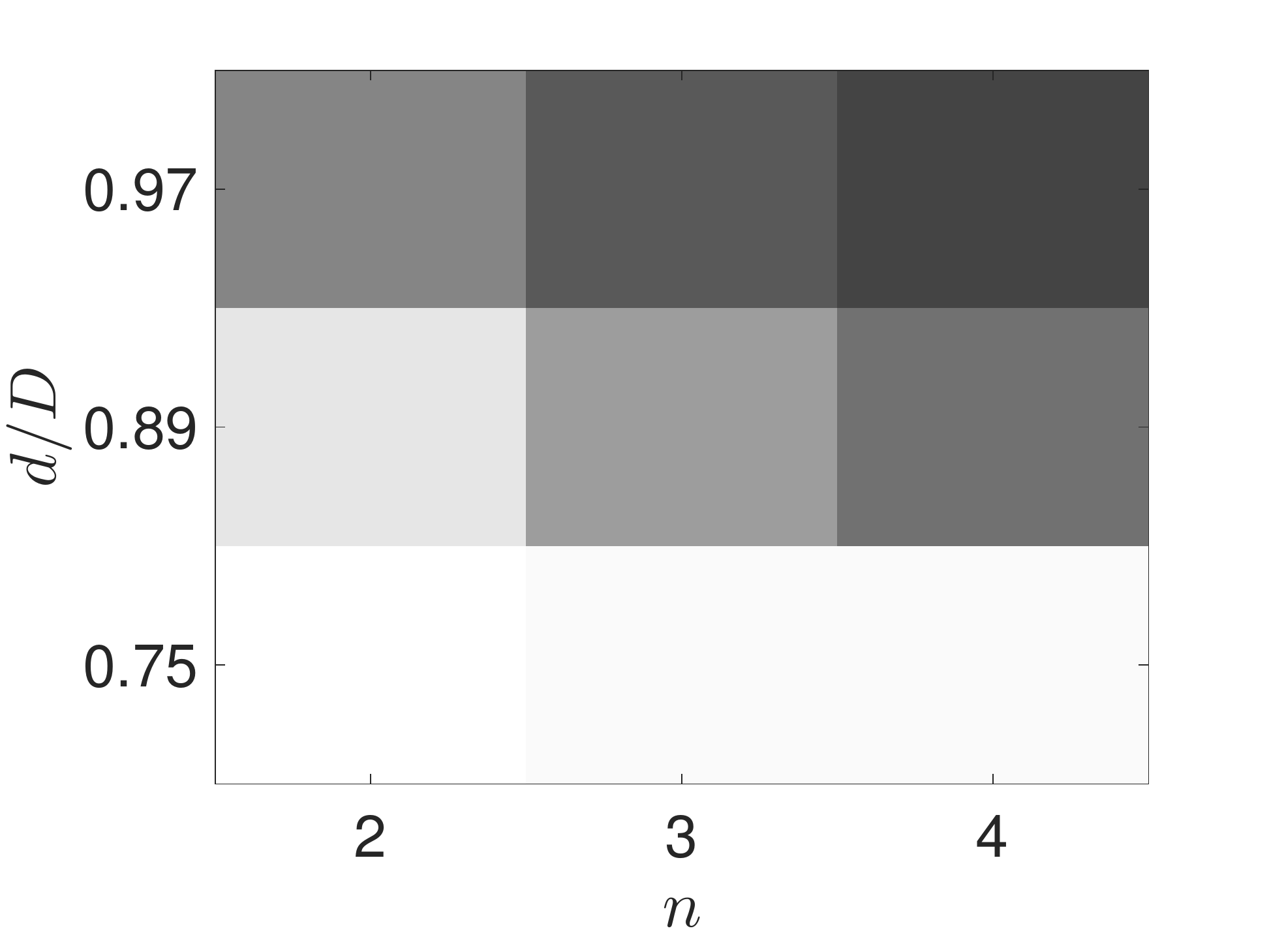}} 	
	\subfigure[RANSAC, $\alpha=0.6$]{\label{figure:RANSACstyle_D_vs_n_Ns300_T50_RANSAC06}\includegraphics[width=0.29\linewidth]{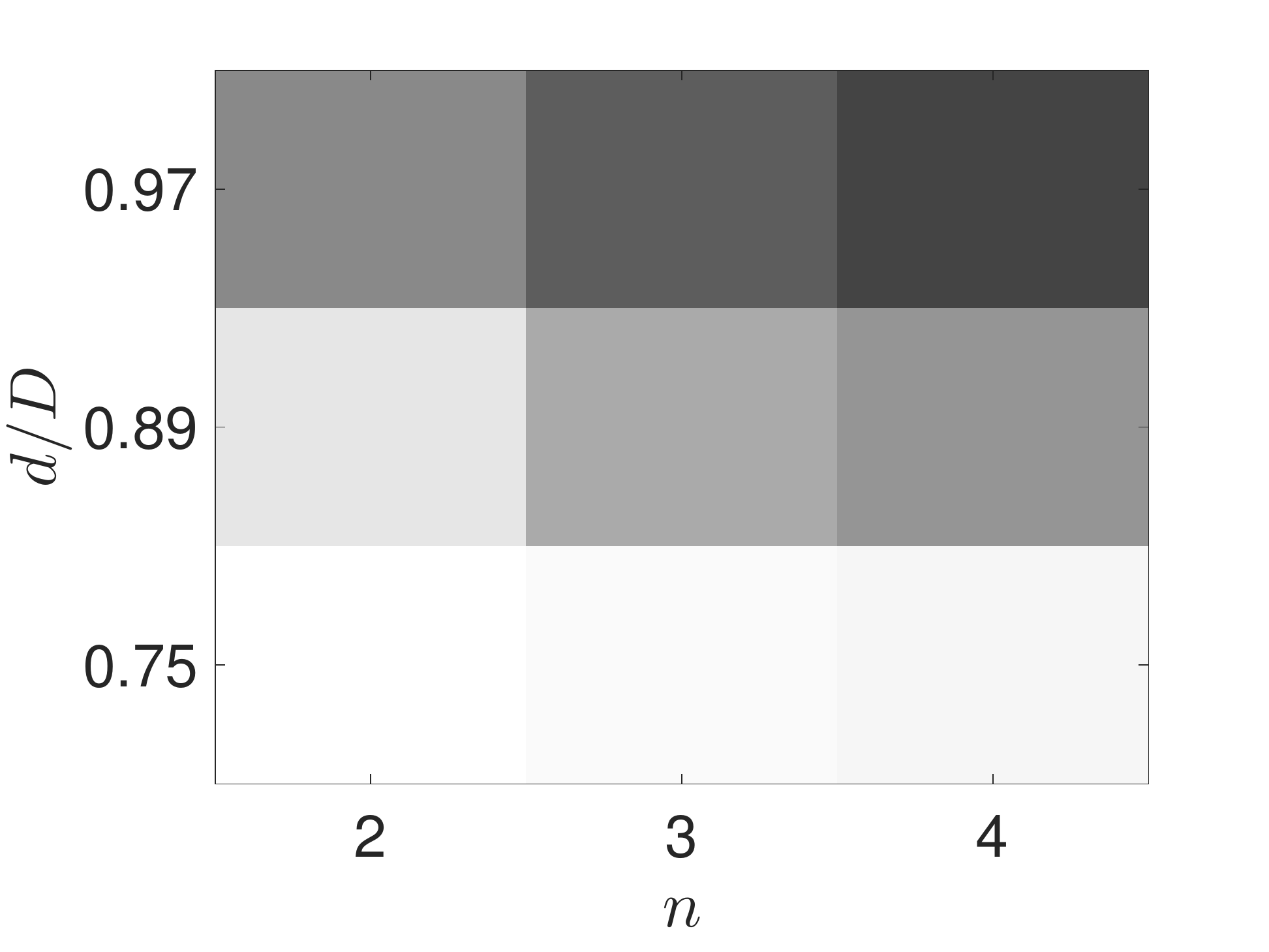}} 
	\subfigure[REAPER, $\alpha=1$]{\label{figure:RANSACstyle_D_vs_n_Ns300_T50_REAPER1}\includegraphics[width=0.29\linewidth]{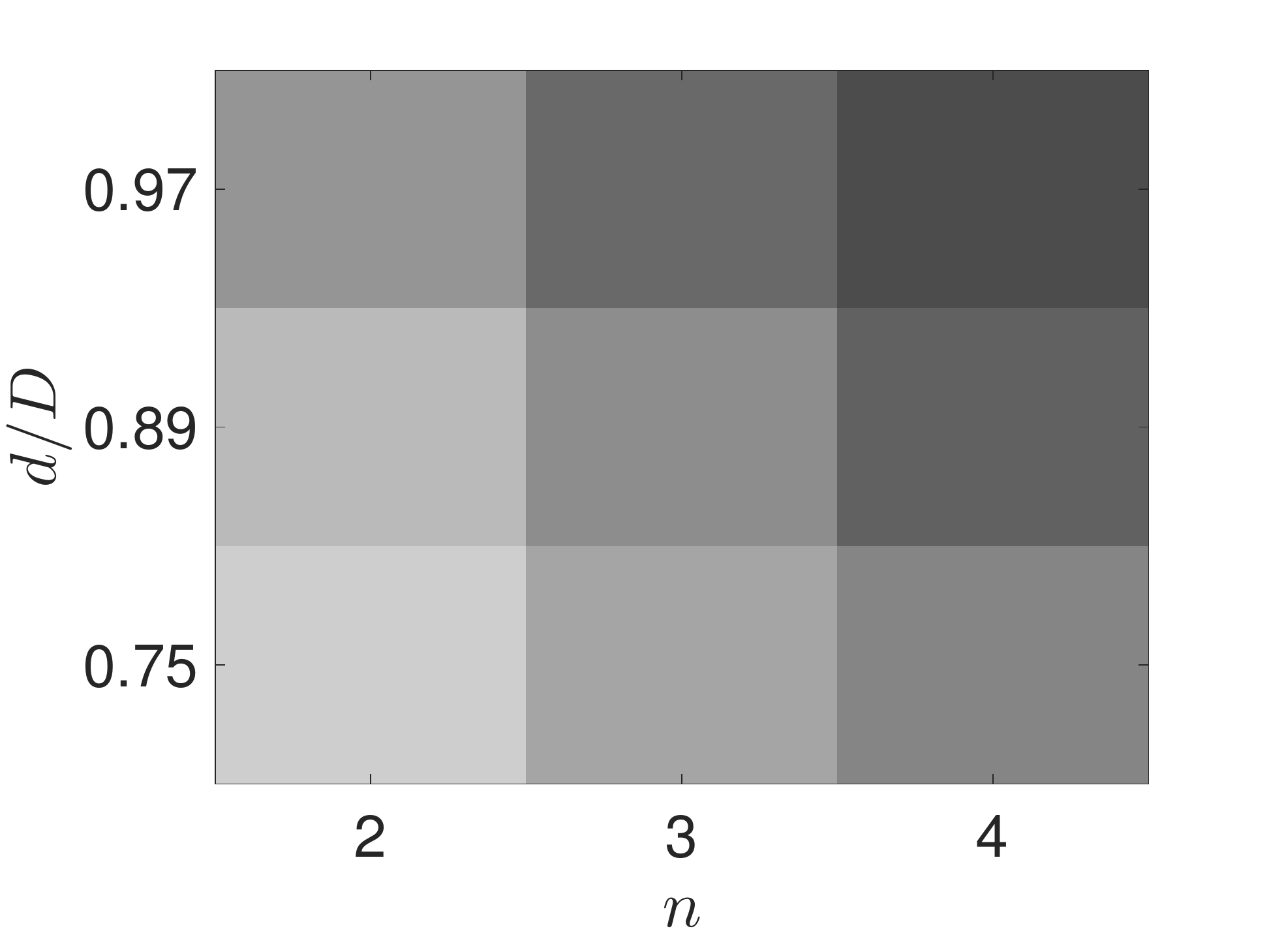}} 
	\subfigure[REAPER, $\alpha=0.8$]{\label{figure:RANSACstyle_D_vs_n_Ns300_T50_REAPER08}\includegraphics[width=0.29\linewidth]{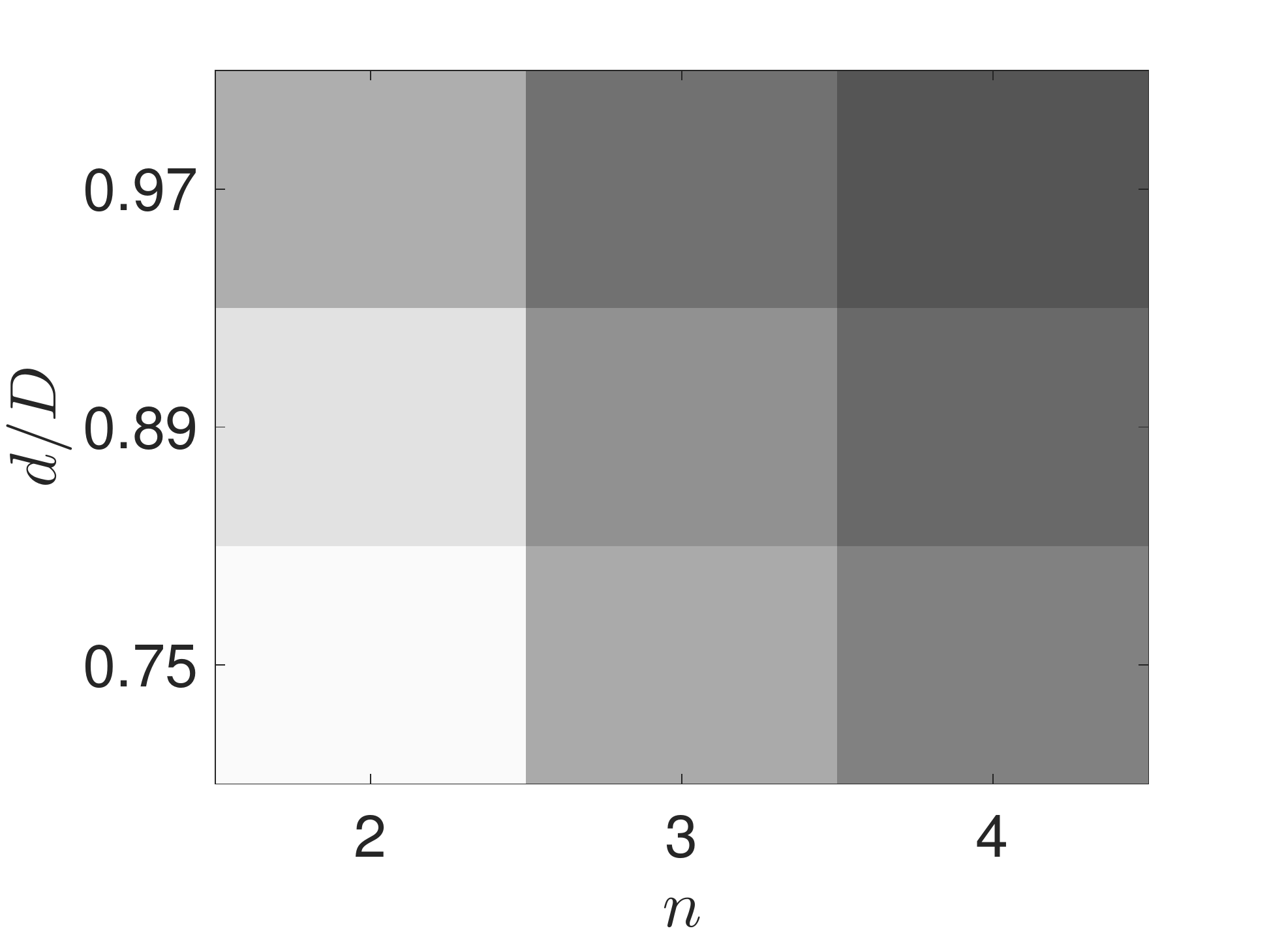}} 	
	\subfigure[REAPER, $\alpha=0.6$]{\label{figure:RANSACstyle_D_vs_n_Ns300_T50_REAPER06}\includegraphics[width=0.29\linewidth]{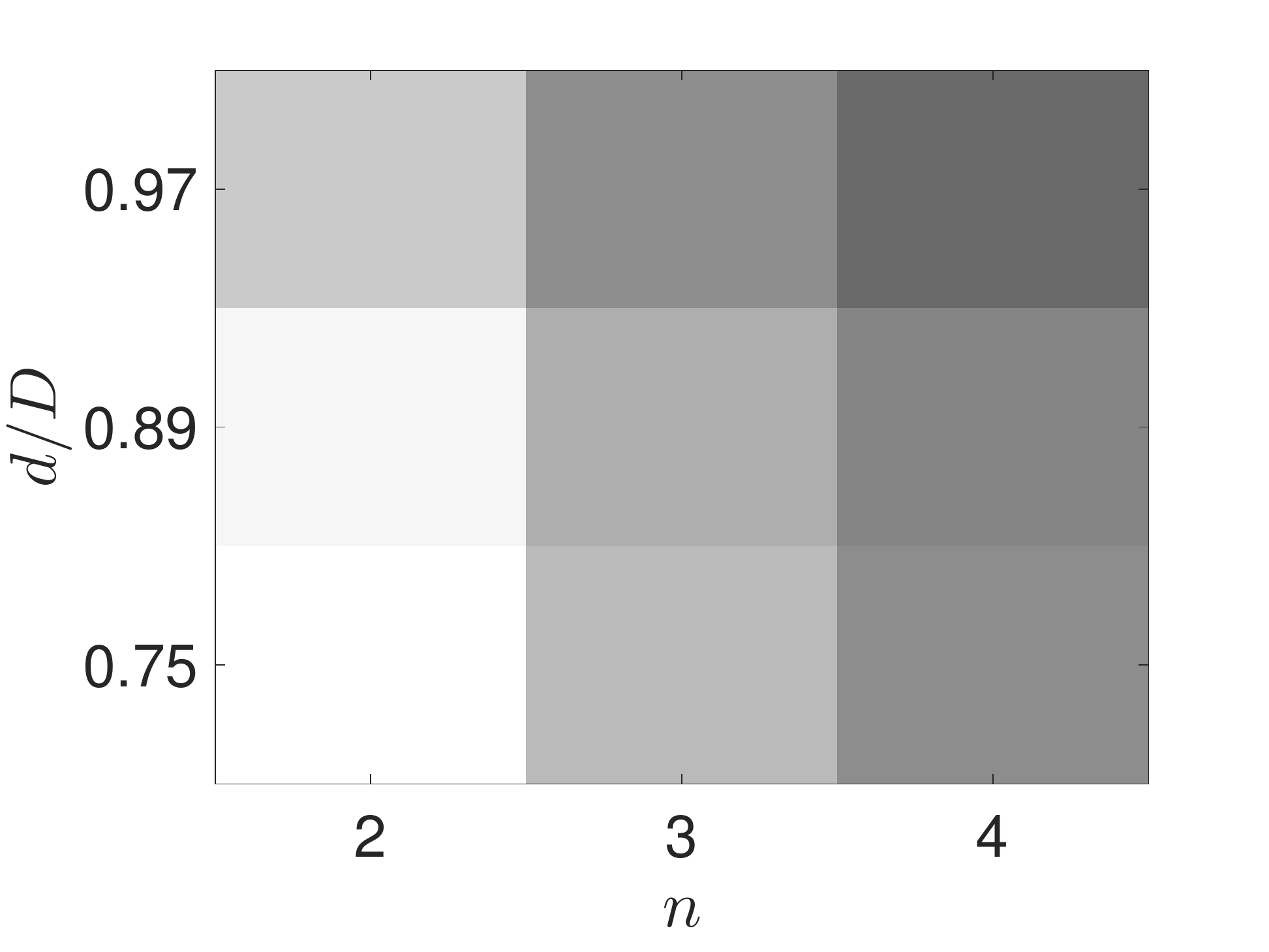}} 
	\subfigure[DPCP-r, $\alpha=1$]{\label{figure:RANSACstyle_D_vs_n_Ns300_T50_DPCP-r1}\includegraphics[width=0.29\linewidth]{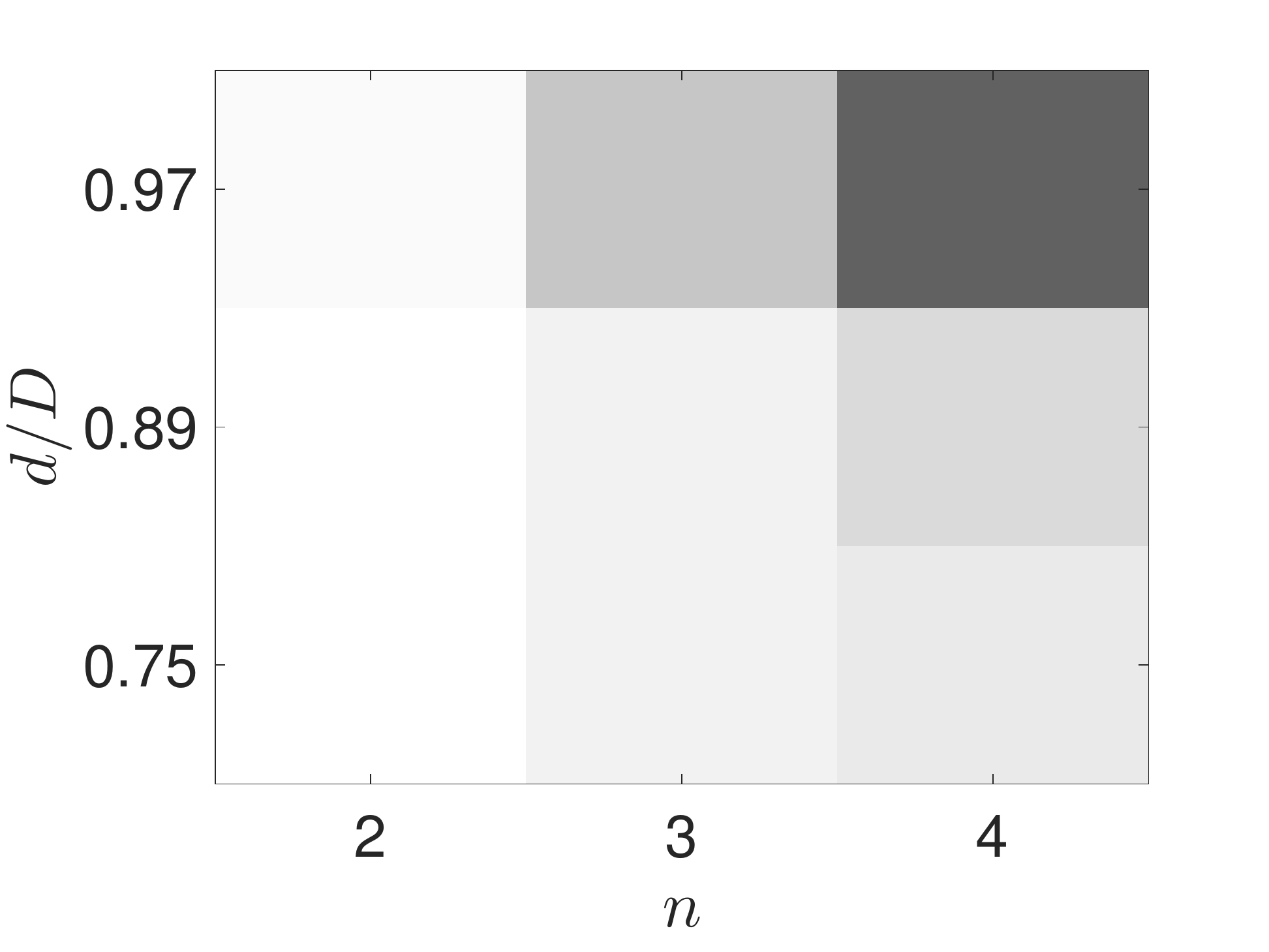}} 
	\subfigure[DPCP-r, $\alpha=0.8$]{\label{figure:RANSACstyle_D_vs_n_Ns300_T50_DPCP-r
	08}\includegraphics[width=0.29\linewidth]{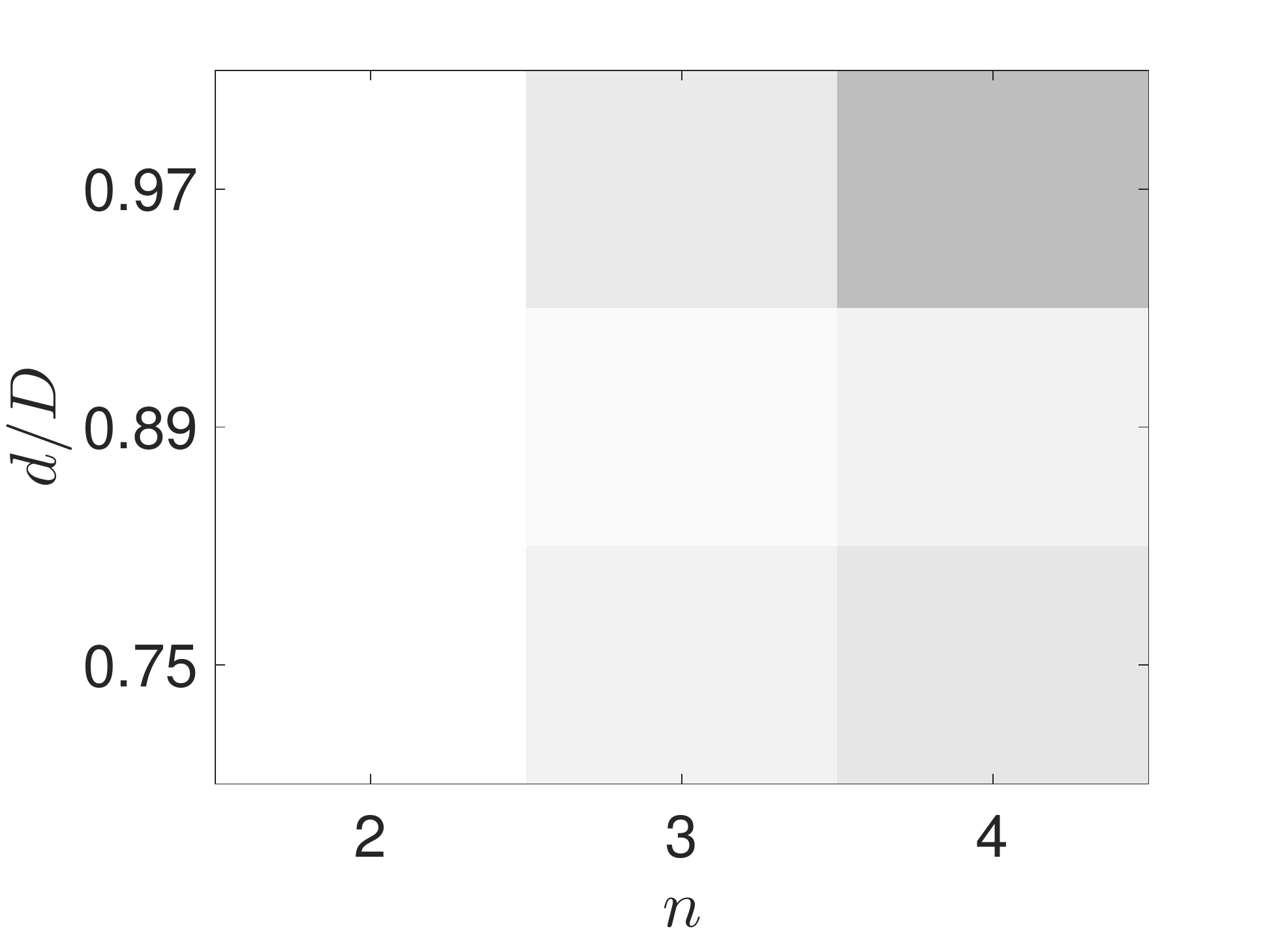}} 	
	\subfigure[DPCP-r, $\alpha=0.6$]{\label{figure:RANSACstyle_D_vs_n_Ns300_T50_DPCP-r06}\includegraphics[width=0.29\linewidth]{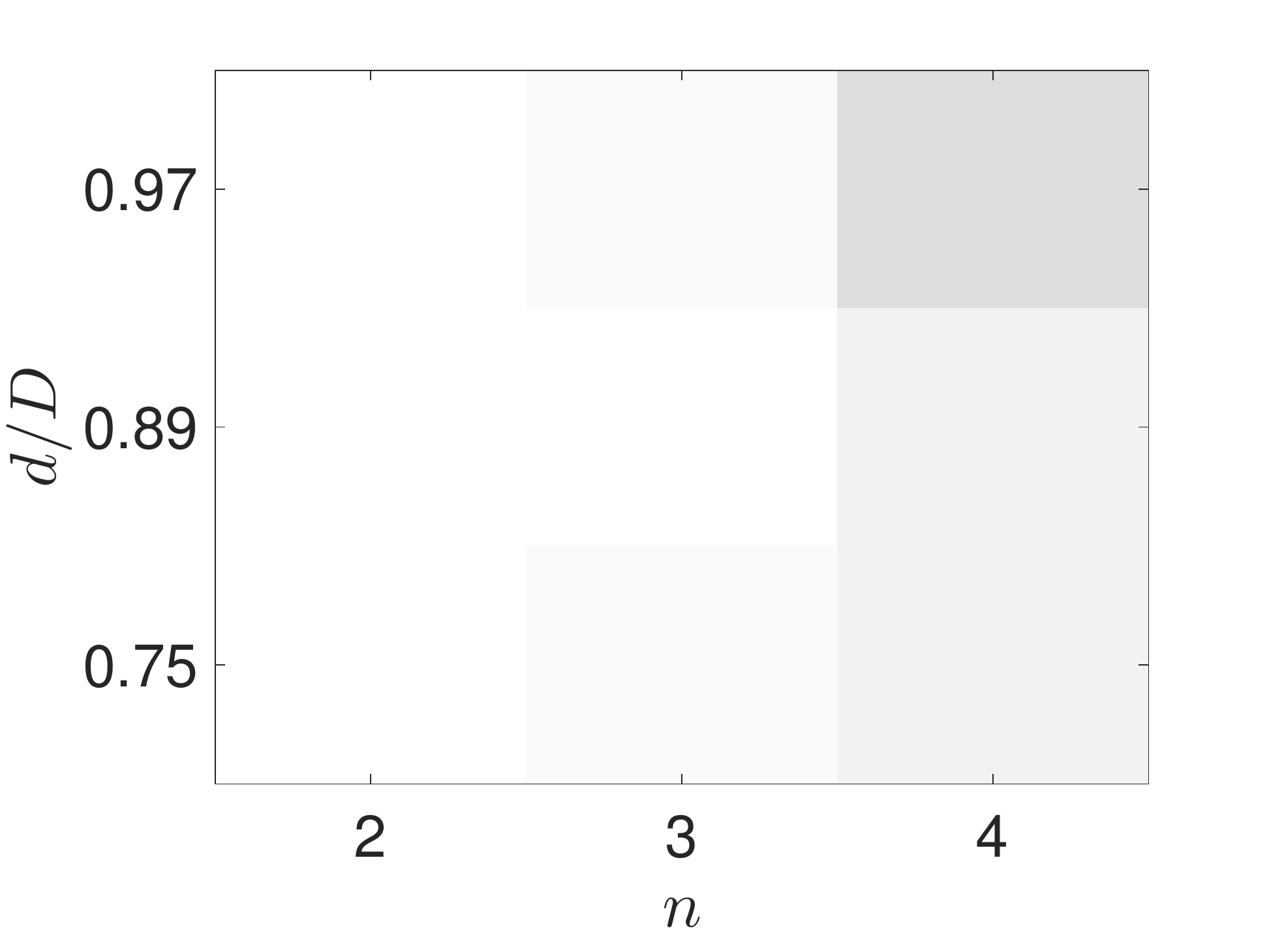}} 		
	\subfigure[DPCP-IRLS, $\alpha=1$]{\label{figure:RANSACstyle_D_vs_n_Ns300_T50_DPCP-IRLS1}\includegraphics[width=0.3\linewidth]{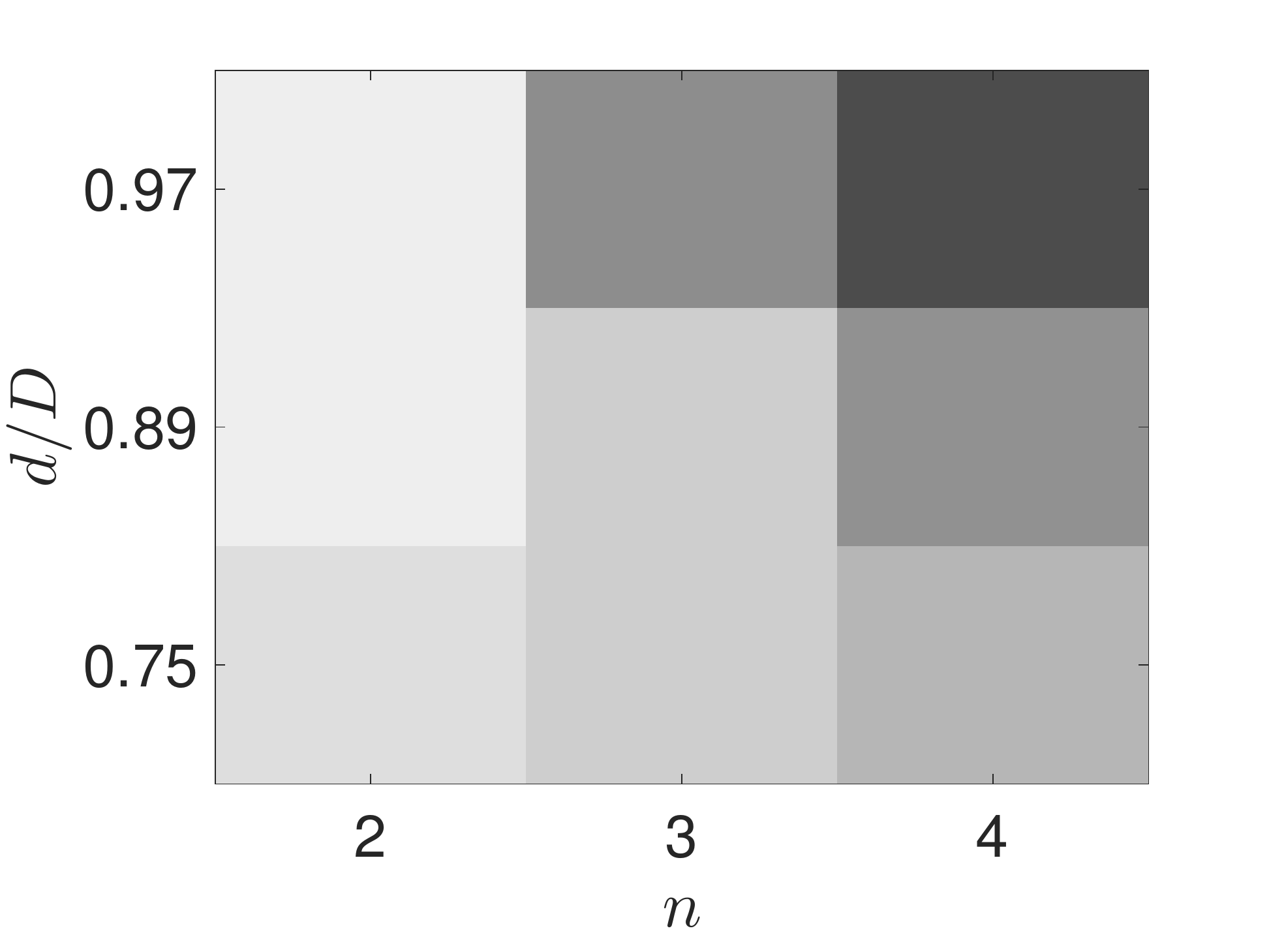}} 
	\subfigure[DPCP-IRLS, $\alpha=0.8$]{\label{figure:RANSACstyle_D_vs_n_Ns300_T50_DPCP-IRLS
	08}\includegraphics[width=0.3\linewidth]{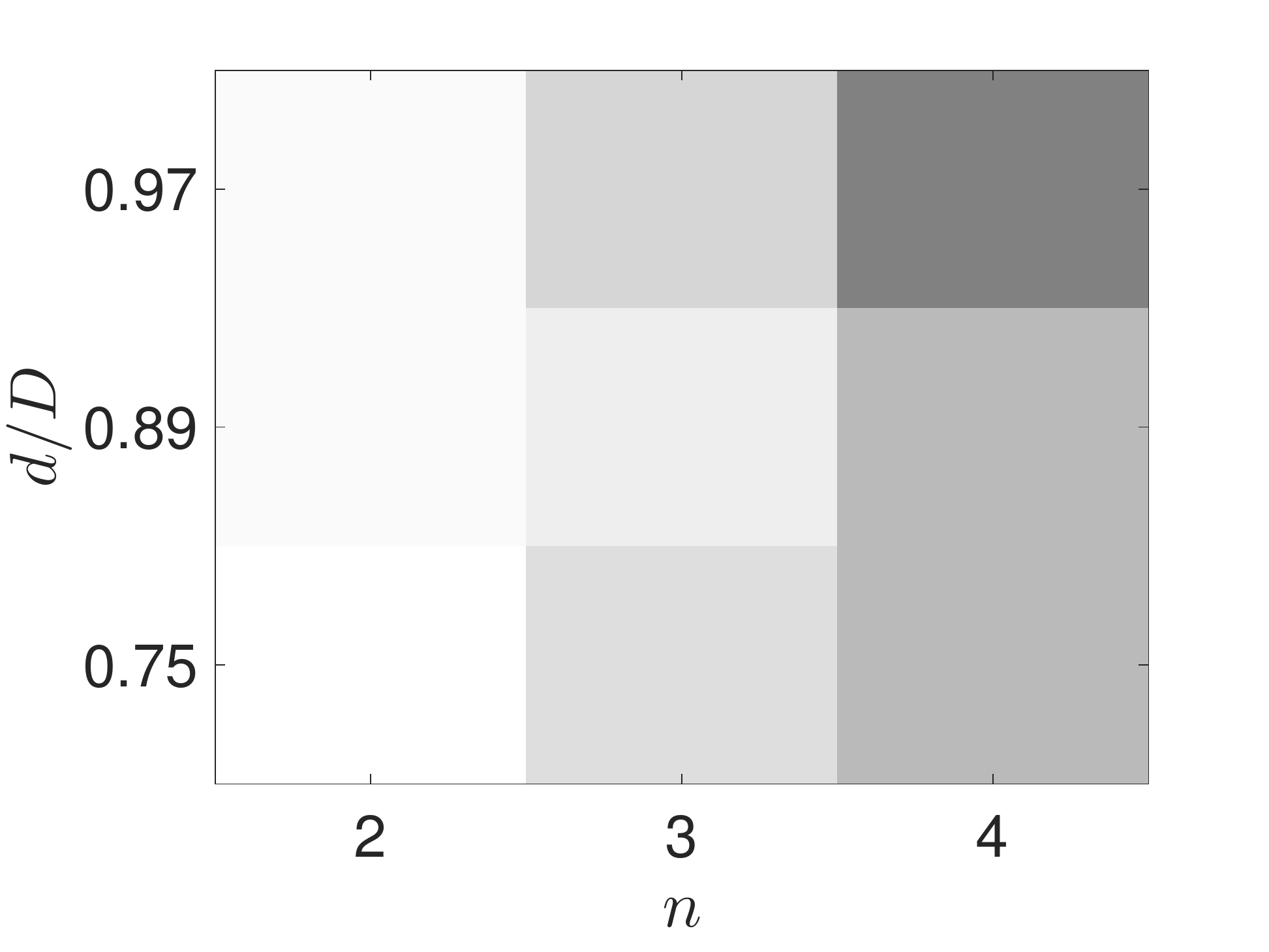}} 	
	\subfigure[DPCP-IRLS, $\alpha=0.6$]{\label{figure:RANSACstyle_D_vs_n_Ns300_T50_DPCP-IRLS06}\includegraphics[width=0.3\linewidth]{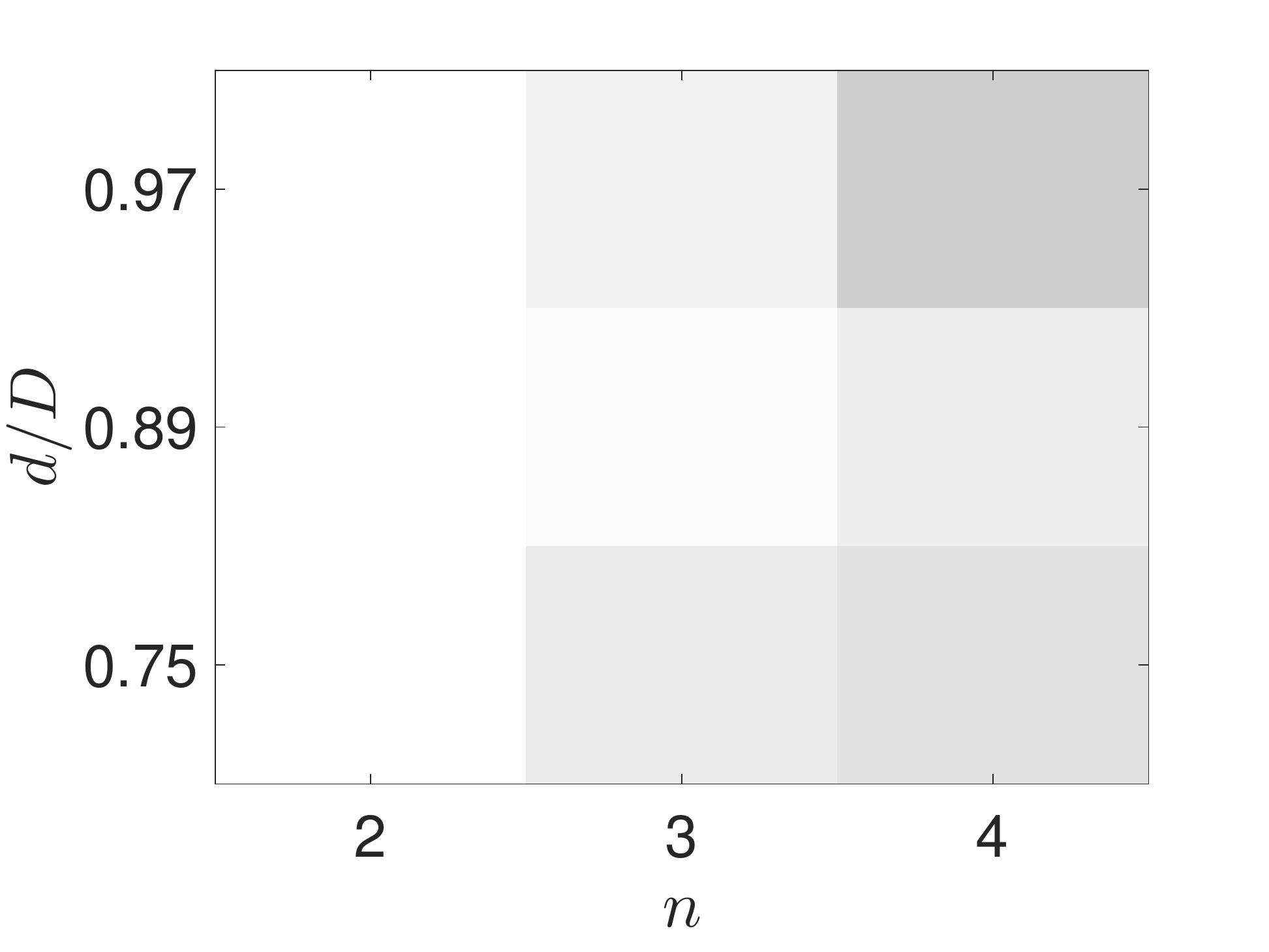}} 
	\caption{Sequential Hyperplane Learning: Clustering accuracy as a function of the number of 
	hyperplanes $n$ vs relative-dimension $d/D$ vs data balancing ($\alpha$). White corresponds to $1$, black to $0$.}
	\label{fig:MultipleDPCPRansacStyleDvsnalpha}
	\end{figure} 

\begin{figure}[t!]
	\centering
	\subfigure[RANSAC]{\label{figure:RANSACstyle_D_vs_n_Ns300_T50_RANSAC08}\includegraphics[width=0.29\linewidth]{RANSACstyle_D_vs_n_Ns300_T50_RANSAC08-eps-converted-to.pdf}} 	
	\subfigure[REAPER]{\label{figure:RANSACstyle_D_vs_n_Ns300_T50_REAPER08}\includegraphics[width=0.29\linewidth]{RANSACstyle_D_vs_n_Ns300_T50_REAPER08-eps-converted-to.pdf}} 	
		\subfigure[DPCP-r]{\label{figure:RANSACstyle_D_vs_n_Ns300_T50_DPCP-r
	08}\includegraphics[width=0.29\linewidth]{RANSACstyle_D_vs_n_Ns300_T50_DPCP-r08-eps-converted-to.pdf}} 	
	\subfigure[DPCP-r-d]{\label{figure:RANSACstyle_D_vs_n_Ns300_T50_DPCP-r-d
	08}\includegraphics[width=0.3\linewidth]{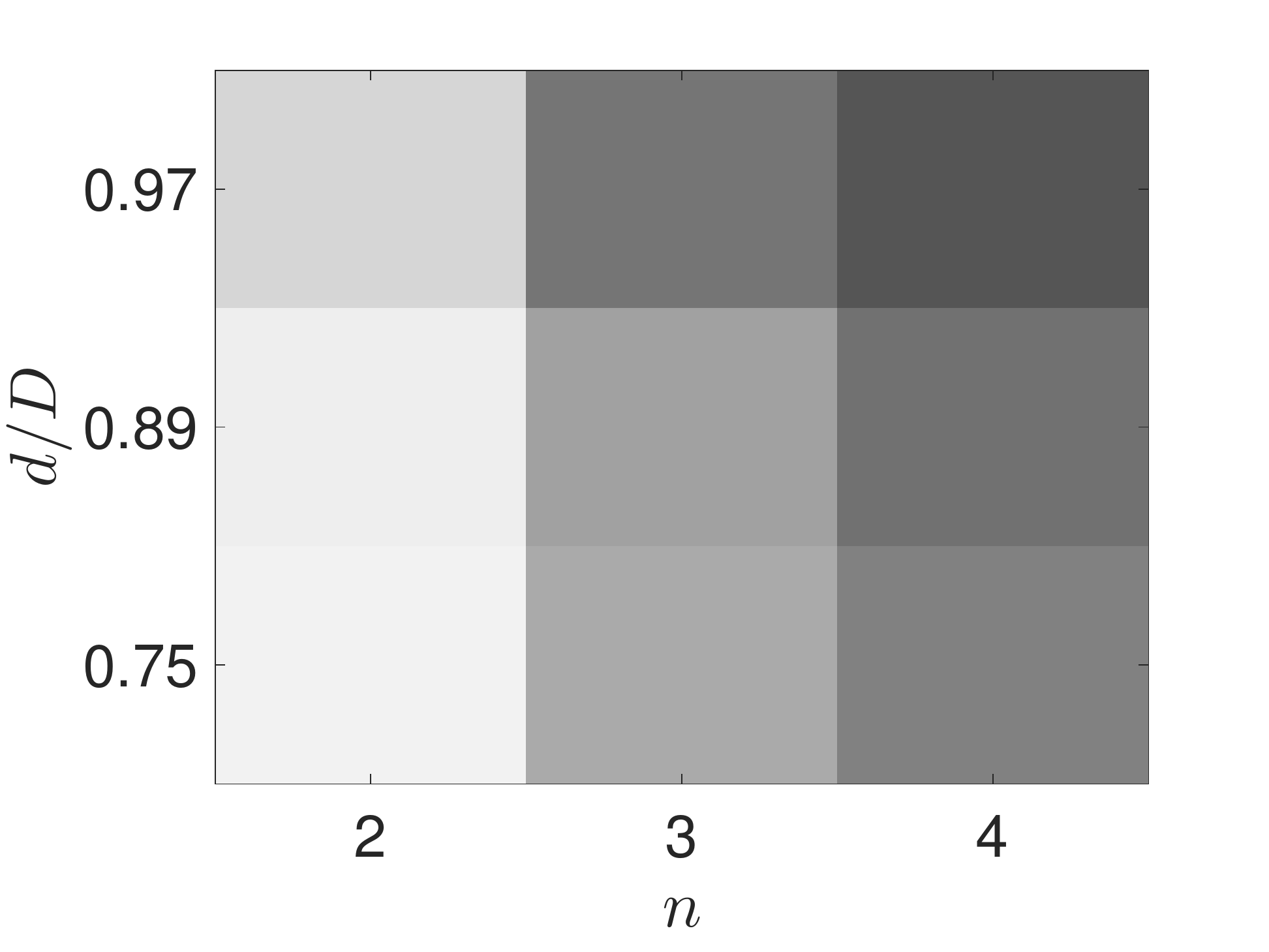}} 		
	\subfigure[DPCP-d]{\label{figure:RANSACstyle_D_vs_n_Ns300_T50_DPCP-d
	08}\includegraphics[width=0.29\linewidth]{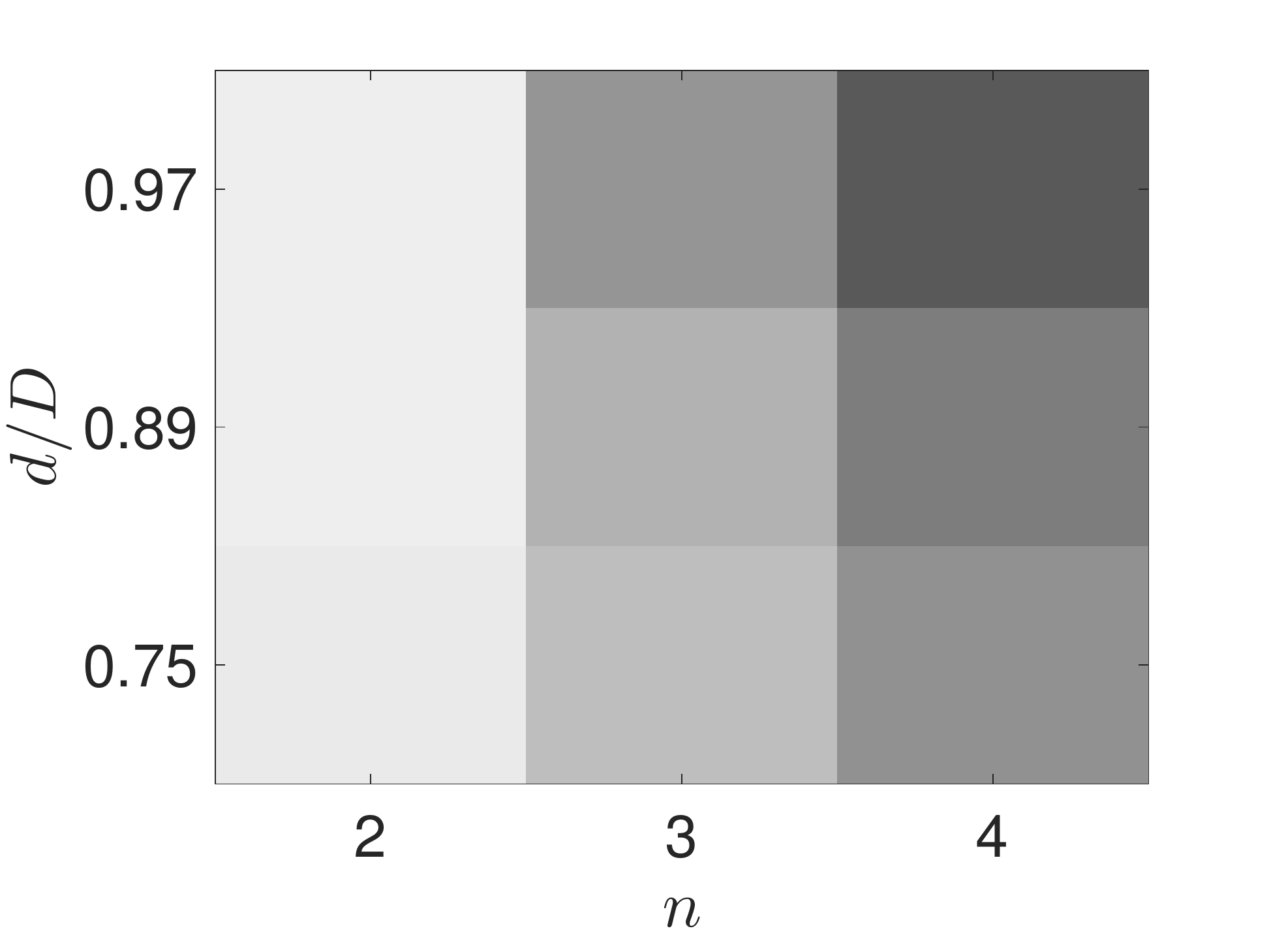}} 	
	\subfigure[DPCP-IRLS]{\label{figure:RANSACstyle_D_vs_n_Ns300_T50_DPCP-IRLS 08}\includegraphics[width=0.3\linewidth]{RANSACstyle_D_vs_n_Ns300_T50_DPCP-IRLS08-eps-converted-to.pdf}} 	
	\caption{Sequential Hyperplane Learning: Clustering accuracy as a function of the number of 
	hyperplanes $n$ vs relative-dimension $d/D$. Data balancing parameter is set to $\alpha = 0.8$.}
	\label{fig:MultipleDPCPRansacStyleDvsn08}
			\end{figure}
			
\begin{figure}[t!]
	\centering
		\subfigure[RANSAC, $\alpha=1$]{\label{figure:RANSACstyle_R_vs_n_Ns300_T50_RANSAC1}\includegraphics[width=0.29\linewidth]{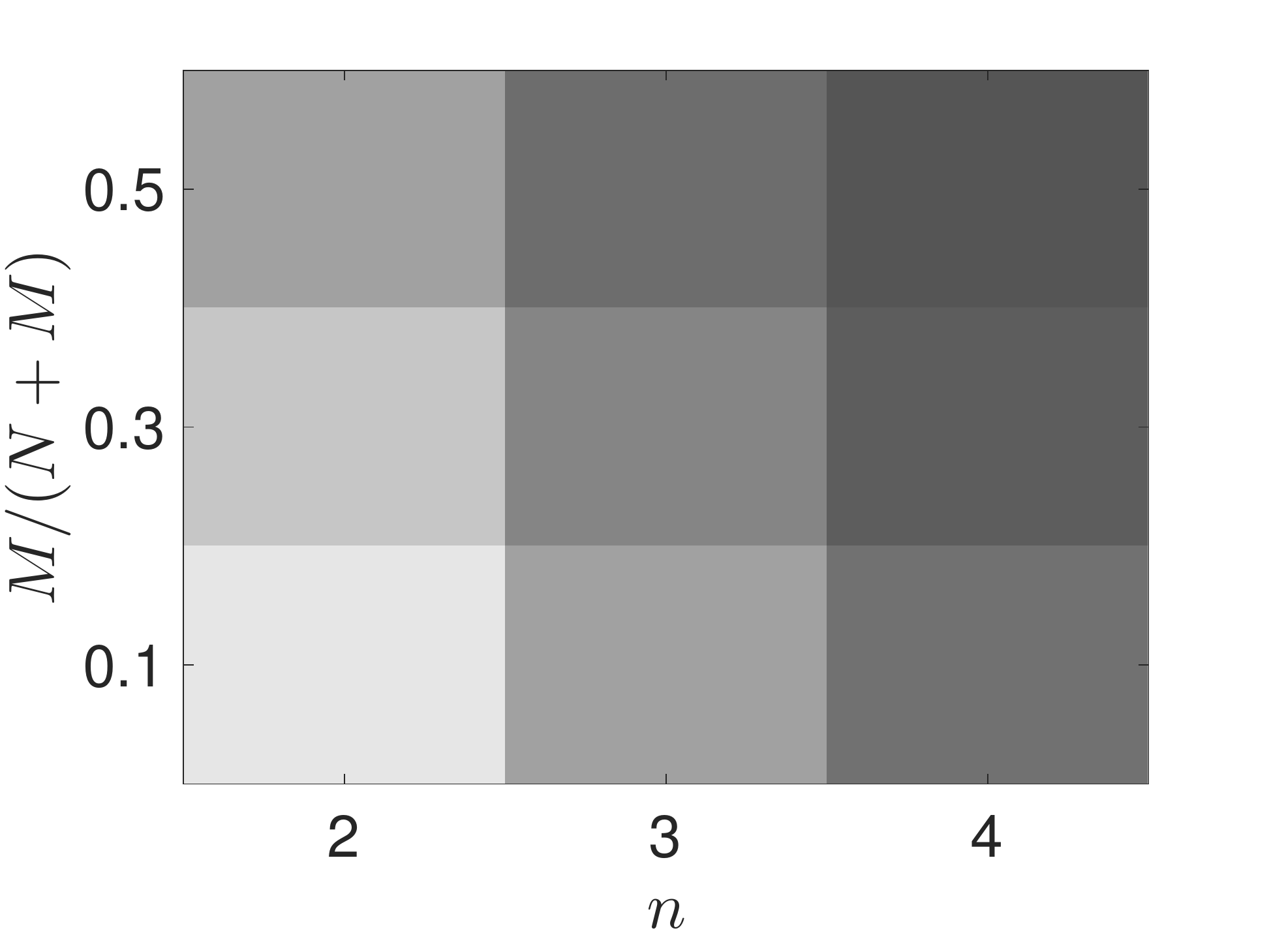}} 
	\subfigure[RANSAC, $\alpha=0.8$]{\label{figure:RANSACstyle_R_vs_n_Ns300_T50_RANSAC08}\includegraphics[width=0.29\linewidth]{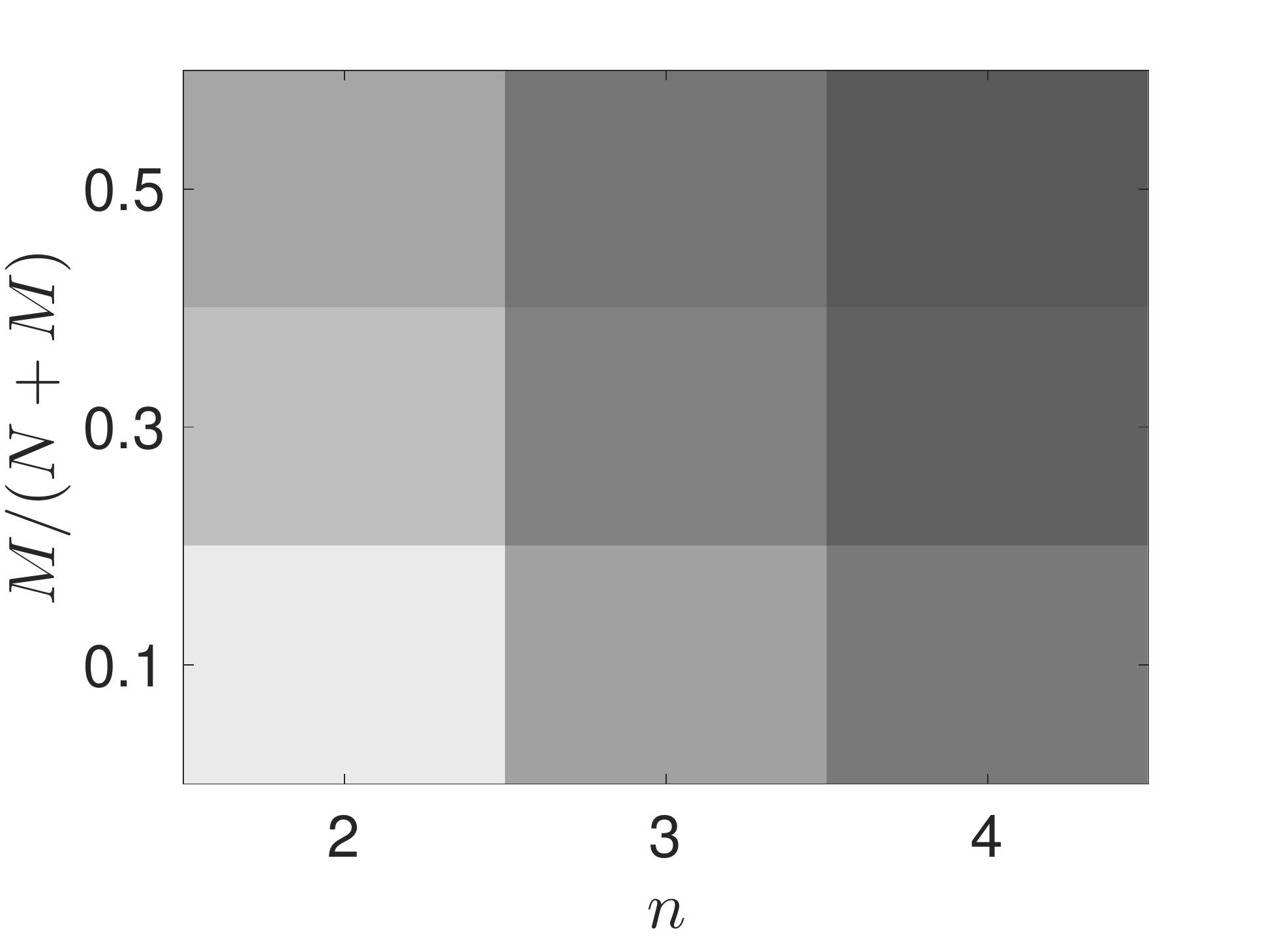}} 	
	\subfigure[RANSAC, $\alpha=0.6$]{\label{figure:RANSACstyle_R_vs_n_Ns300_T50_RANSAC06}\includegraphics[width=0.29\linewidth]{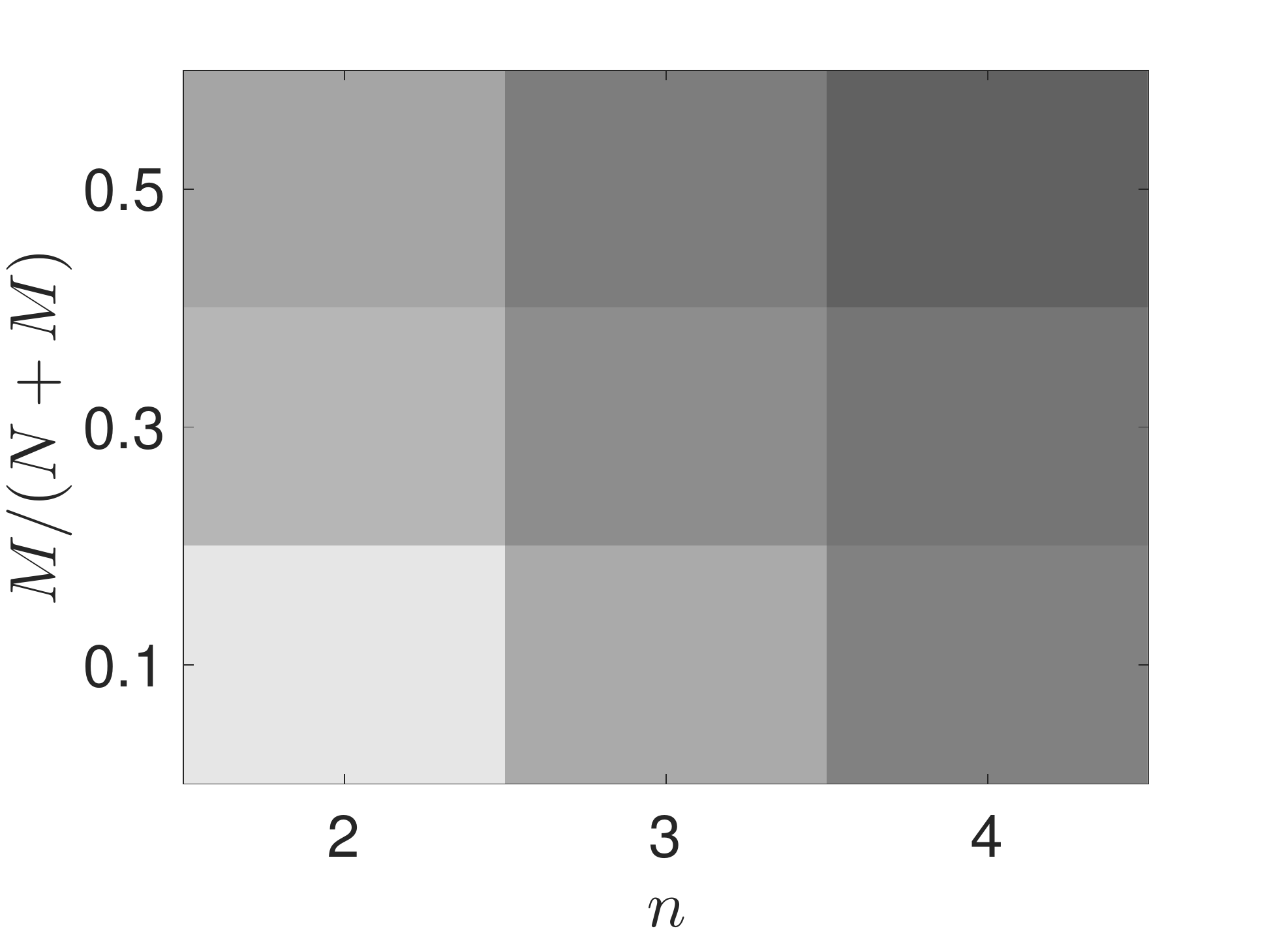}} 
	\subfigure[REAPER, $\alpha=1$]{\label{figure:RANSACstyle_R_vs_n_Ns300_T50_REAPER1}\includegraphics[width=0.29\linewidth]{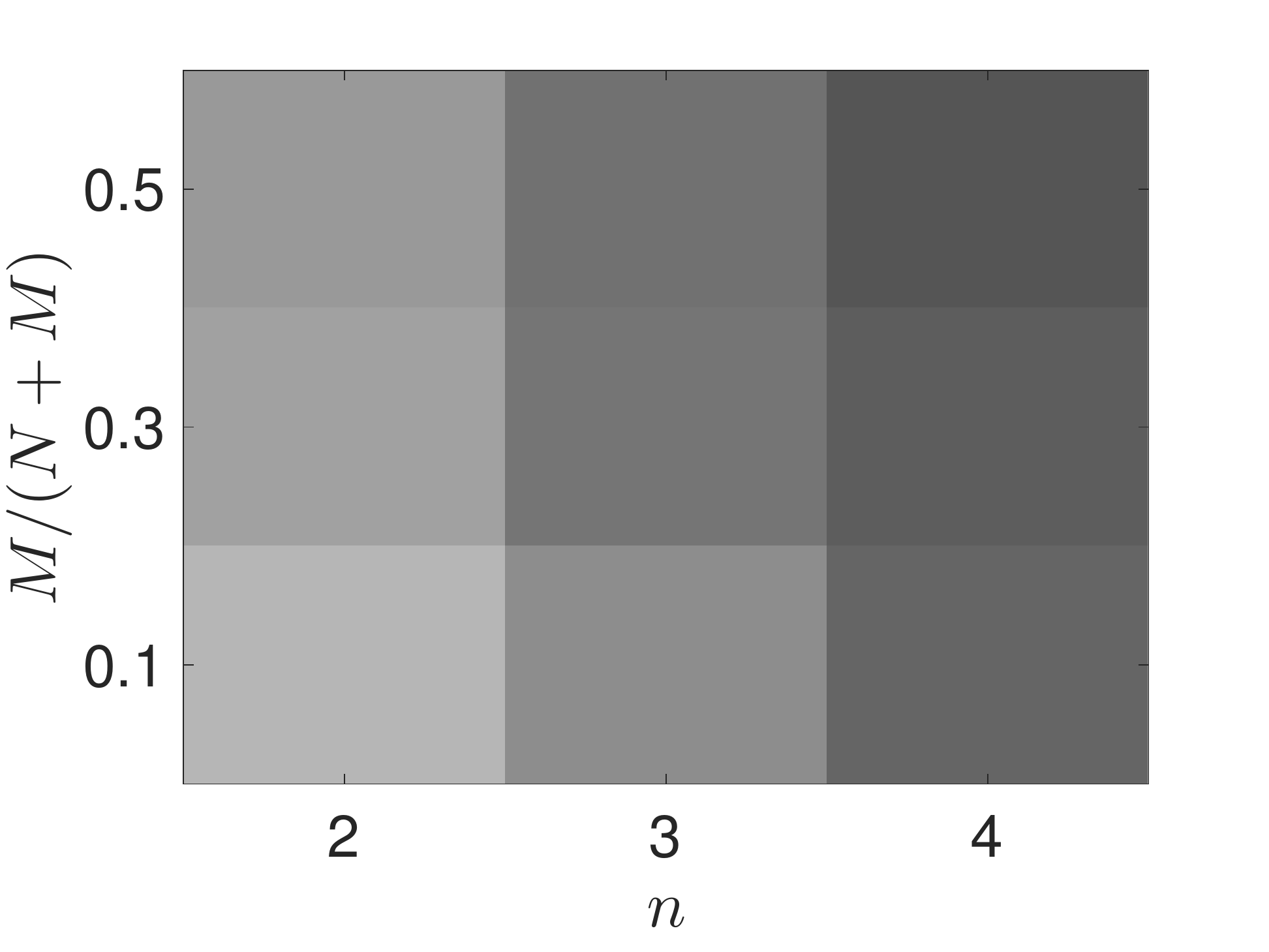}} 
	\subfigure[REAPER, $\alpha=0.8$]{\label{figure:RANSACstyle_R_vs_n_Ns300_T50_REAPER08}\includegraphics[width=0.29\linewidth]{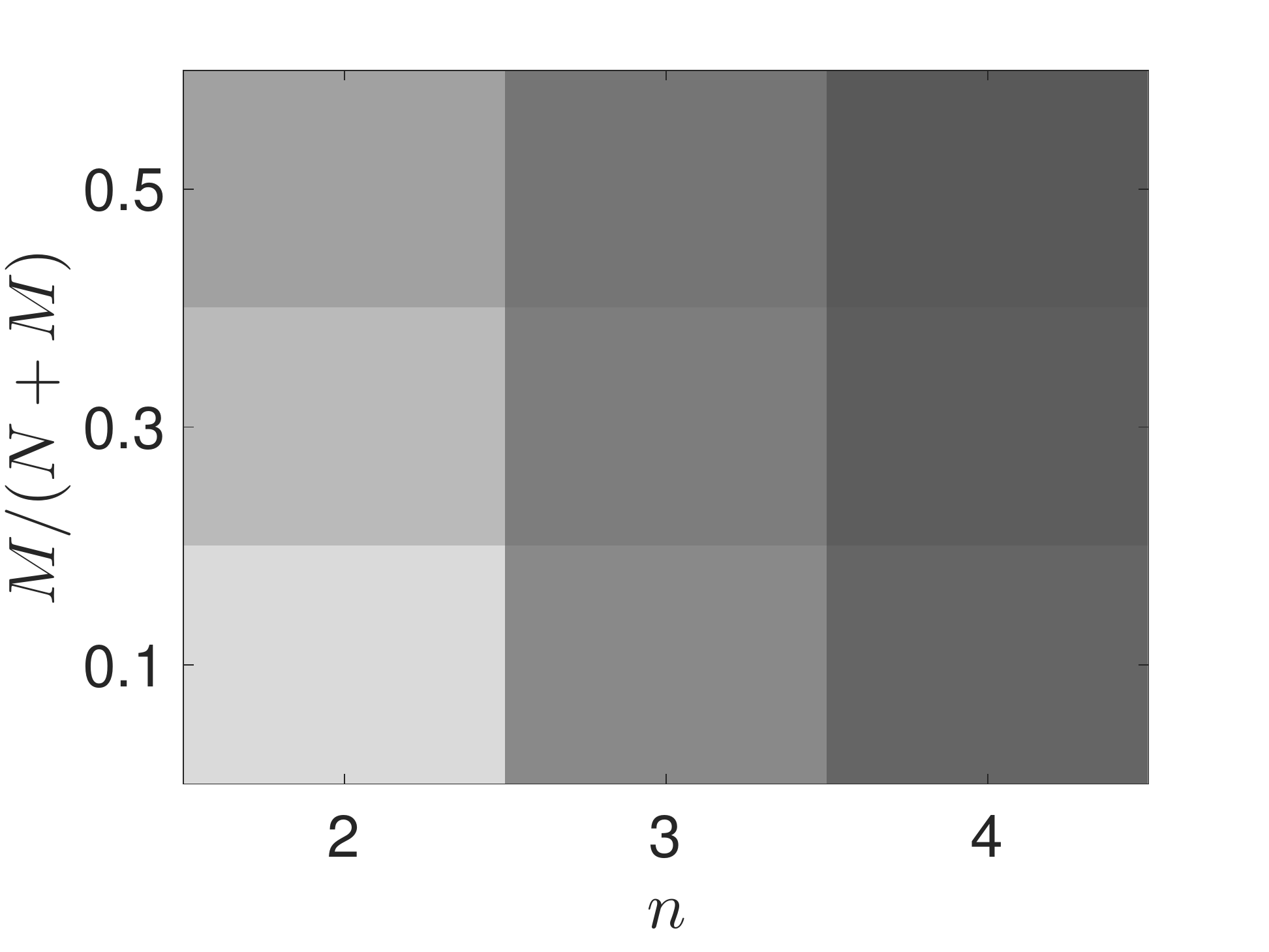}} 	
	\subfigure[REAPER, $\alpha=0.6$]{\label{figure:RANSACstyle_R_vs_n_Ns300_T50_REAPER06}\includegraphics[width=0.29\linewidth]{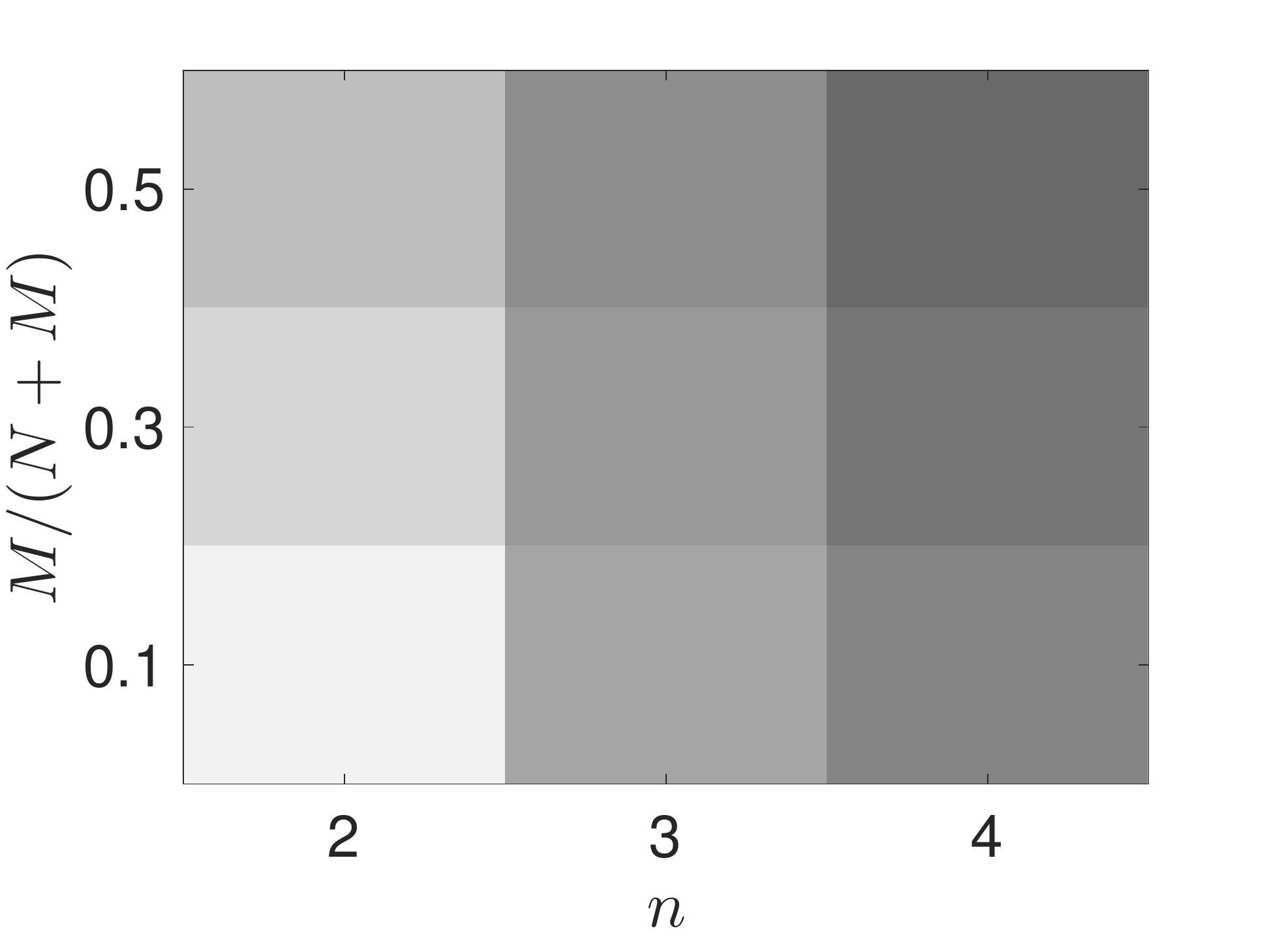}} 
	\subfigure[DPCP-r, $\alpha=1$]{\label{figure:RANSACstyle_R_vs_n_Ns300_T50_DPCP-r1}\includegraphics[width=0.29\linewidth]{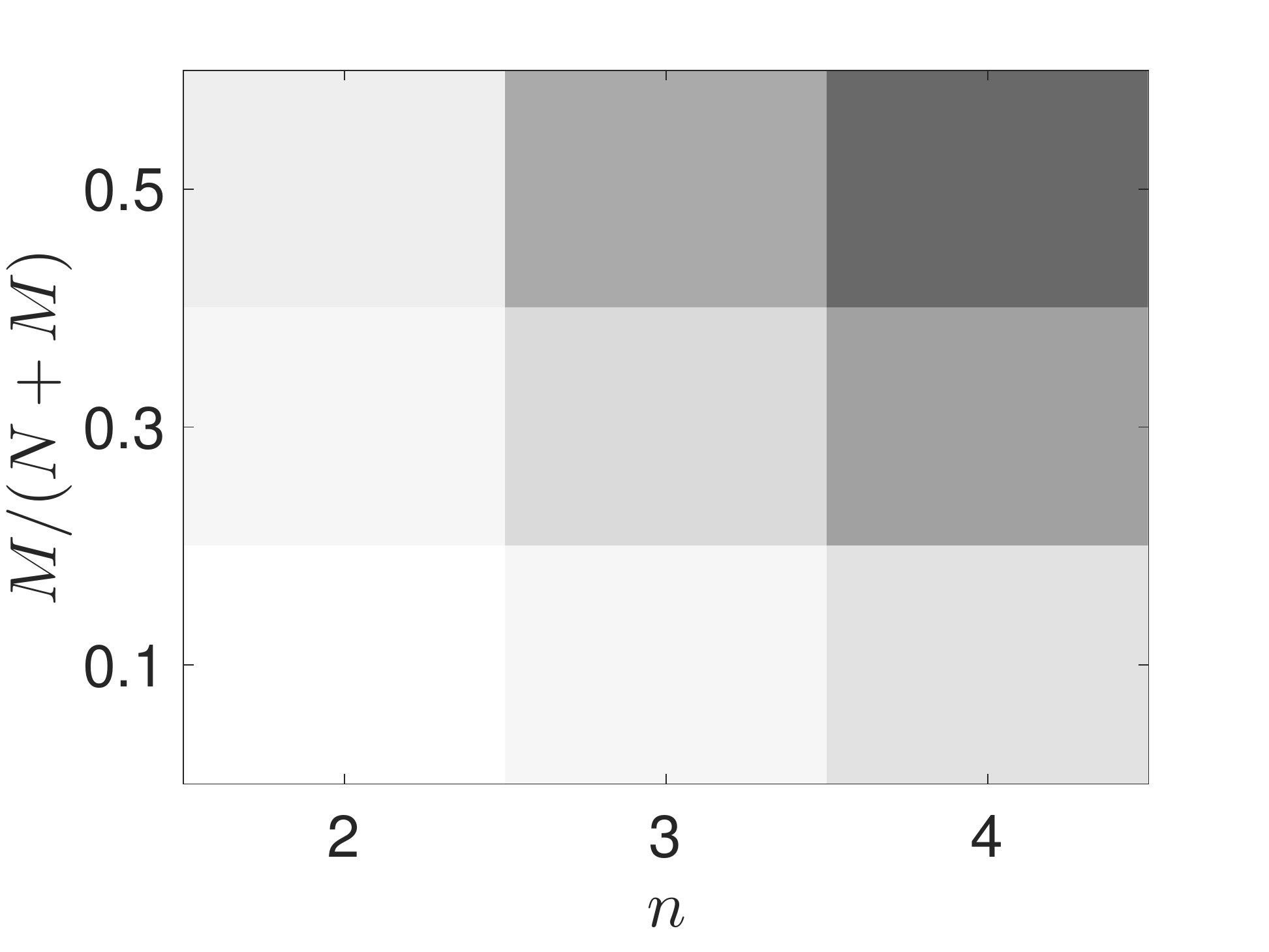}} 
	\subfigure[DPCP-r, $\alpha=0.8$]{\label{figure:RANSACstyle_R_vs_n_Ns300_T50_DPCP-r
	08}\includegraphics[width=0.29\linewidth]{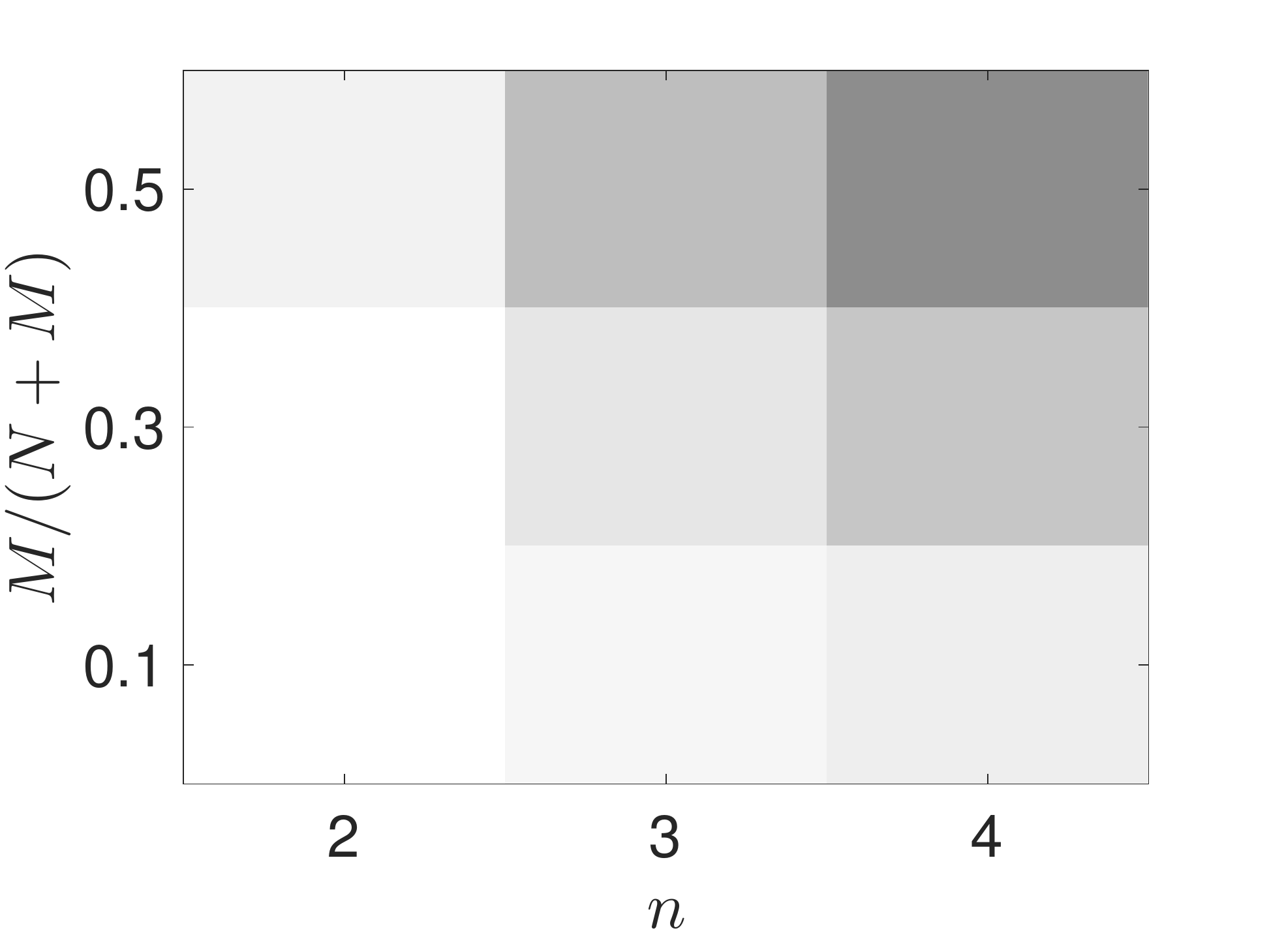}} 	
	\subfigure[DPCP-r, $\alpha=0.6$]{\label{figure:RANSACstyle_R_vs_n_Ns300_T50_DPCP-r06}\includegraphics[width=0.29\linewidth]{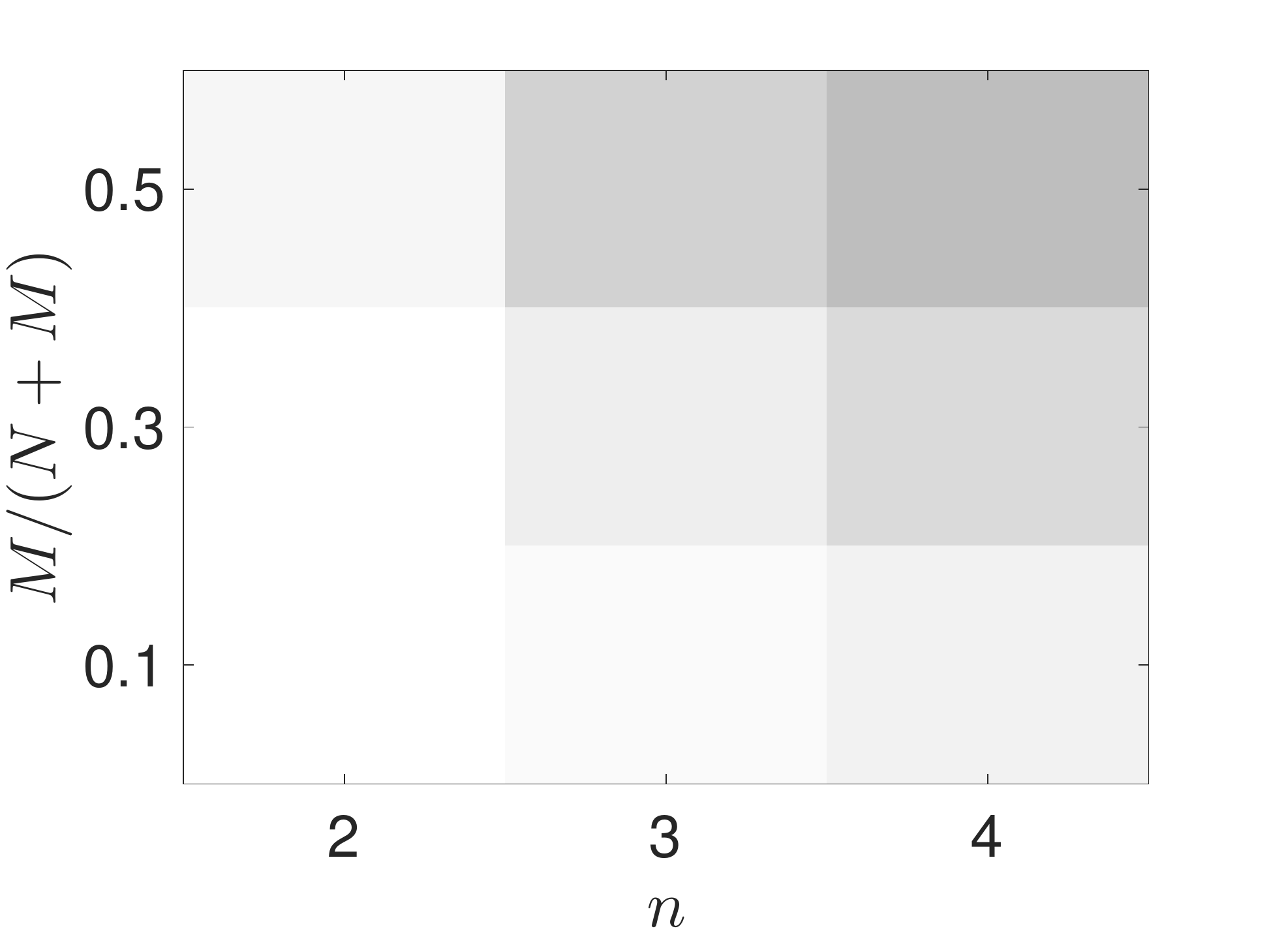}} 		
	\subfigure[DPCP-IRLS, $\alpha=1$]{\label{figure:RANSACstyle_R_vs_n_Ns300_T50_DPCP-IRLS1}\includegraphics[width=0.29\linewidth]{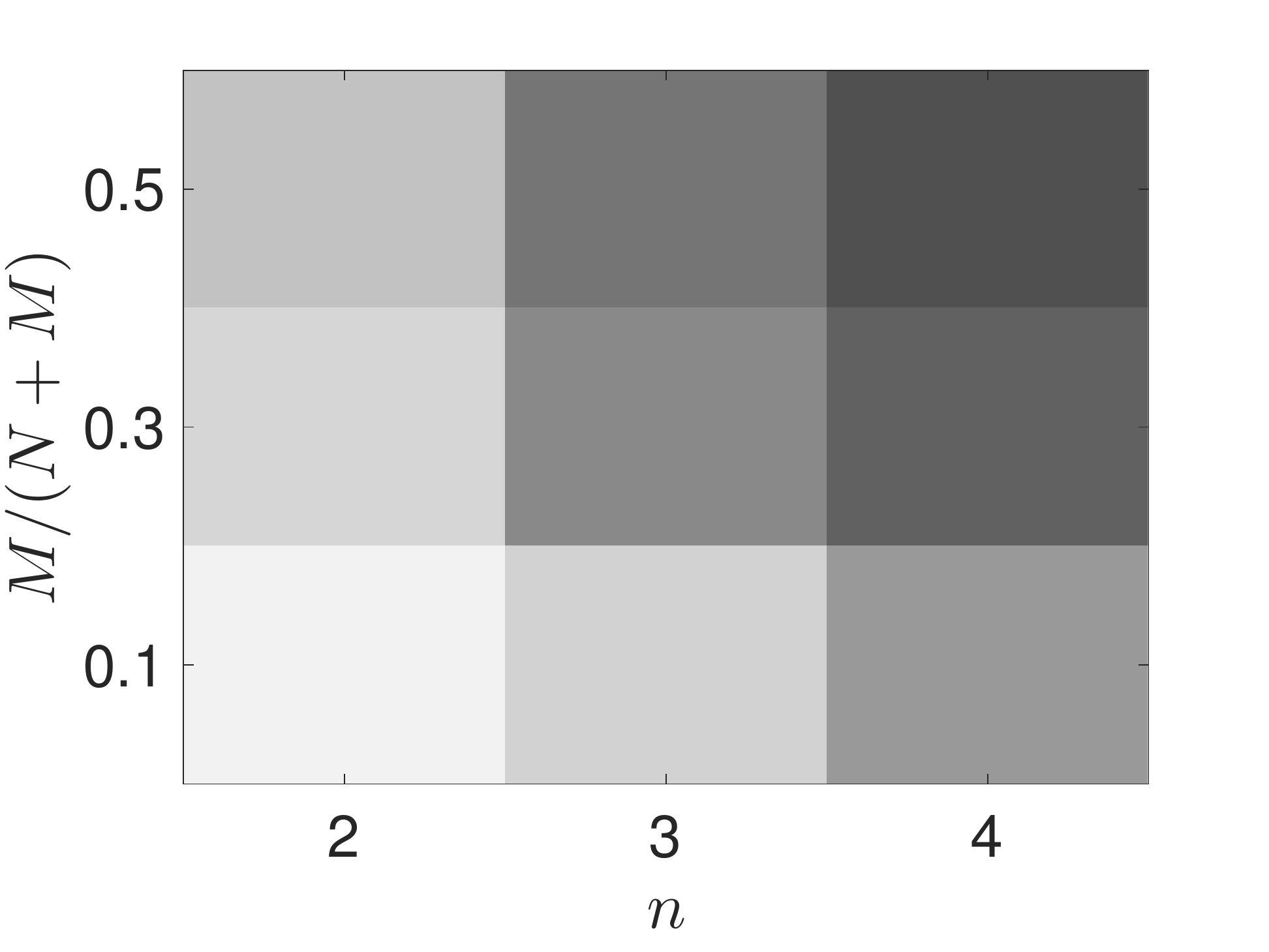}} 
	\subfigure[DPCP-IRLS, $\alpha=0.8$]{\label{figure:RANSACstyle_R_vs_n_Ns300_T50_DPCP-IRLS
	08}\includegraphics[width=0.29\linewidth]{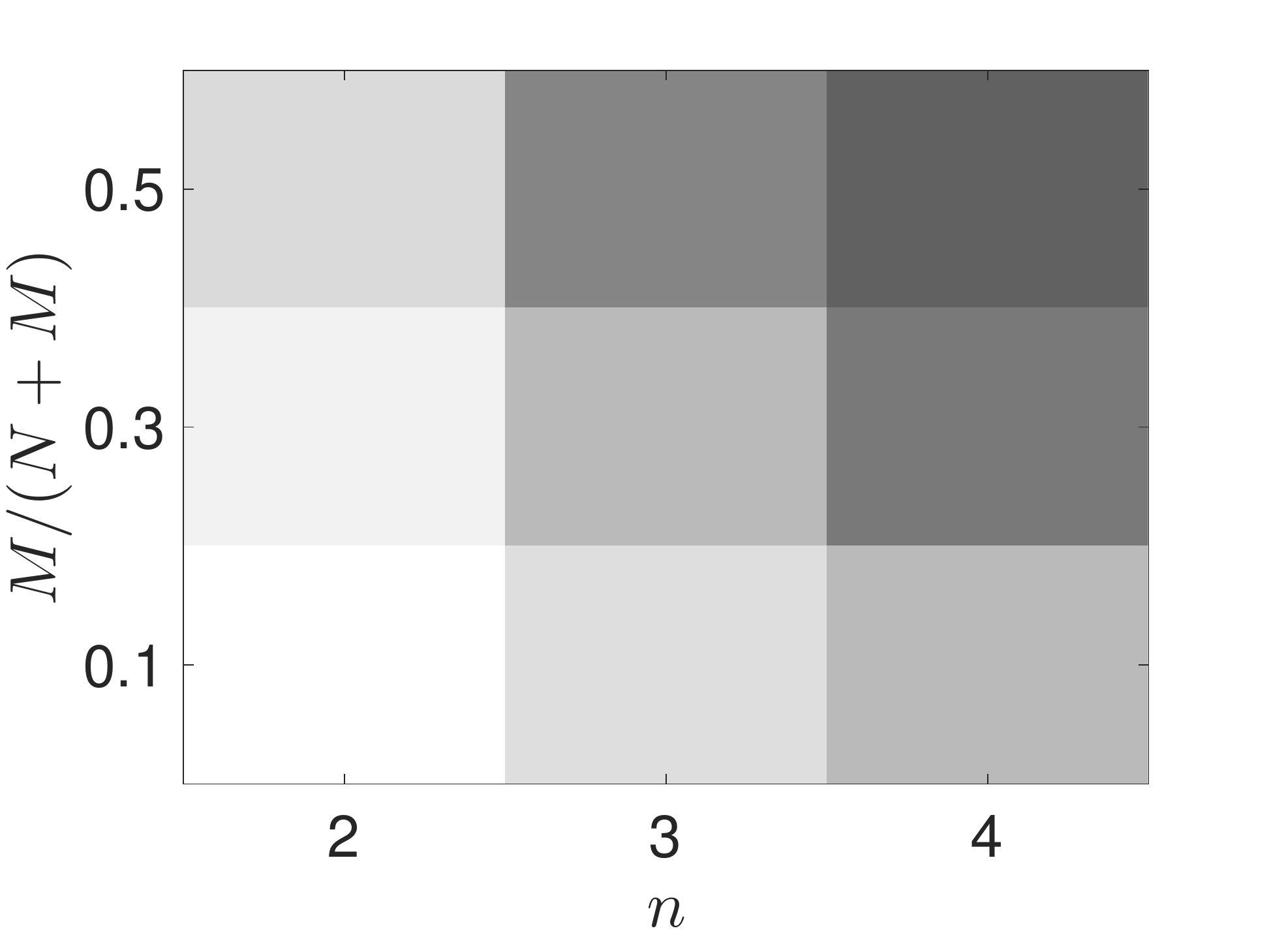}} 	
	\subfigure[DPCP-IRLS, $\alpha=0.6$]{\label{figure:RANSACstyle_R_vs_n_Ns300_T50_DPCP-IRLS06}\includegraphics[width=0.29\linewidth]{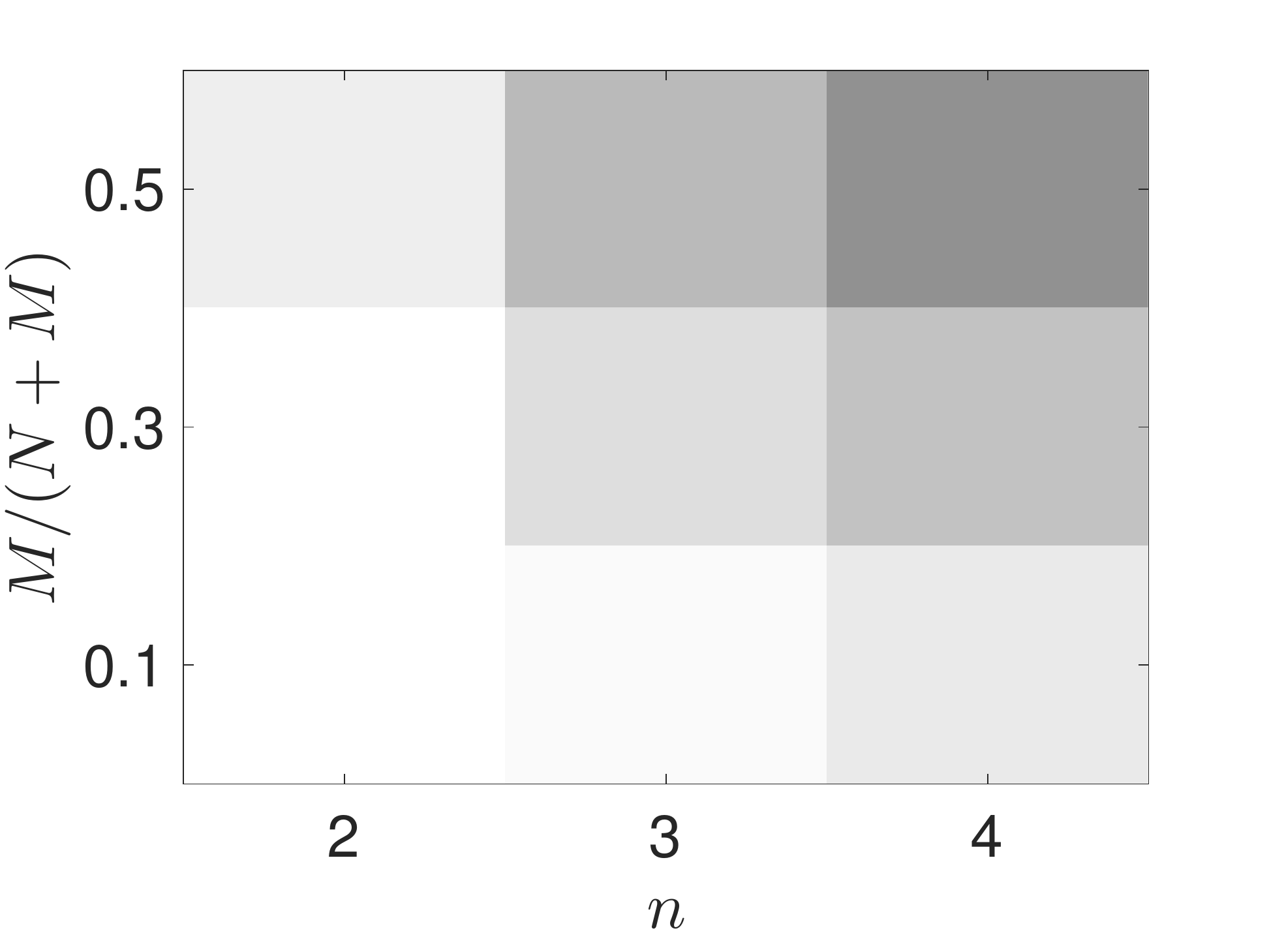}} 
	\caption{Sequential Hyperplane Learning: Clustering accuracy as a function of the number of 
	hyperplanes $n$ vs outlier ratio vs data balancing ($\alpha$). White corresponds to $1$, black to $0$.}
	\label{fig:MultipleDPCPRansacStyleRvsnalpha}
		\end{figure} 

\begin{figure}[t!]
	\centering
	\subfigure[RANSAC]{\label{figure:RANSACstyle_R_vs_n_Ns300_T50_RANSAC08}\includegraphics[width=0.29\linewidth]{RANSACstyle_R_vs_n_Ns300_T50_RANSAC08-eps-converted-to.pdf}} 	
	\subfigure[REAPER]{\label{figure:RANSACstyle_R_vs_n_Ns300_T50_REAPER08}\includegraphics[width=0.29\linewidth]{RANSACstyle_R_vs_n_Ns300_T50_REAPER08-eps-converted-to.pdf}} 	
		\subfigure[DPCP-r]{\label{figure:RANSACstyle_R_vs_n_Ns300_T50_DPCP-r
	08}\includegraphics[width=0.29\linewidth]{RANSACstyle_R_vs_n_Ns300_T50_DPCP-r08-eps-converted-to.pdf}} 	
	\subfigure[DPCP-r-d]{\label{figure:RANSACstyle_R_vs_n_Ns300_T50_DPCP-r-d
	08}\includegraphics[width=0.3\linewidth]{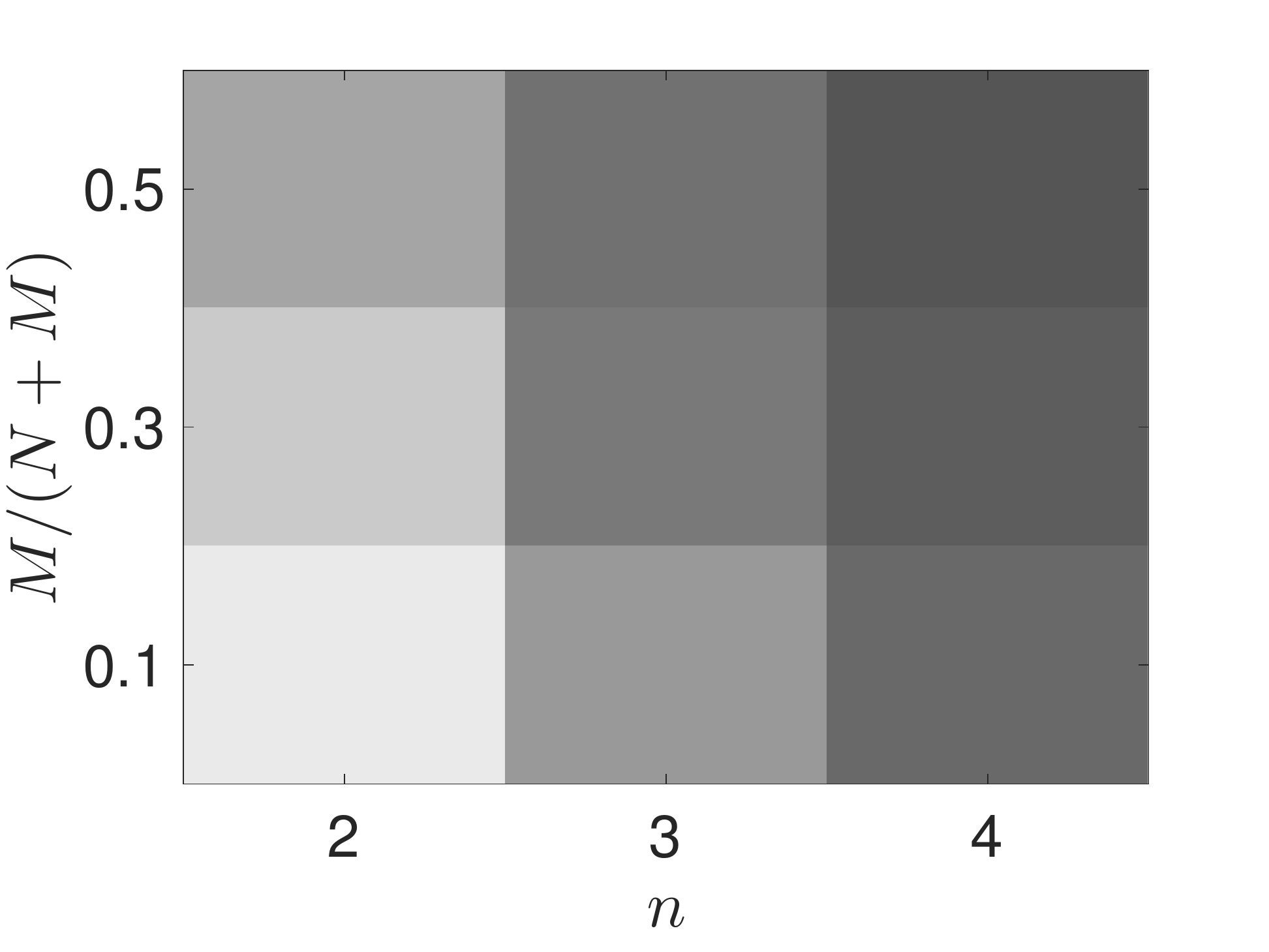}} 		
	\subfigure[DPCP-d]{\label{figure:RANSACstyle_R_vs_n_Ns300_T50_DPCP-d
	08}\includegraphics[width=0.29\linewidth]{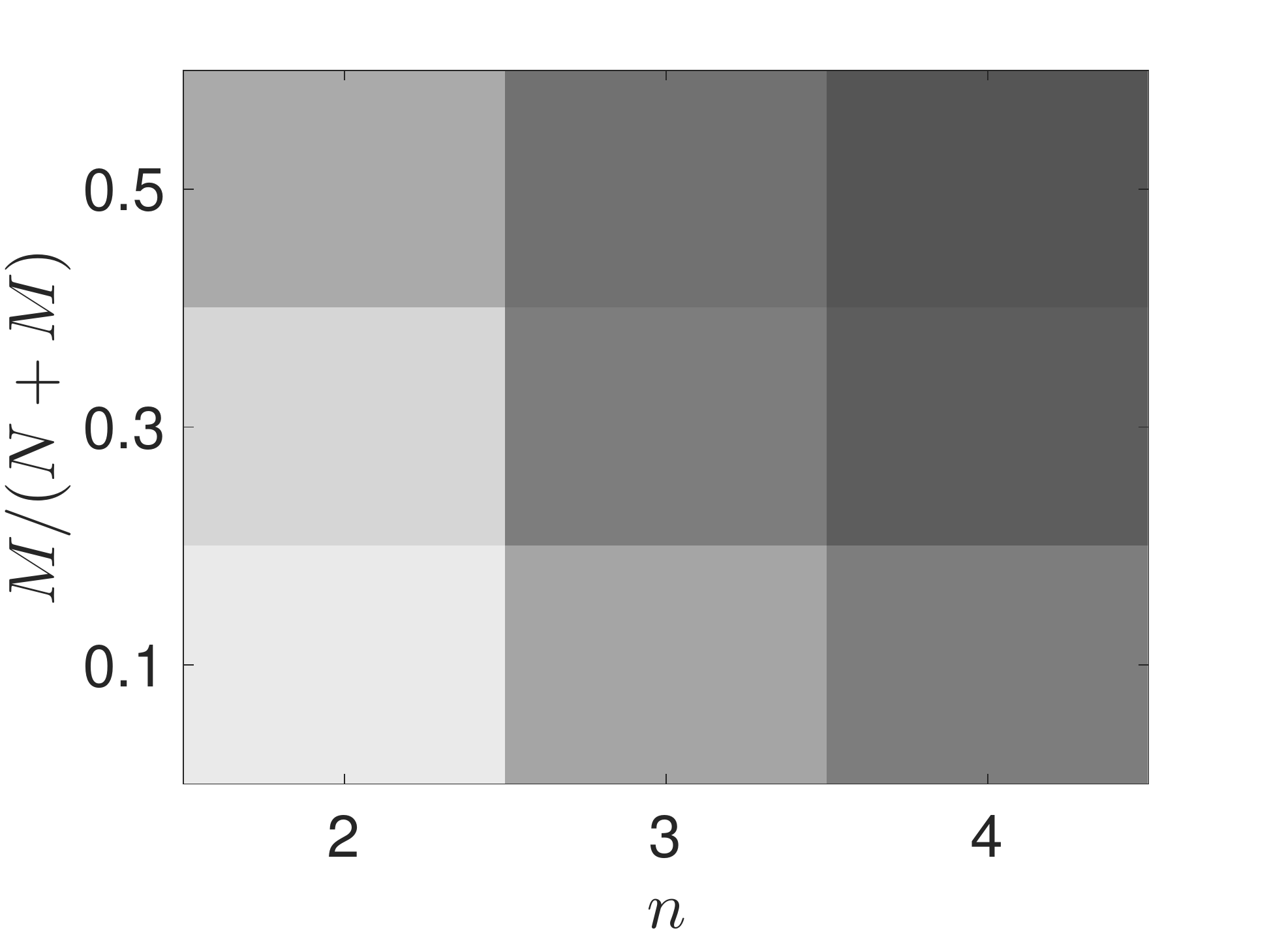}} 	
	\subfigure[DPCP-IRLS]{\label{figure:RANSACstyle_R_vs_n_Ns300_T50_DPCP-IRLS
	08}\includegraphics[width=0.3\linewidth]{RANSACstyle_R_vs_n_Ns300_T50_DPCP-IRLS08-eps-converted-to.pdf}} 	
	\caption{Sequential Hyperplane Learning: Clustering accuracy as a function of the number of 
	hyperplanes $n$ vs outlier ratio. Data balancing parameter is set to $\alpha = 0.8$.}
	\label{fig:MultipleDPCPRansacStyleRvsn08}
		\end{figure}

\indent \myparagraph{Dataset design} We begin by evaluating experimentally the sequential hyperplane learning Algorithm \ref{alg:MultipleDPCPsequential} using synthetic data. The coordinate dimension $D$ of the data is inspired by major applications where hyperplane arrangements appear. In particular, we recall that 
\begin{itemize}
\item In $3D$ point cloud analysis, the coordinate 
dimension is $3$, but since the various planes do not necessarily pass through a common origin, i.e., they are affine, one may work with homogeneous coordinates, which increases the coordinate dimension of the data by $1$
(see \cite{Tsakiris:AffineASC-ArXiv15}), i.e., $D=4$.
\item In two-view geometry one works with correspondences between pairs of $3D$ points. Each such correspondence is treated as a point itself, equal to the tensor product of the two $3D$ corresponding points, 
thus having coordinate dimension $D=9$.
  \end{itemize} As a consequence, we choose $D=4,9$ as well as $D=30$, where the choice of $30$ is a-posteriori justified as being sufficiently larger than $4$ or $9$, so that the clustering problem becomes more challenging.
For each choice of $D$ we randomly generate $n=2,3,4$ hyperplanes of $\Re^D$ and sample each 
hyperplane as follows. For each choice of $n$ the total number of points in the dataset is set to
$300n$, and the number $N_i$ of points sampled from hyperplane $i >1$ is set to $N_i = \alpha^{i-1} N_{i-1}$, so that
\begin{align}
\sum_{i=1}^n N_i = (1+\alpha+\cdots+\alpha^{n-1})N_1 = 300n.
\end{align} Here $\alpha \in (0,1]$ is a parameter that controls the \emph{balancing} of the clusters: $\alpha=1$ means the clusters are perfectly balanced, while smaller values of $\alpha$ lead to less balanced clusters. 
In our experiment we try $\alpha=1,0.8,0.6$. Having specified the size of each cluster, the points of each cluster are sampled from a zero-mean unit-variance Gaussian distribution with support in the corresponding hyperplane. To make the experiment more realistic, we corrupt points from each hyperplane by adding white Gaussian noise of standard deviation $\sigma=0.01$ with support in the direction orthogonal to the hyperplane. Moreover, we corrupt the dataset by adding $10\%$ outliers sampled from a standard zero-mean unit-variance Gaussian distribution with support in the entire ambient space. 

\indent \myparagraph{Algorithms and parameters} In Algorithm \ref{alg:MultipleDPCPsequential} we solve the DPCP problem by using all four variations DPCP-r, DPCP-r-d, DPCP-d and DPCP-IRLS 
(see Section \ref{subsection:DPCPcomputation}),
thus leading to four different versions of the algorithm. All DPCP algorithms are set to terminate if either
a maximal number of $20$ iterations for DPCP-r or $100$ iterations for DPCP-r-d,DPDP-d, DPCP-IRLS 
is reached, or if the algorithm converges within accuracy of $10^{-3}$. 
We also compare with the REAPER analog of Algorithm \ref{alg:MultipleDPCPsequential}, where the computation of each normal vector is done by the IRLS version of REAPER (see Section \ref{section:PriorArt}) instead of DPCP. As with the DPCP algorithms, its maximal number of iterations is $100$ and its convergence accuracy is $10^{-3}$. 

Finally, we compare with RANSAC, which is the predominant method for clustering hyperplanes in low ambient dimensions (e.g., for $D=4,9$). For fairness, we implement a version of RANSAC which involves the same weighting scheme as Algorithm \ref{alg:MultipleDPCPsequential}, but instead of weighting the points, it uses the normalized weights as a discrete probability distribution on the data points; thus points that lie close to some of the already computed hyperplanes, have a low probability of being selected. Moreover, we control the running time of RANSAC so that it is as slow as DPCP-r, the latter being the slowest among the four DPCP algorithms. 

\indent \myparagraph{Results} Since not all results can fit in a single figure, we show the mean clustering accuracy over $50$ independent experiments in Fig. \ref{fig:MultipleDPCPRansacStyleDvsnalpha} only for RANSAC, REAPER, DPCP-r and
DPDP-IRLS (i.e., not including DPCP-r-d and DPCP-d), but for all values $\alpha=1,0.8,0.6$, as well as in Fig. \ref{fig:MultipleDPCPRansacStyleDvsn08} for all methods but only for $\alpha=0.8$. The accuracy is normalized to range from $0$ to $1$, with 
$0$ corresponding to black color, and $1$ corresponding to white. 
			
First, observe that the performance of almost all methods improves as the clusters become
more unbalanced ($\alpha=1 \rightarrow \alpha=0.6$). This is intuitively expected, as the smaller $\alpha$ is the more 
dominant is the $i$-th hyperplane with respect to hyperplanes $i+1,\dots,n$, and so the more likely it is to be identified at iteration $i$ of the sequential algorithm. The only exception to this intuitive phenomenon is 
RANSAC, which appears to be insensitive to the balancing of the data. This is because 
RANSAC is configured to run a relatively long amount of time, approximately equal to the running
time of DPCP-r, and as it turns out this compensates for the unbalancing of the data, since many
different samplings take place, thus leading to approximately constant behavior across different $\alpha$. 

In fact, notice that RANSAC is the best among all methods when $D=4$,
with mean clustering accuracy ranging from $99\%$ when $n=2$, to $97\%$ when $n=4$.
 On the other hand, 
 RANSAC's performance drops dramatically when we move to higher coordinate dimensions and more than
$2$ hyperplanes. For example, for $\alpha=0.8$ and $n=4$, the mean clustering accuracy of RANSAC 
drops from $97\%$ for $D=4$, to $44\%$ for $D=9$, to $28\%$ for $D=30$. This is due to the fact
that the probability of sampling $D-1$ points from the same hyperplane decreases as $D$ increases.

Secondly, the proposed Algorithm \ref{alg:MultipleDPCPsequential} using DPCP-r is uniformly the best method (and using DPCP-IRLS is the second best), with the slight exception of $D=4$, where its clustering accuracy ranges for $\alpha=0.8$ from $99\%$ for $n=2$ (same as RANSAC), to $89\%$ for $n=4$, as opposed to the $97\%$ of RANSAC for the latter case. In fact, all DPCP variants were superior than RANSAC or REAPER in the challenging scenario of $D=30,n=4$, where for $\alpha=0.6$, DPCP-r, DPCP-IRLS, DPCP-d and DPCP-r-d gave $86\%,81\%,74\%$ and $52\%$ accuracy respectively, as opposed to $28\%$ for RANSAC and $42\%$ for REAPER. 

Table \ref{table:synthetic_D_vs_n_noise_outliers_T50_rt} reports running times in seconds for $\alpha=1$
and $n=2,4$. It is readily seen that DPCP-r is the slowest among all methods (recall that RANSAC has 
been configured to be as slow as DPCP-r). Remarkably, DPCP-d and REAPER are the fastest among all
methods with a difference of approximately two orders of magnitude from DPCP-r. However, as we saw 
above, none of these methods performs nearly as well as DPCP-r. From that perspective, DPCP-IRLS
is interesting, since it seems to be striking a balance between running time and performance. 

\setlength{\tabcolsep}{0.11em}
\begin{table}[t!]
	\centering	
	\ra{0.7}
	\caption{Mean running times in seconds, corresponding to the experiment of Fig. \ref{fig:MultipleDPCPRansacStyleDvsnalpha} for data balancing parameter $\alpha = 1$.} 
	\small	
	\begin{tabular}
		{lllrrrrrrrrrrrrrrrrrrrrrrrrrrrrrrrrrrrrrrrrrrrrrrrrrrrr}\toprule[1pt]
		 \phantom{a}  &&& \multicolumn{7}{c}{$n = 2$} &&&&&&&&& \multicolumn{7}{c}{$n = 4$} \\
		\cmidrule{5-8}		\cmidrule{21-24} 	
		 \phantom{a}  &&& \multicolumn{1}{c}{$D = 4$} &&& 
		 \multicolumn{1}{c}{$D= 9$} &&&  \multicolumn{1}{c}{$D=30$} &&&&&&&&& \multicolumn{1}{c}{$D = 4$} &&& 
		 \multicolumn{1}{c}{$D = 9$} &&& \multicolumn{1}{c}{$D=30$} \\		
		 	\cmidrule{3-4}	\cmidrule{7-8}   \cmidrule{10-11}   \cmidrule{18-19}  \cmidrule{22-22} \cmidrule{25-25}					 RANSAC &&&  $1.18$ &&& $1.25$ &&&  $1.76$ &&&&&&&&&  $8.27$ &&& $11.61$ &&& $16.00$	\\			 
			  REAPER &&&  $0.05$ &&& $0.04$ &&&  $0.04$ &&&&&&&&&  $0.18$ &&& $0.17$ &&& $0.19$	\\	
			  DPCP-r &&&  $1.18$ &&& $1.24$ &&&  $1.75$ &&&&&&&&&  $8.21$ &&& $11.55$ &&& $15.89$	\\			  
			  DPCP-d &&&  $0.02$ &&& $0.02$ &&&  $0.05$ &&&&&&&&&  $0.10$ &&& $0.16$ &&& $0.42$	\\	
			  DPCP-IRLS &&&  $0.12$ &&& $0.14$ &&&  $0.21$ &&&&&&&&&  $0.77$ &&& $0.81$ &&& $0.82$	\\				  		  			  \bottomrule[1pt]	
	\end{tabular}	\label{table:synthetic_D_vs_n_noise_outliers_T50_rt}	
\end{table} \normalsize

Moving on, we fix $D=9$ and vary the outlier ratio as $10\%,30\%,50\%$ (in the previous experiment the outlier ratio was $10\%$). Then the mean clustering accuracy over $50$ independent trials is 
shown in Fig. \ref{fig:MultipleDPCPRansacStyleRvsnalpha} and Fig. \ref{fig:MultipleDPCPRansacStyleRvsn08}.
In this experiment the hierarchy of the methods is clear: Algorithm \ref{alg:MultipleDPCPsequential} using DPCP-r and using DPCP-IRLS are the best and second best methods, respectively, while the rest of the methods perform
equally poorly. As an example, in the challenging scenario of $n=4,D=30$ and $50\%$ outliers, for $\alpha=0.6$, DPCP-r gives $74\%$ accuracy, while the next best method is DPCP-IRLS with $58\%$ accuracy; in that scenario RANSAC and REAPER give $38\%$ and $41\%$ accuracy respectively, while DPCP-r-d and DPCP-d give $41\%$ and $40\%$ respectively. Moreover, for $n=2,D=30$ and $\alpha=0.8$ DPCP-r and DPCP-IRLS give $95\%$ and $86\%$ accuracy, while all other methods give about $65\%$.

\subsection{3D Plane clustering of real Kinect data} \label{subsection:ExperimentsReal}

\indent \myparagraph{Dataset and objective} In this section we explore various Iterative Hyperplane Clustering\footnote{Recall from \S \ref{section:PriorArt} and \S \ref{subsection:AlgorithmsKHstyle} that by \emph{iterative hyperplane clustering}, we mean the process of estimating $n$ hyperplanes, then assigning each point to its closest hyperplane, then refining the hyperplanes associated to a cluster only from the points of the cluster, re-assigning points to hyperplanes and so on.} (IHL) algorithms using the benchmark dataset
NYUdepthV2 \cite{Silberman:ECCV12}. This dataset consists of $1449$ RGBd data instances acquired using
the Microsoft kinect sensor. Each instance of NYUdepthV2 corresponds to an indoor scene, and consists of the $480 \times 640 \times 3$ RGB data together with depth data for each pixel, i.e., a total of $480 \cdot 640$ depth values. In turn, the depth data can be used to reconstruct a $3D$ point cloud associated to the scene. In this experiment we use such $3D$ point clouds to learn plane arrangements and segment 
the pixels of the corresponding RGB images based on their plane membership. This is an important problem 
in robotics, where estimating the geometry of a scene is essential for successful robot navigation.

\indent \myparagraph{Manual annotation} 
While the coarse geometry of most indoor scenes can be approximately described by a union of a few ($\le 9$) planes, many points in the scene do not lie in these planes, and may thus be viewed as outliers. Moreover,
it is not always clear how many planes one should select or which these planes are. In fact, NYUdepthV2 
does not contain any ground truth annotation based on planes, rather the scenes are annotated semantically with a view to object recognition. For this reason, among a total of $1449$ scenes, we manually annotated $92$ scenes, in which the dominant planes are well-defined and capture most of the scene; see for example 
Figs. \ref{figure:IMAGE2_GC0}-\ref{figure:IMAGE2_AnnotationGC0} and \ref{figure:IMAGE5_GC0}-\ref{figure:IMAGE5_AnnotationGC0}. Specifically, for each of the $92$ images, at most $9$ dominant
planar regions were manually marked in the image and the set of pixels within these regions were declared inliers,
while the remaining pixels were declared outliers. For each planar region a ground truth normal vector was computed using DPCP-r. Finally, two planar regions that were considered distinct during manual annotation, were merged if the absolute inner product of their corresponding normal vectors was above $0.999$.  

\indent \myparagraph{Pre-processing} For computational reasons, the hyperplane clustering algorithms that we use (to be described in the next paragraph) do not act directly on the original $3D$ point cloud, rather on a weighted subset of it, corresponding to a superpixel representation of each image. In particular, each
$480 \times 640 \times 3$ RGB image is segmented to about $1000$ superpixels and the entire $3D$ point sub-cloud corresponding to each superpixel is replaced by the point in the geometric center of the superpixel.  To account for
the fact that the planes associated with an indoor scene are affine, i.e., they do not pass through a common origin,
we work in homogeneous coordinates, i.e., we append a fourth coordinate to each $3D$ point representing a superpixel and normalize it to unit $\ell_2$-norm. Finally, a weight is assigned to each representative $3D$ point, equal to the number of pixels in the underlying superpixel. The role of this weight is to regulate the influence of each point in the modeling error (points representing larger superpixels should have more influence). 

\indent \myparagraph{Algorithms} The first algorithm that we test is the sequential RANSAC algorithm (SHL-RANSAC), which identifies one plane at a time. Secondly, we explore a family of variations of the IHL algorithm (see \S \ref{section:PriorArt})
based on SVD, DPCP, REAPER and RANSAC. In particular, IHL(2)-SVD indicates the classic IHL algorithm which computes normal vectors through the Singular Value Decomposition (SVD), and minimizes an $\ell_2$ objective (this is K-Hyperplanes). IHL(1)-DPCP-r-d, IHL(1)-DPCP-d and IHL(1)-DPCP-IRLS, denote IHL variations of DPCP according to Algorithm \ref{alg:IHL-DPCP}, depending on which method is used to solve the DPCP problem \eqref{eq:ell1} \footnote{IHL(1)-DPCP-r was not included since it was slowing down the experiment considerably, while its performance was similar to the rest of the DPCP methods.}. Similarly, IHL(1)-REAPER and IHL(1)-RANSAC denote the obvious adaptation of IHL where the normals are computed with REAPER and RANASC, respectively, and an $\ell_1$ objective is minimized. 

A third method that we explore is a hybrid between Algebraic Subspace Clustering (ASC), RANSAC and IHL,
(IHL-ASC-RANSAC). First, the
vanishing polynomial associated to ASC (see \S \ref{section:PriorArt}) is computed with RANSAC instead of SVD, which is the traditional way; this ensures robustness to outliers.
Then spectral clustering applied on the angle-based affinity associated to ASC (see equation \eqref{eq:ABA})  
yields $n$ clusters. Finally, one iteration of IHL-RANSAC refines these clusters and yields a normal vector per cluster (the normal vectors are necessary for the post-processing step).

\setlength{\tabcolsep}{0.17em}
\begin{table}[t]
	\centering	
	\ra{0.7}
	\caption{$3D$ plane clustering error for a subset of the real Kinect dataset NYUdepthV2. $n$ is the number of fitted planes. GC(0) and GC(1) refer to clustering error without or with spatial smoothing, respectively. } 
	\small	
	\begin{tabular}
		{lllllrrrrrrrrrrrrrrrrrrrrrrrrr}\toprule[1pt]		
		\phantom{a}  && \phantom{a}  &&& \multicolumn{3}{c}{$n \le 2$} & 
		\phantom{a} && \multicolumn{3}{c}{$n \le 4$} & \phantom{a} && \multicolumn{3}{c}{all} \\	
		\cmidrule{5-8} \cmidrule{10-13} \cmidrule{15-18}   
	 \phantom{a}  &&	\phantom{a} &&& GC(0) && GC(1) & \phantom{a} && GC(0) && GC(1) & \phantom{a} &&GC(0) && GC(1)  \\ 				 	
		\midrule[0.5pt]
		$(\tau=10^{-1})$&& IHL-ASC-RANSAC &&& $16.25$&&$12.68$  &&&  $28.22$ &&$19.56$ &&&  $32.95$&&$19.83$  \\ 
		&& SHL-RANSAC &&& $16.01$&&$12.89$  &&&  $29.63$ &&$23.96$ &&&  $36.00$&&$28.70$  \\          		
		&& IHL(1)-RANSAC    &&& $12.73$&&$9.12$  &&&  $22.16$ &&$14.56$ &&&  $28.61$&&$18.22$  \\
		&& IHL(1)-DPCP-r-d   &&& $\boldsymbol{12.25}$&&$\boldsymbol{8.72}$  &&&  $\boldsymbol{19.78}$ &&$\boldsymbol{13.15}$ &&&  $\boldsymbol{24.37}$&&$\boldsymbol{15.55}$ \\		
		&& IHL(1)-DPCP-d        &&& $12.45$&&$8.95$  &&&  $19.91$ &&$13.30$ &&&  $24.66$&&$16.04$  \\
		\midrule[0.1pt]
		$(\tau=10^{-2})$&& IHL-ASC-RANSAC &&& $9.50$&&$5.19$  &&&  
		$19.72$ &&$10.29$ &&&  $25.15$&&$12.18$  \\ 
		&& SHL-RANSAC &&& $\boldsymbol{6.27}$&&$\boldsymbol{3.29}$  &&&  $\boldsymbol{12.84}$ &&$\boldsymbol{6.75}$ &&&  $\boldsymbol{19.17}$&&$\boldsymbol{9.81}$ \\          		
		&& IHL(1)-RANSAC   &&& $7.97$&&$5.06$  &&&  $14.37$ &&$8.78$ &&&  $20.34$&&$12.43$  \\
		&& IHL(1)-DPCP-r-d   &&& $12.70$&&$9.07$  &&&  $21.46$ &&$13.98$ &&&  $25.94$&&$16.20$  \\		
		&& IHL(1)-DPCP-d       &&& $12.70$&&$9.08$  &&&  $21.50$ &&$14.03$ &&&  $26.04$&&$16.22$  \\
		\midrule[0.1pt]
		$(\tau=10^{-3})$&& IHL-ASC-RANSAC &&& $8.75$&&$4.80$  &&&  
		$20.35$ &&$10.72$ &&&  $24.46$&&$\boldsymbol{11.95}$  \\ 
		&&  SHL-RANSAC &&& $15.26$&&$8.34$  &&&  $25.89$ &&$11.49$ &&&  $33.08$&&$13.90$  \\          		
		&& IHL(1)-RANSAC    &&& $\boldsymbol{7.48}$&&$\boldsymbol{4.79}$  &&&  $\boldsymbol{13.86}$ &&$\boldsymbol{8.79}$ &&&  $\boldsymbol{19.39}$&&$12.07$  \\
		&& IHL(1)-DPCP-r-d  &&& $12.93$&&$9.33$  &&&  $21.06$ &&$13.60$ &&&  $25.65$&&$16.27$  \\		
		&& IHL(1)-DPCP-d      &&& $12.93$&&$9.33$  &&&  $21.06$ &&$13.59$ &&&  $25.63$&&$16.23$  \\	
		\midrule[0.5pt]
		$(mean)$&& IHL-ASC-RANSAC&&& $11.50$&&$7.56$  &&&  
		$22.76$ &&$13.52$ &&&  $27.52$&&$14.65$  \\ 
		$(mean)$&&  SHL-RANSAC &&& $12.51$&&$8.17$  &&&  
		$22.78$ &&$14.07$ &&&  $29.42$&&$17.47$\\          		
		$(mean)$&& IHL(1)-RANSAC   &&& $\boldsymbol{9.39}$&&$\boldsymbol{6.32}$  &&&  
		$\boldsymbol{16.80}$ &&$\boldsymbol{10.71}$ &&&  $\boldsymbol{22.78}$&&$\boldsymbol{14.24}$ \\
		$(mean)$&& IHL(1)-DPCP-r-d   &&& $12.63$&&$9.04$  &&&  
		$20.77$ &&$13.58$ &&&  $25.32$&&$16.01$ \\		
		$(mean)$&& IHL(1)-DPCP-d       &&& $12.69$&&$9.12$  &&&  
		$20.83$ &&$13.64$ &&&  $25.45$&&$16.16$  \\			
		&& IHL(2)-SVD   &&& $13.36$&&$9.96$  &&&  
		$21.85$ &&$14.40$ &&&  $26.22$&&$16.71$    \\		
		&& IHL(1)-REAPER   &&& $12.45$&&$8.98$  &&&  
		$20.94$ &&$13.71$ &&&  $25.52$&&$16.27$  \\
		&& IHL(1)-DPCP-IRLS    &&& $12.47$&&$9.01$  &&&  
		$20.77$ &&$13.64$ &&&  $25.38$&&$16.10$ \\ 					
		\bottomrule[1pt]	
	\end{tabular}	\label{table:NYU}
\end{table} \normalsize

\begin{figure}[t!]
	\centering
	\subfigure[original image]{\label{figure:IMAGE5_GC0}\includegraphics[width=0.3\linewidth]{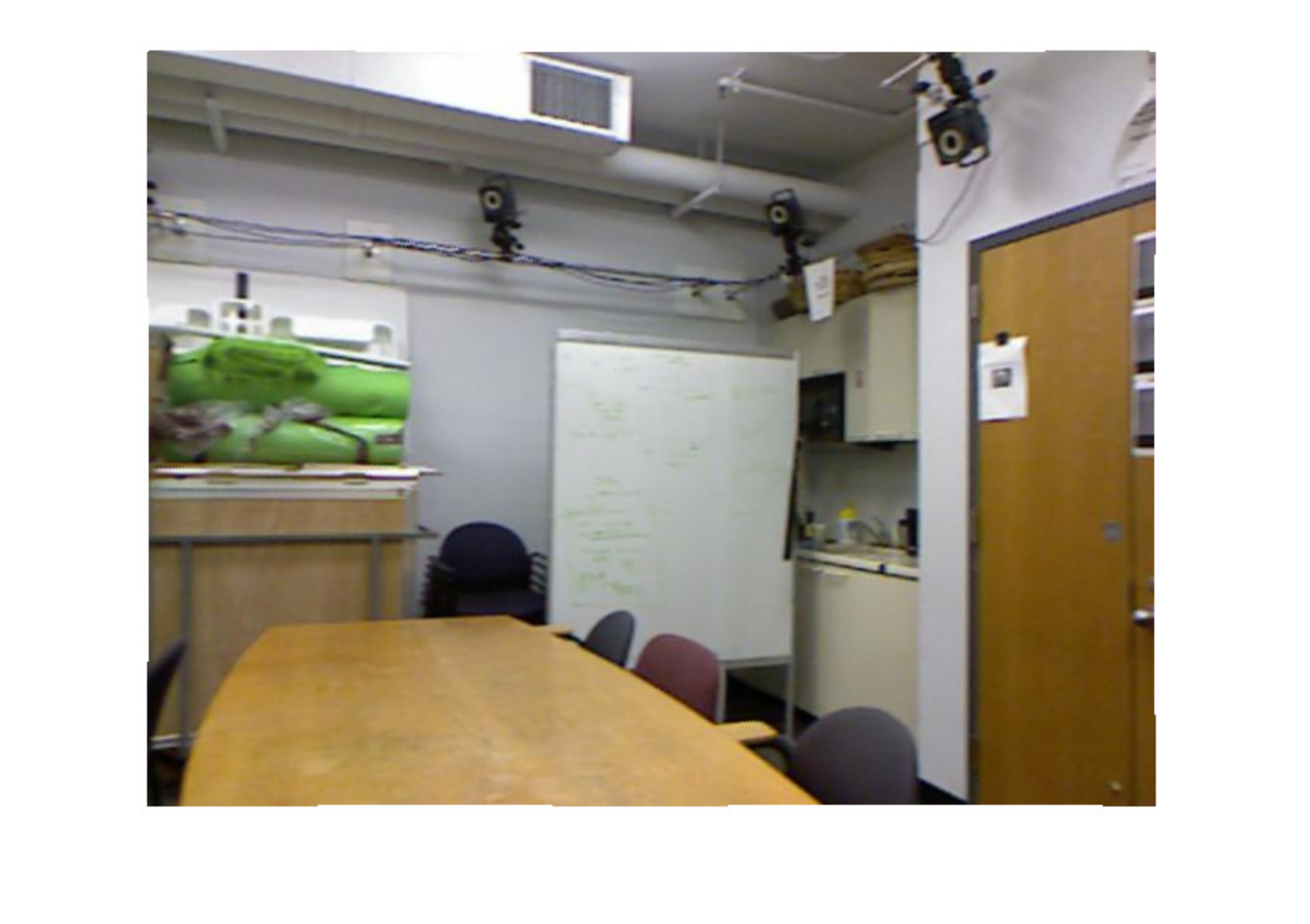}}			\subfigure[annotation]{\label{figure:IMAGE5_AnnotationGC0}\includegraphics[width=0.3\linewidth]{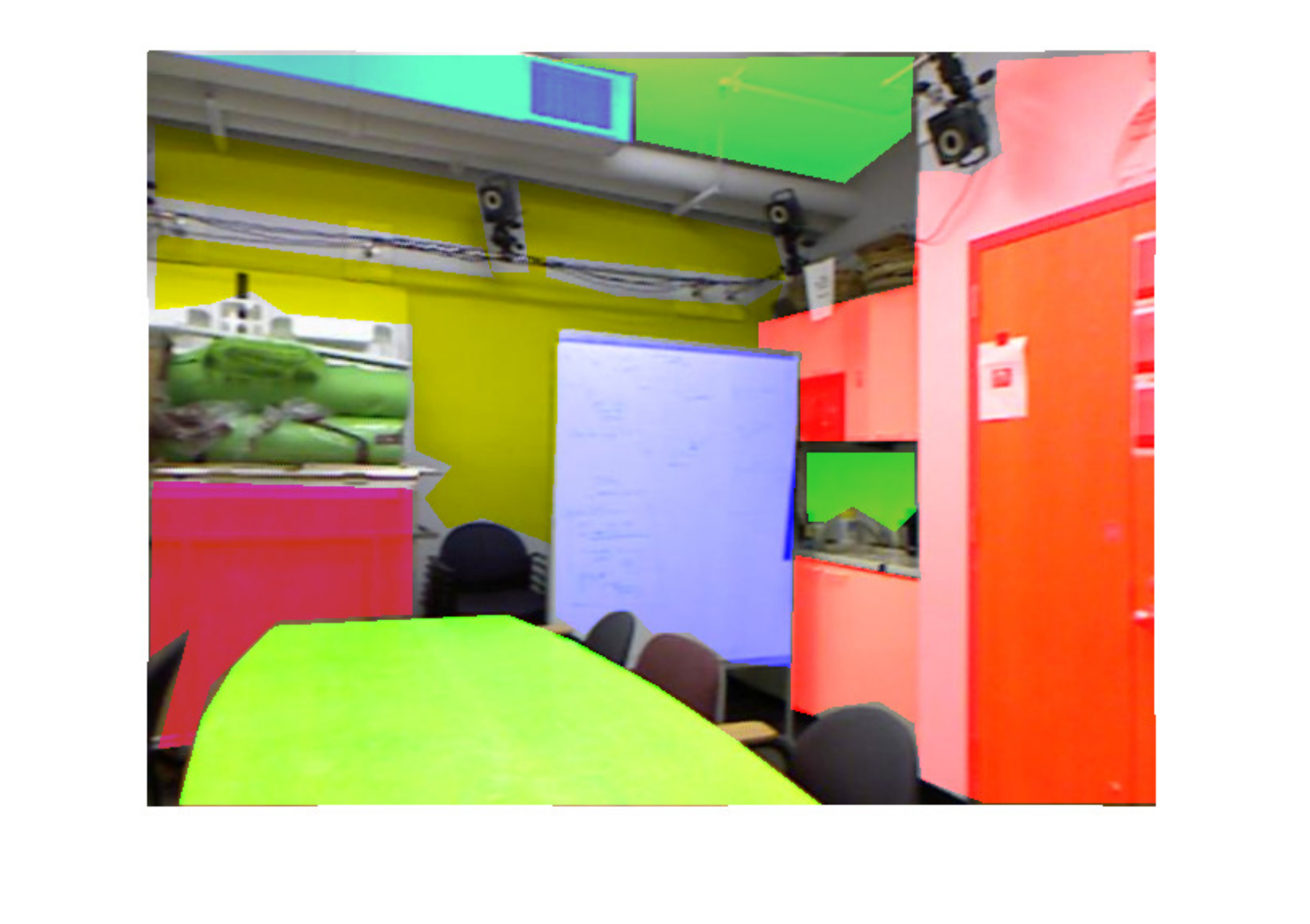}}	
	\subfigure[IHL-ASC-RANSAC ($24.7\%$)]{\label{figure:I5_ASC_RANSAC_GC0}\includegraphics[width=0.3\linewidth]{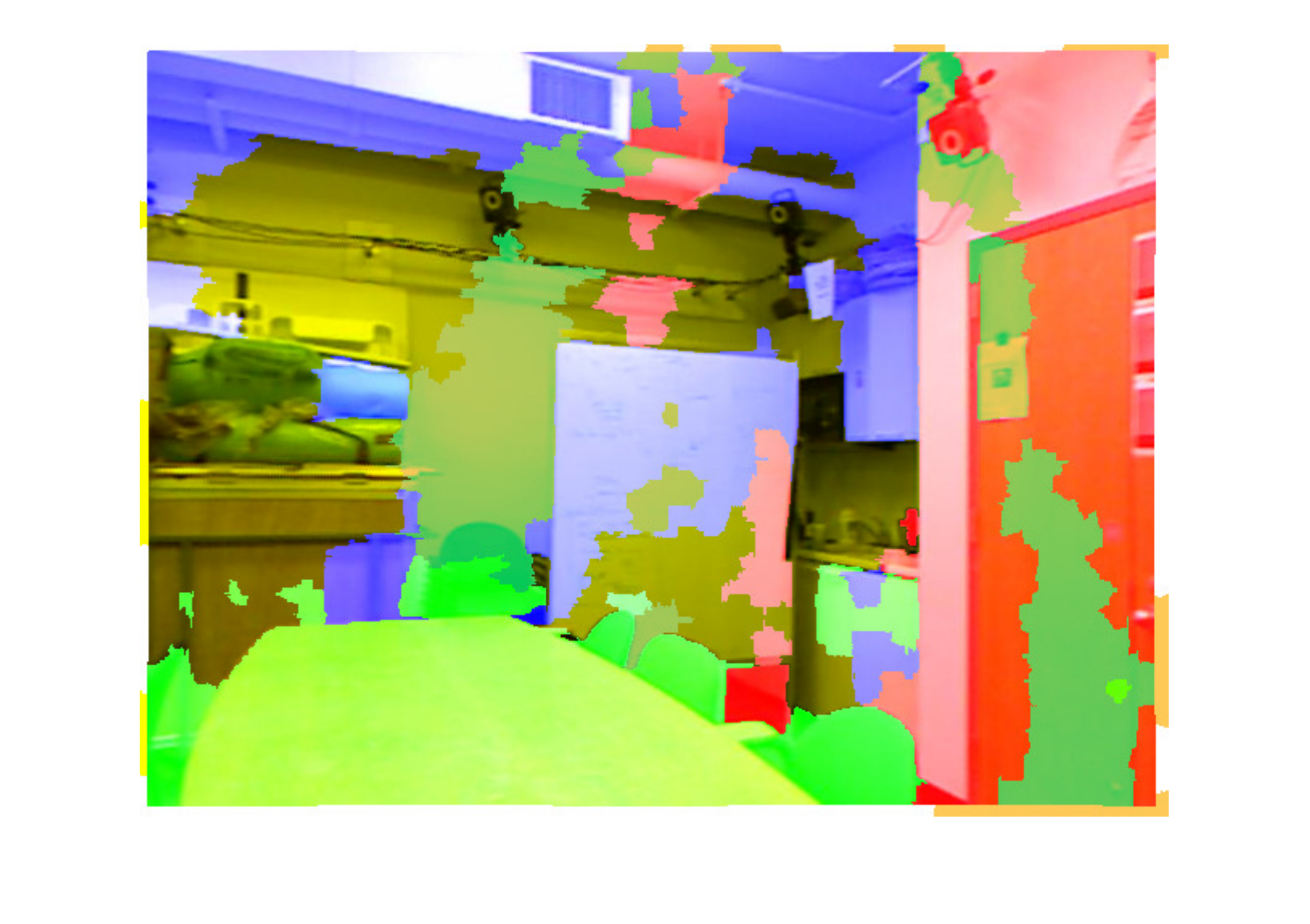}}\subfigure[SHL-RANSAC ($9.11\%$)]{\label{figure:I5_RANSAC_RANSAC_GC0}\includegraphics[width=0.3\linewidth]{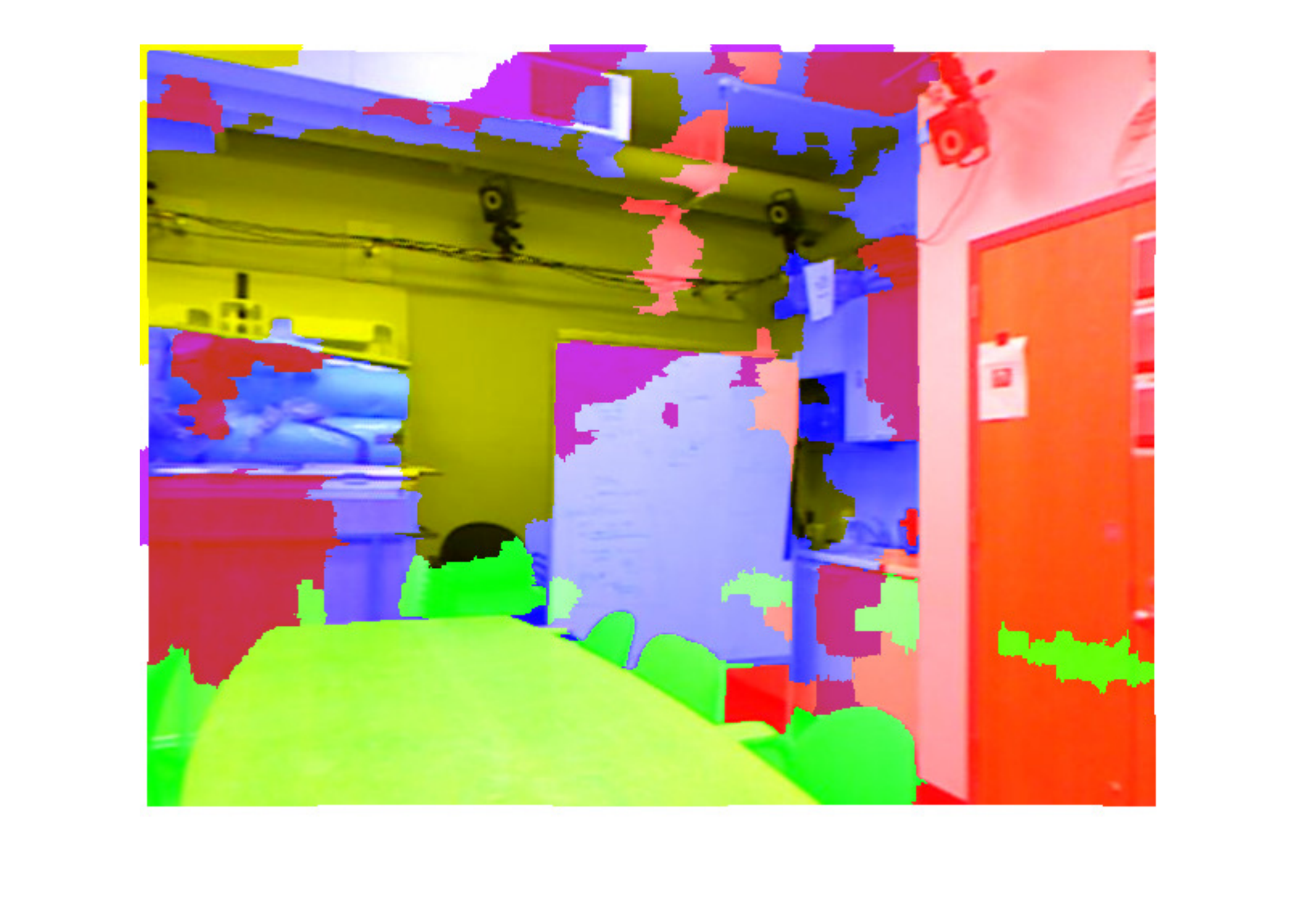}}
	\subfigure[IHL(1)-RANSAC ($11.9\%$)]{\label{figure:I5_KH_RANSAC_GC0}\includegraphics[width=0.3\linewidth]{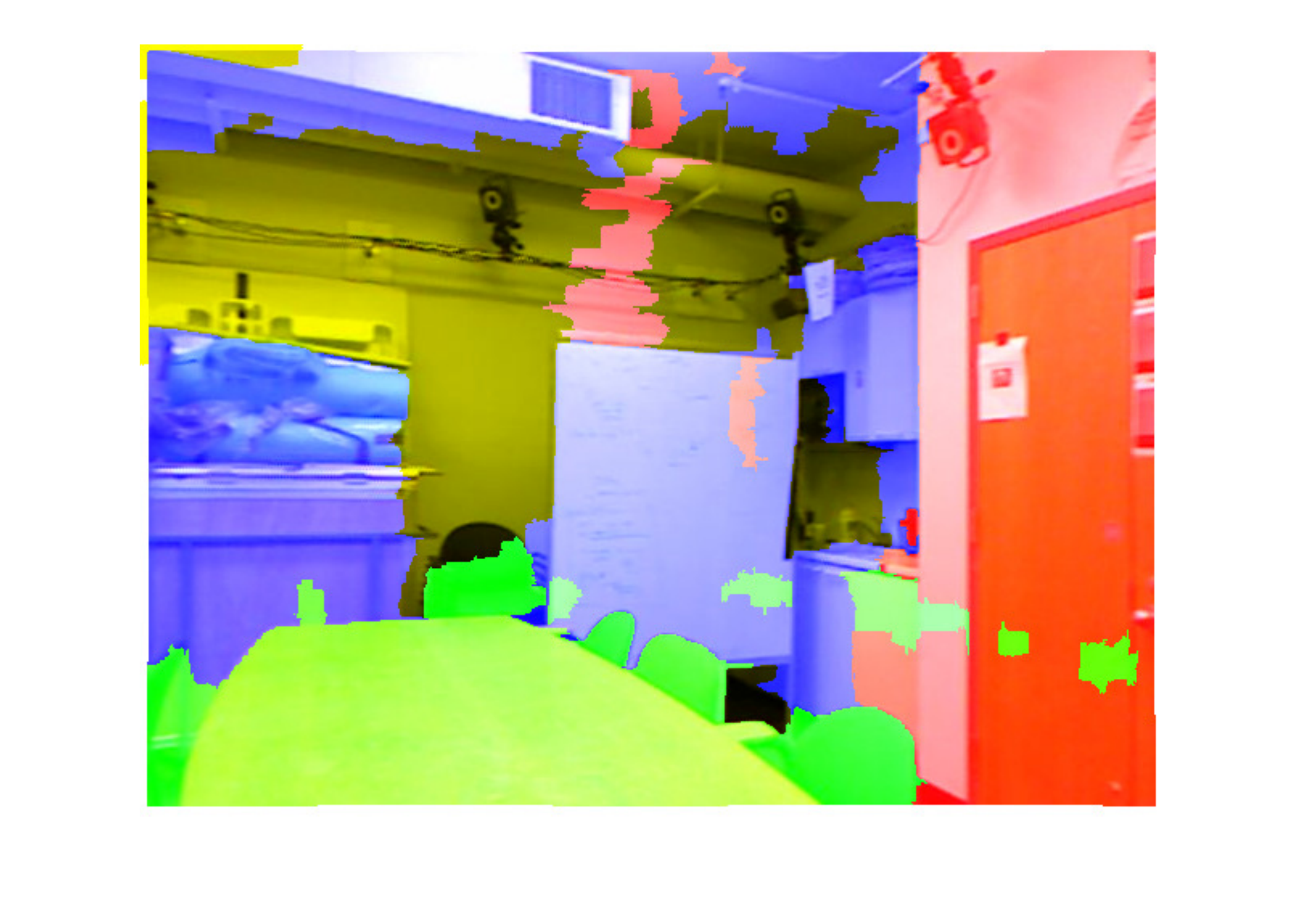}}	
	\subfigure[IHL(1)-DPCP-r-d ($9.88\%$)]{\label{figure:I5_KH_DPCP_GC0}\includegraphics[width=0.3\linewidth]{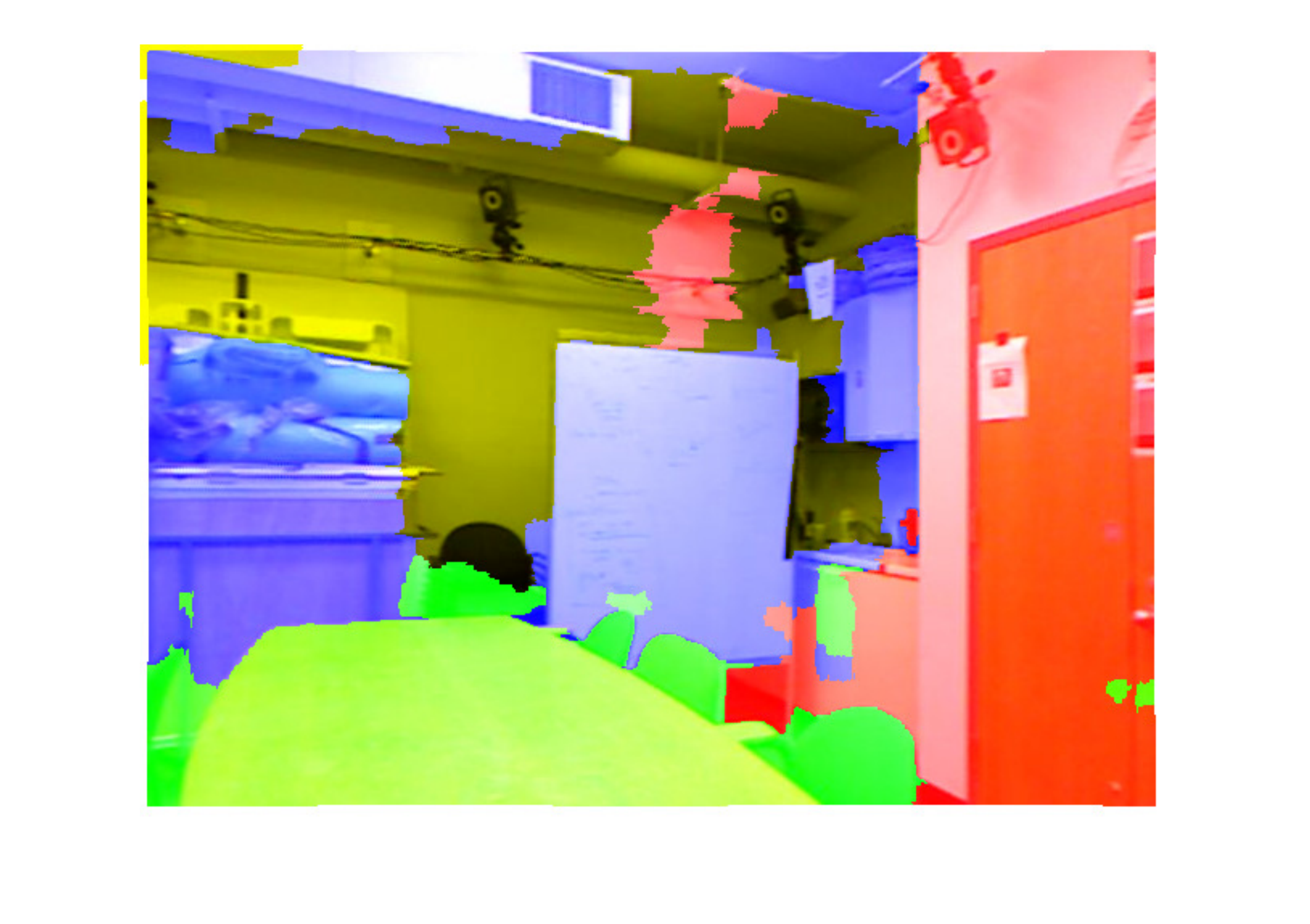}}	
	\subfigure[IHL(1)-DPCP-d ($10.24\%$)]{\label{figure:I5_KH_ADM_GC0}\includegraphics[width=0.3\linewidth]{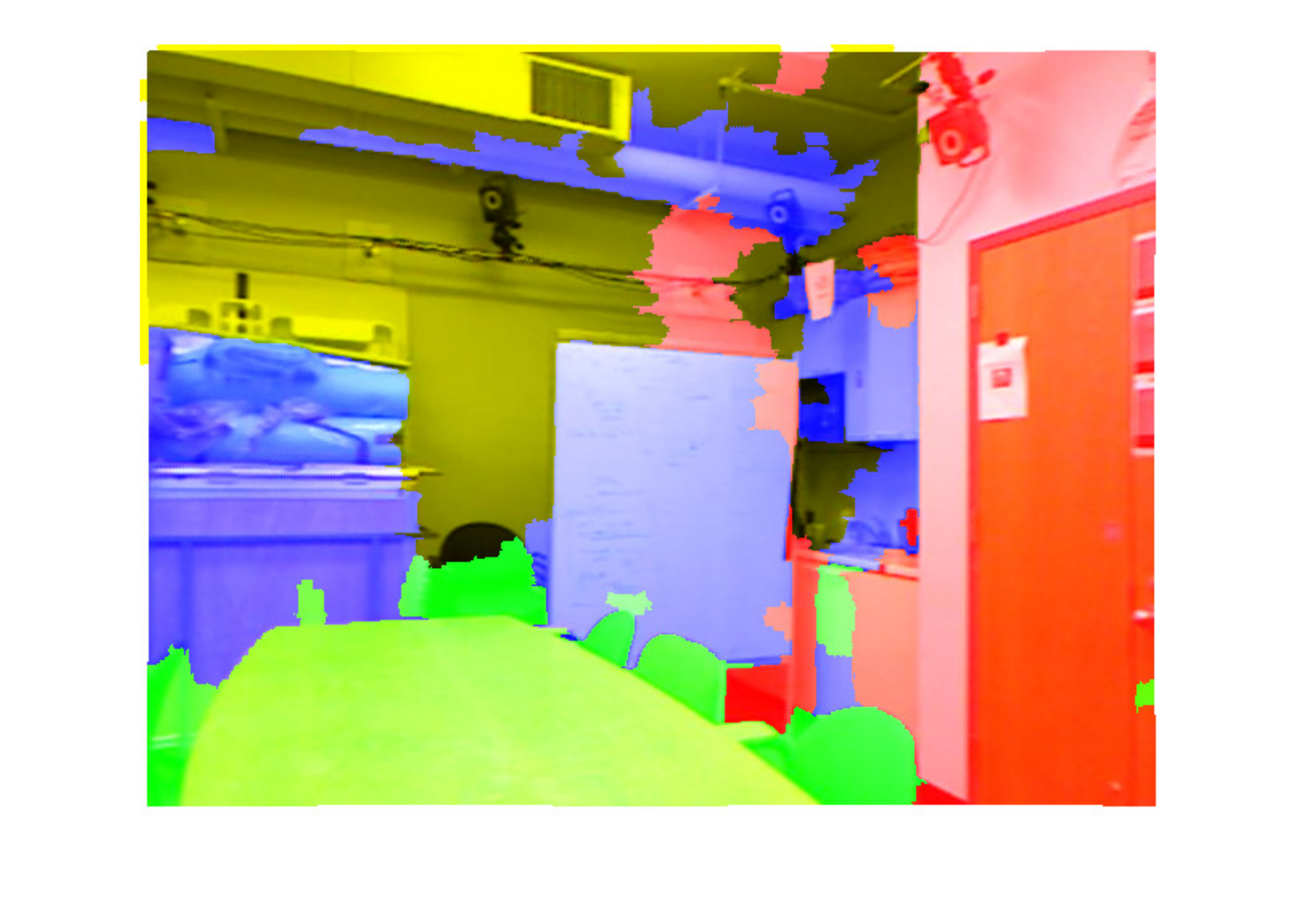}}		
	\subfigure[IHL(2)-SVD ($9.78\%$)]{\label{figure:I5_KH_SVD_GC0}\includegraphics[width=0.3\linewidth]{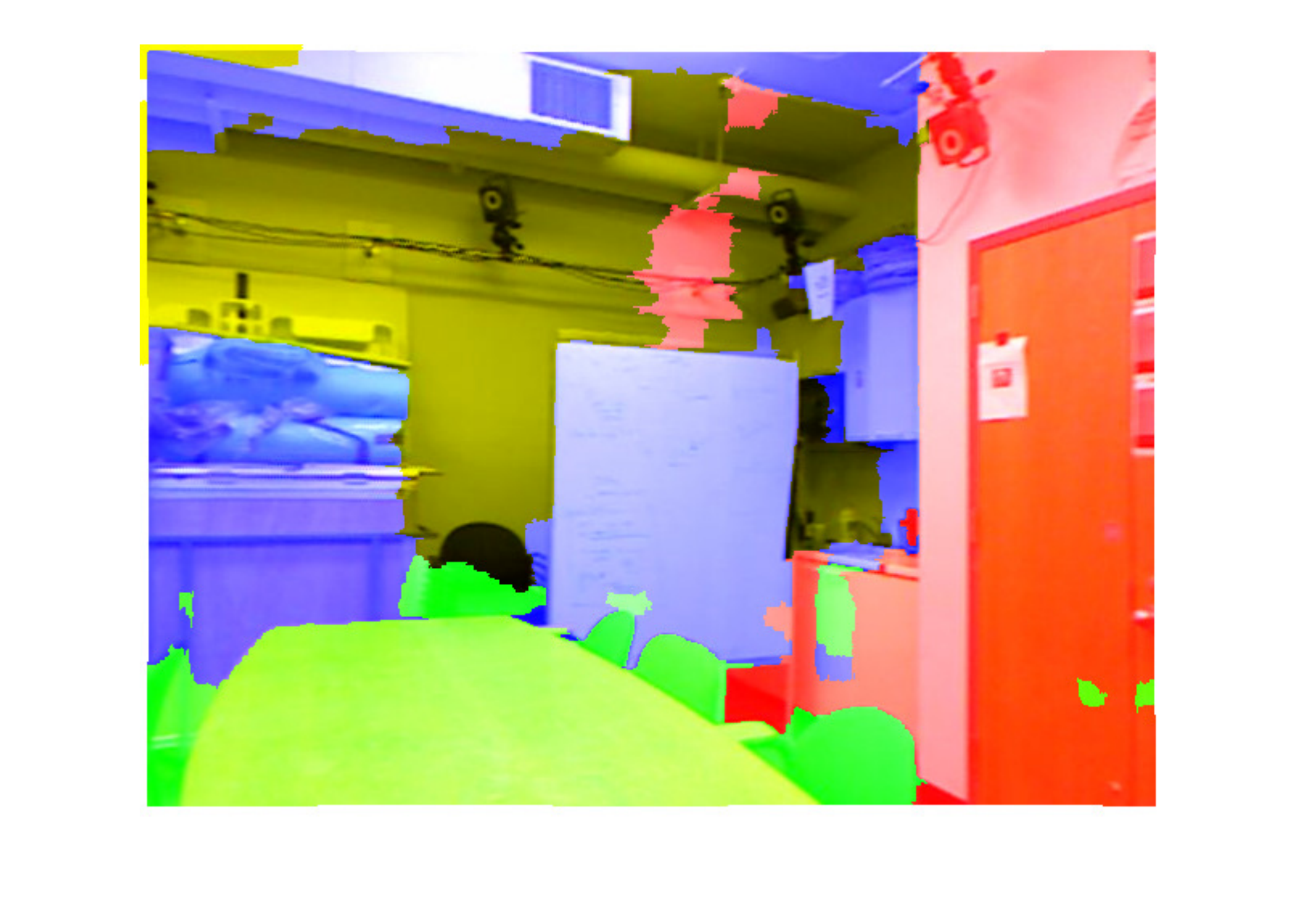}}	
	\subfigure[IHL(1)-REAPER ($9.05\%$)]{\label{figure:I5_KH_REAPER_GC0}\includegraphics[width=0.3\linewidth]{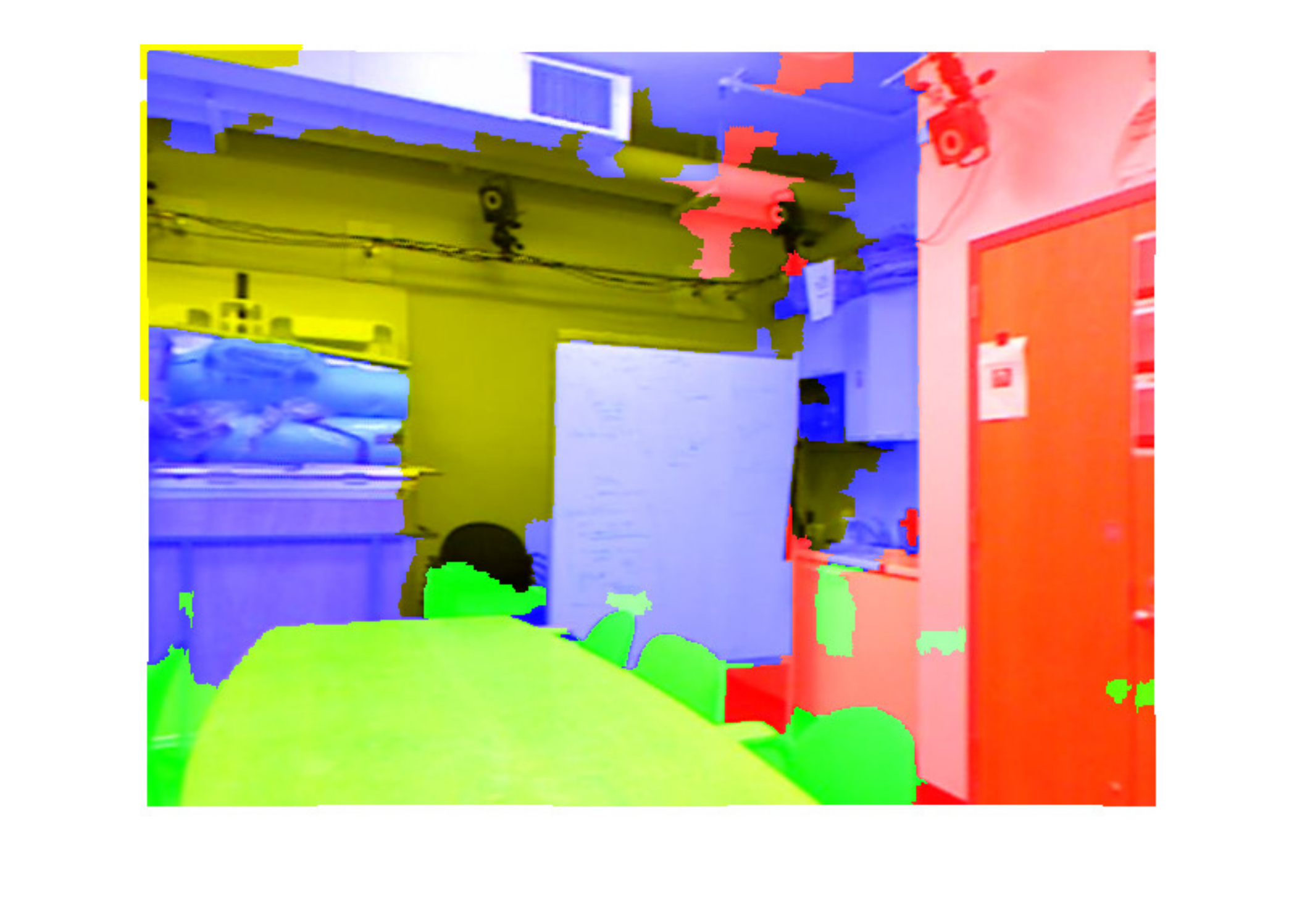}}
	\subfigure[IHL(1)-DPCP-IRLS ($9.78\%$)]{\label{figure:I5_KH_L12_IRLS_GC0}\includegraphics[width=0.3\linewidth]{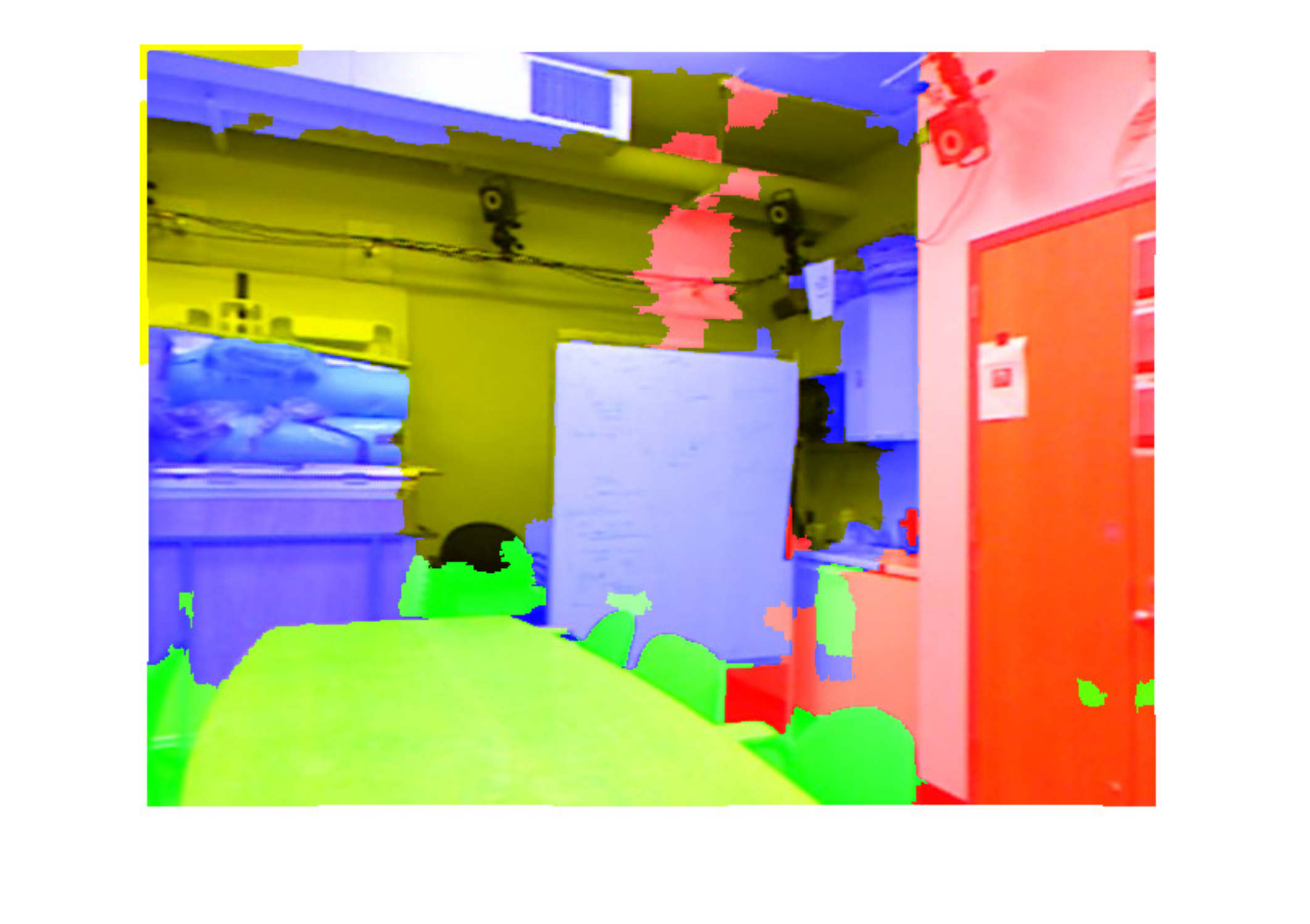}}			
	\caption{Segmentation into planes of image $5$ in dataset NYUdepthV2 without spatial smoothing. Numbers are segmentation errors.}\label{figure:IMAGE5-GC0}
\end{figure} 

\begin{figure}[t!]
	\centering
	\subfigure[original image]{\label{figure:OriginalImage}\includegraphics[width=0.3\linewidth]{IMAGE5-eps-converted-to.pdf}}			\subfigure[annotation]{\label{figure:IMAGE5_Annotation}\includegraphics[width=0.3\linewidth]{IMAGE5_Annotation-eps-converted-to.pdf}}\\	\subfigure[IHL-ASC-RANSAC ($7.98\%$)]{\label{figure:I5_ASC_RANSAC_GC1}\includegraphics[width=0.3\linewidth]{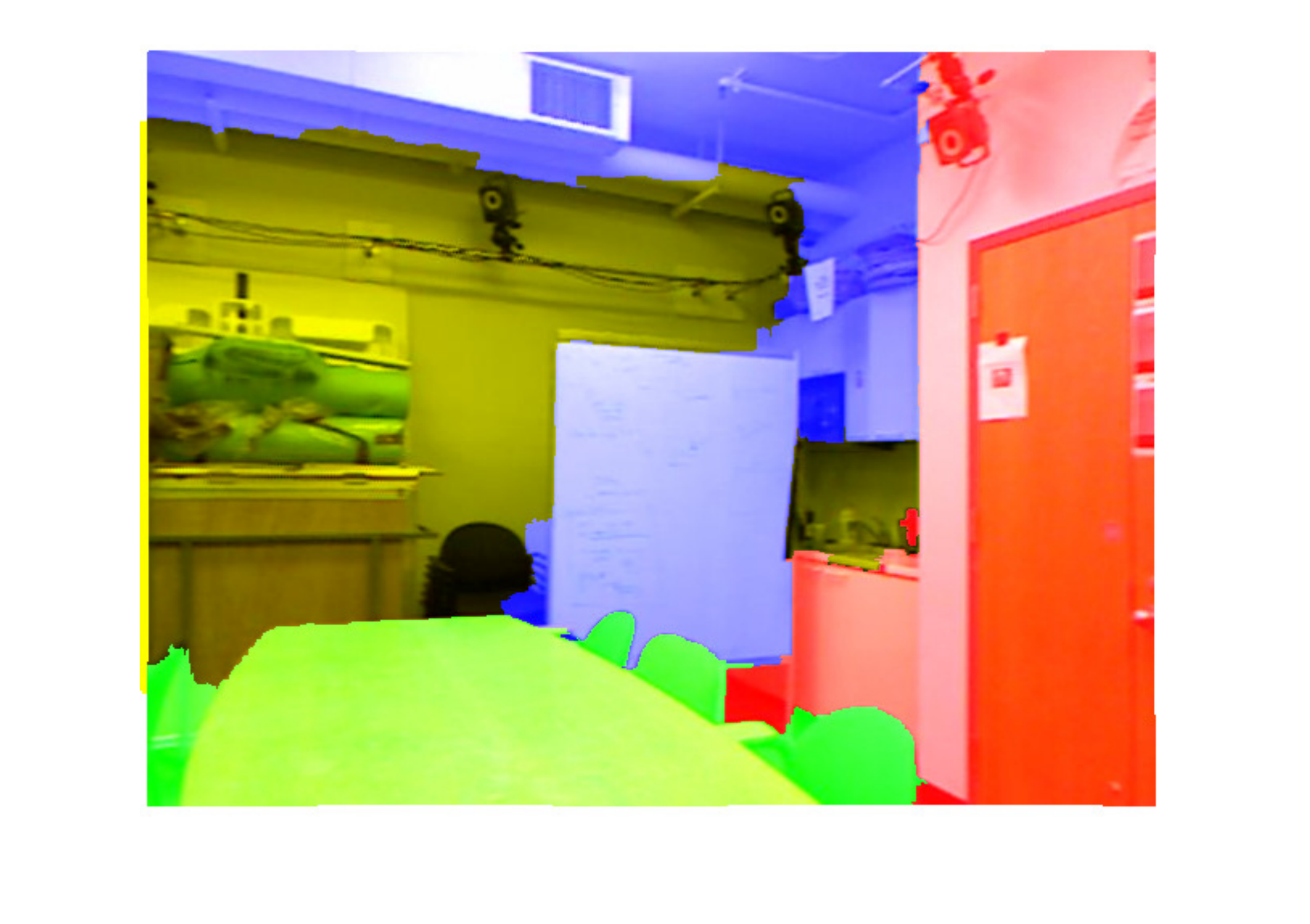}}
	\subfigure[SHL-RANSAC ($3.24\%$)]{\label{figure:I5_RANSAC_RANSAC_GC1}\includegraphics[width=0.3\linewidth]{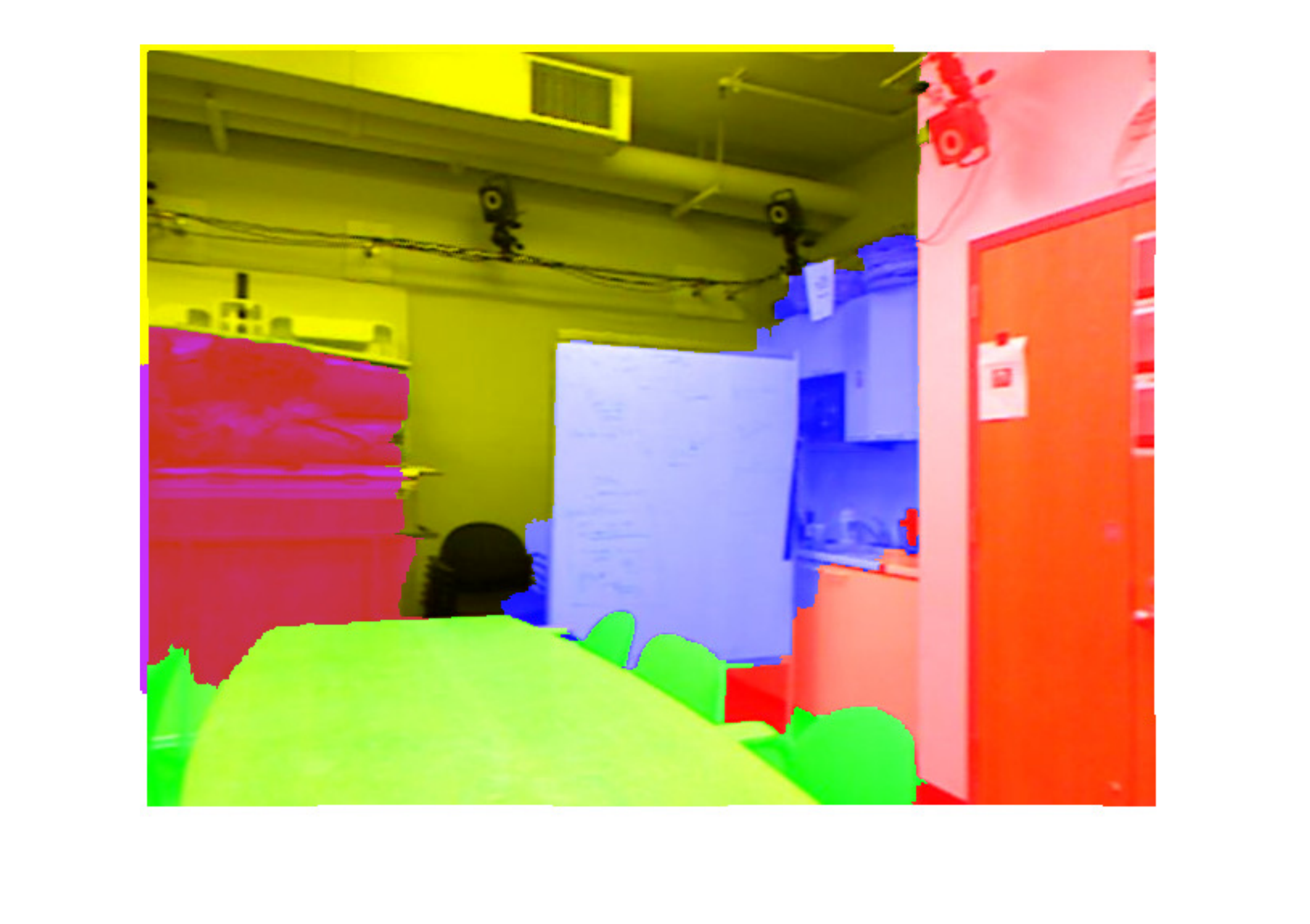}}
	\subfigure[IHL(1)-RANSAC ($8.36\%$)]{\label{figure:I5_KH_RANSAC_GC1}\includegraphics[width=0.3\linewidth]{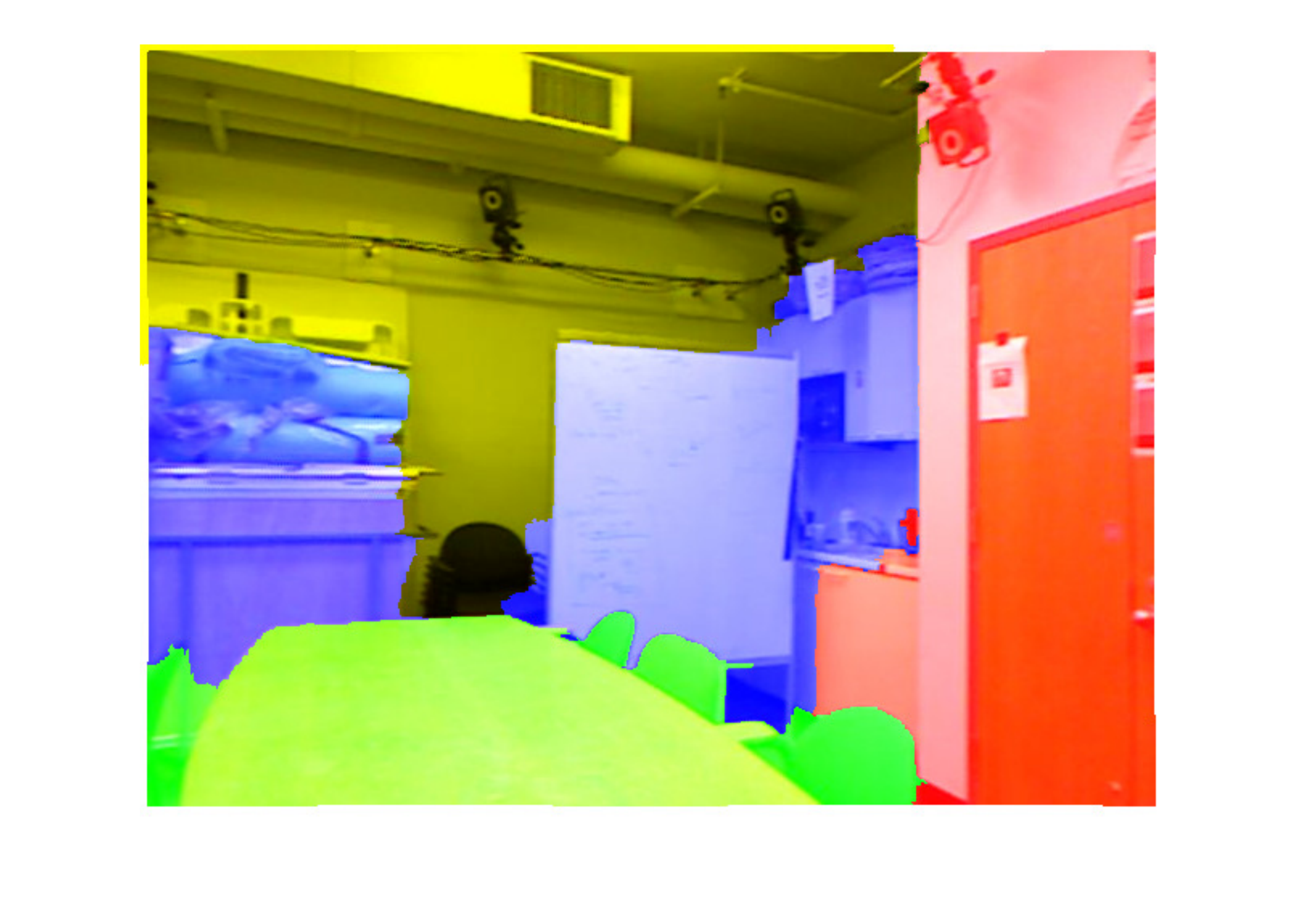}}	\subfigure[IHL-DPCP-ADM ($8.36\%$)]{\label{figure:I5_KH_DPCP_GC1}\includegraphics[width=0.3\linewidth]{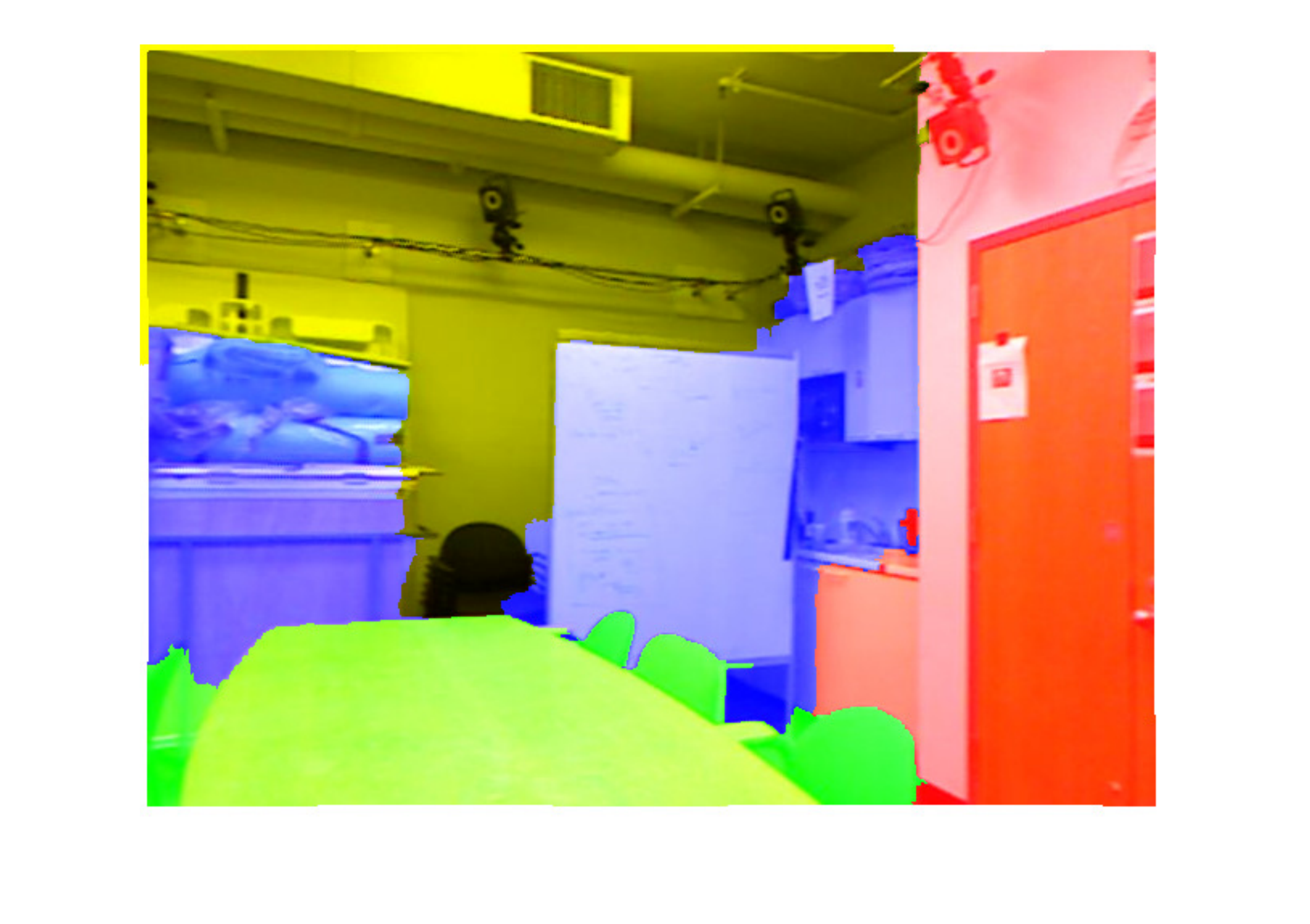}}	
	\subfigure[IHL(1)-DPCP-d ($8.057\%$)]{\label{figure:I2_KH_ADM_GC1}\includegraphics[width=0.3\linewidth]{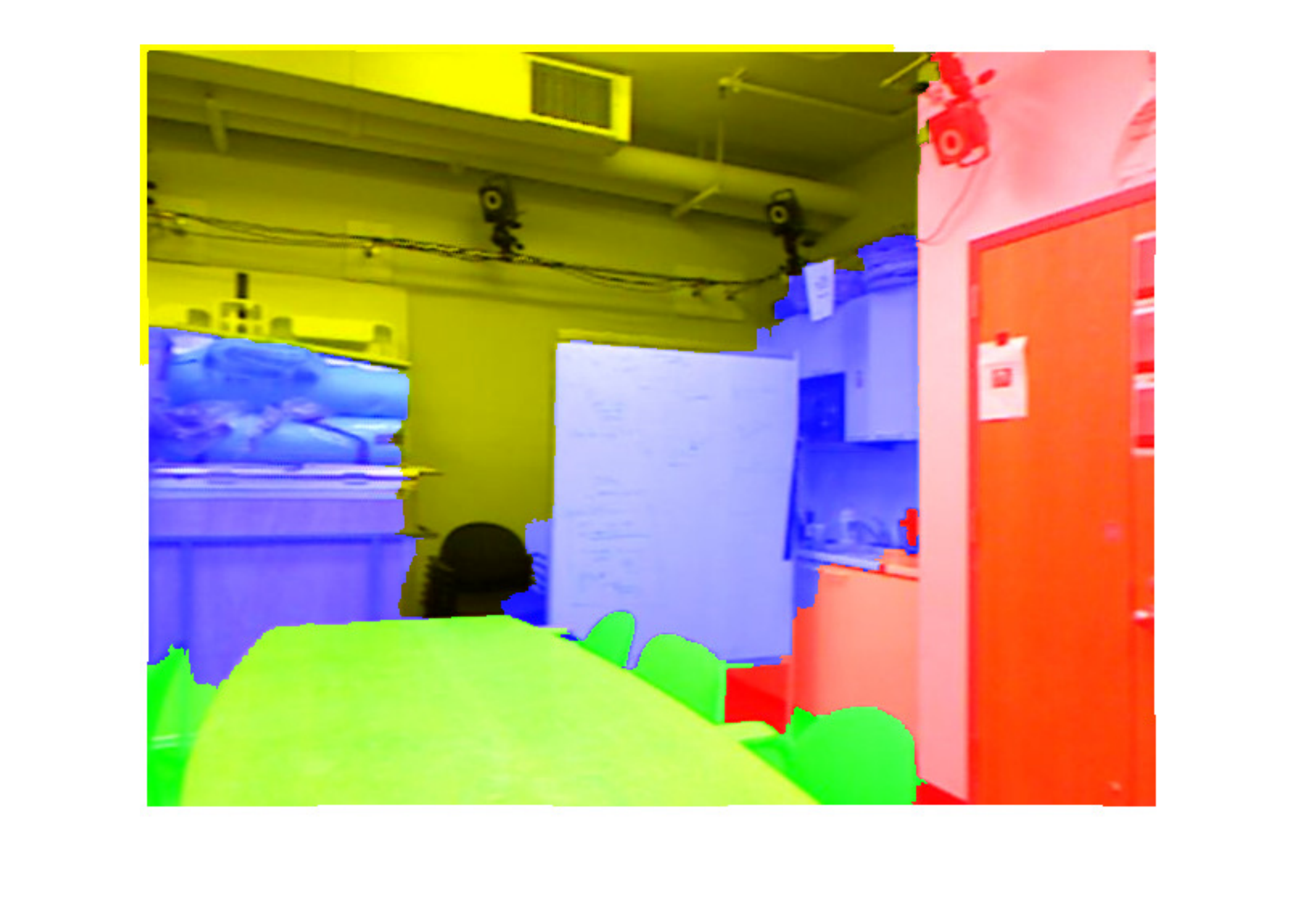}}		
	\subfigure[IHL(2)-SVD ($8.36\%$)]{\label{figure:I5_KH_SVD_GC1}\includegraphics[width=0.3\linewidth]{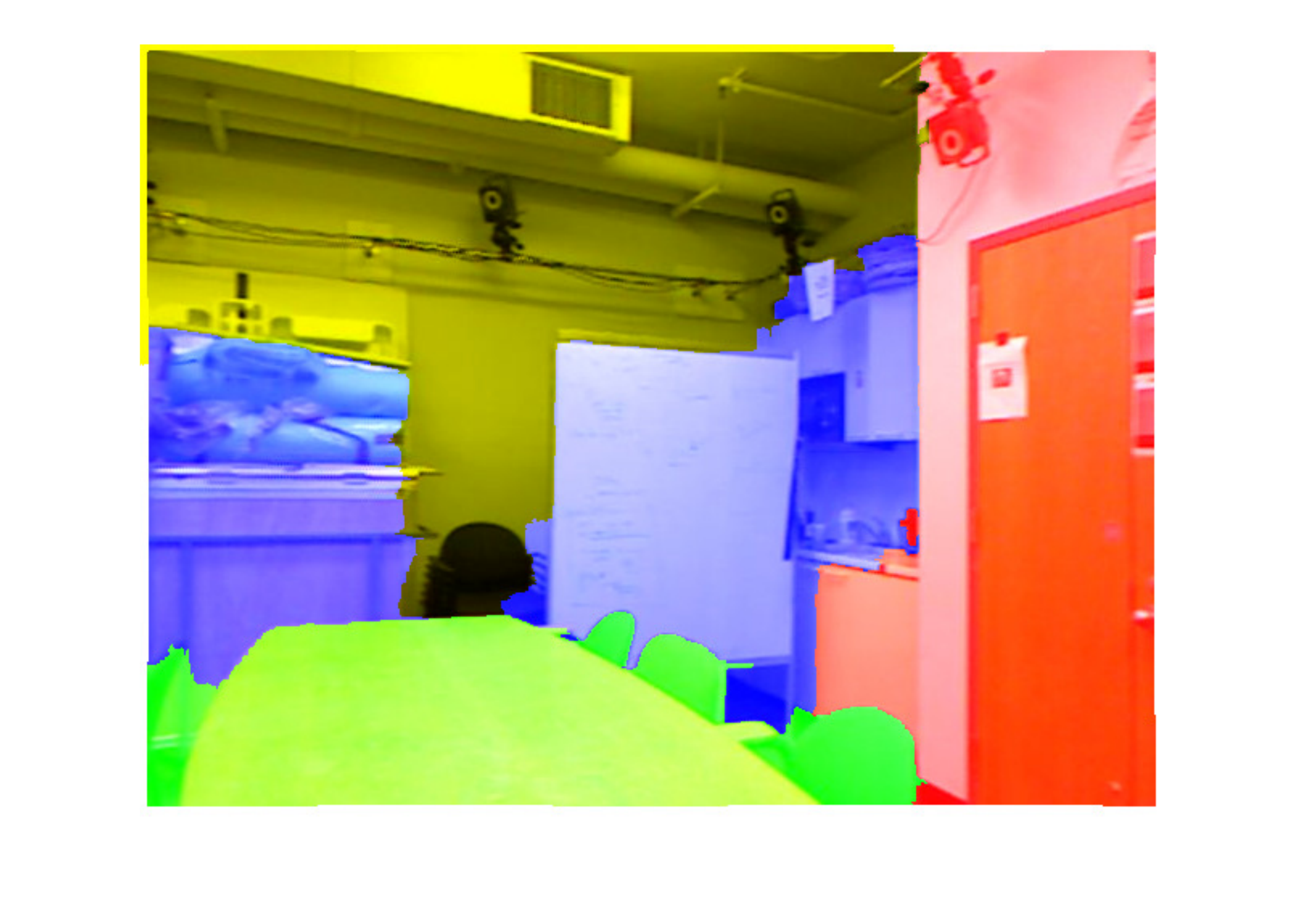}}	
	\subfigure[IHL(1)-REAPER ($8.07\%$)]{\label{figure:I5_KH_REAPER_GC1}\includegraphics[width=0.3\linewidth]{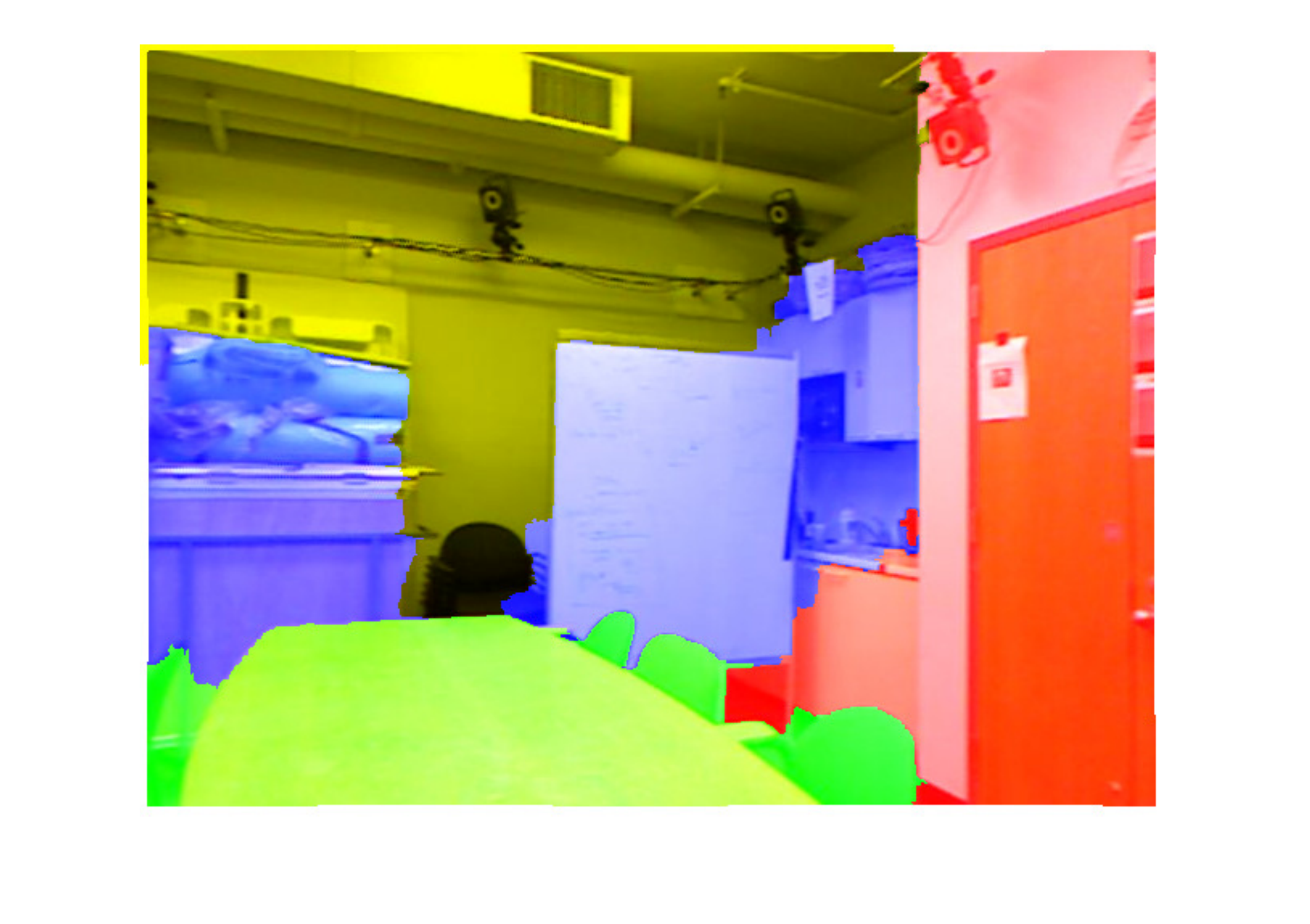}}
	\subfigure[IHL(1)-DPCP-IRLS ($8.36\%$)]{\label{figure:I5_KH_L12_IRLS_GC1}\includegraphics[width=0.3\linewidth]{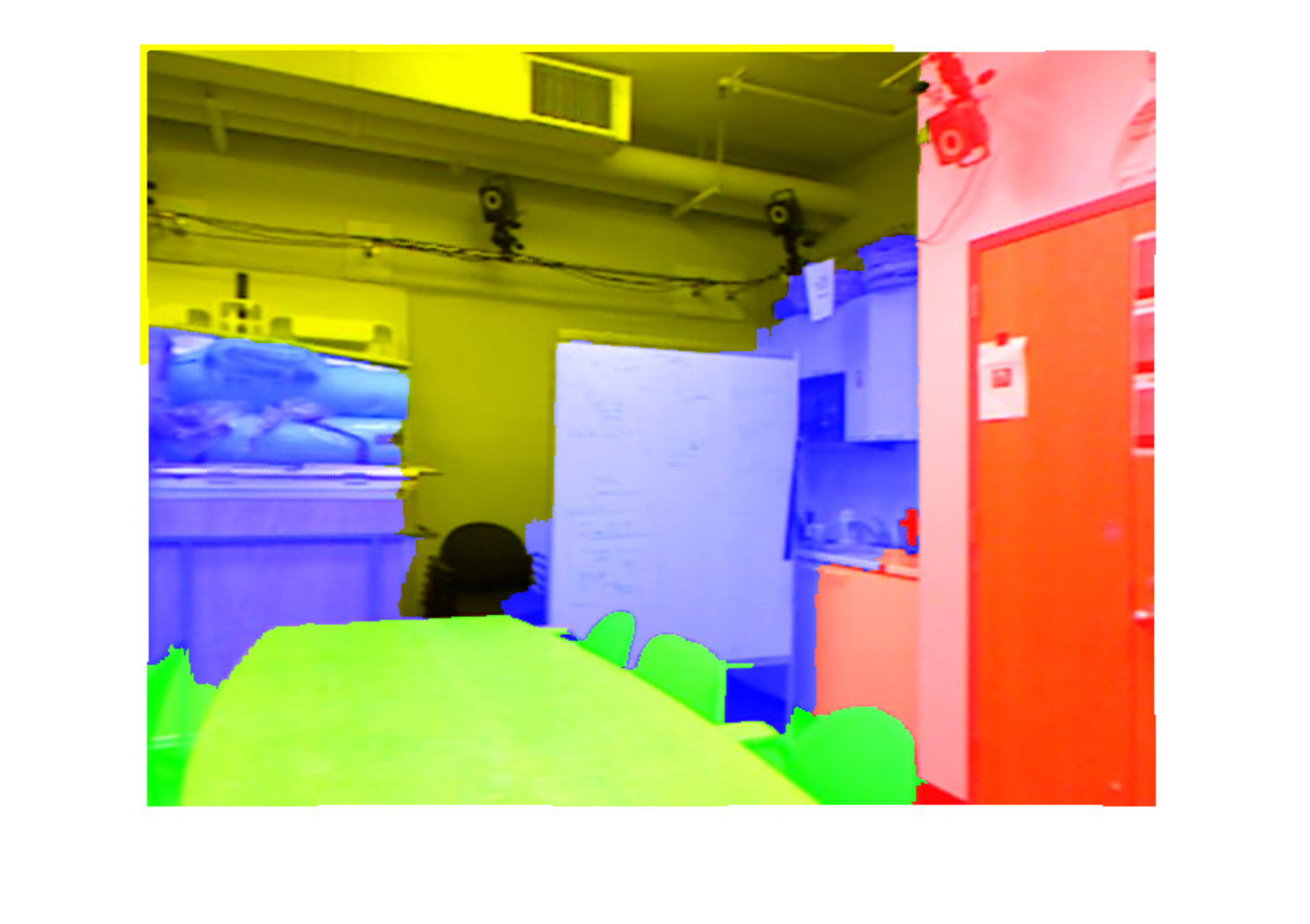}}			
	\caption{Segmentation into planes of image $5$ in dataset NYUdepthV2 with spatial smoothing. Numbers are segmentation errors.}\label{figure:IMAGE5-GC1}
\end{figure} 

 \begin{figure}[t!]
	\centering
	\subfigure[original image]{\label{figure:IMAGE2_GC0}\includegraphics[width=0.3\linewidth]{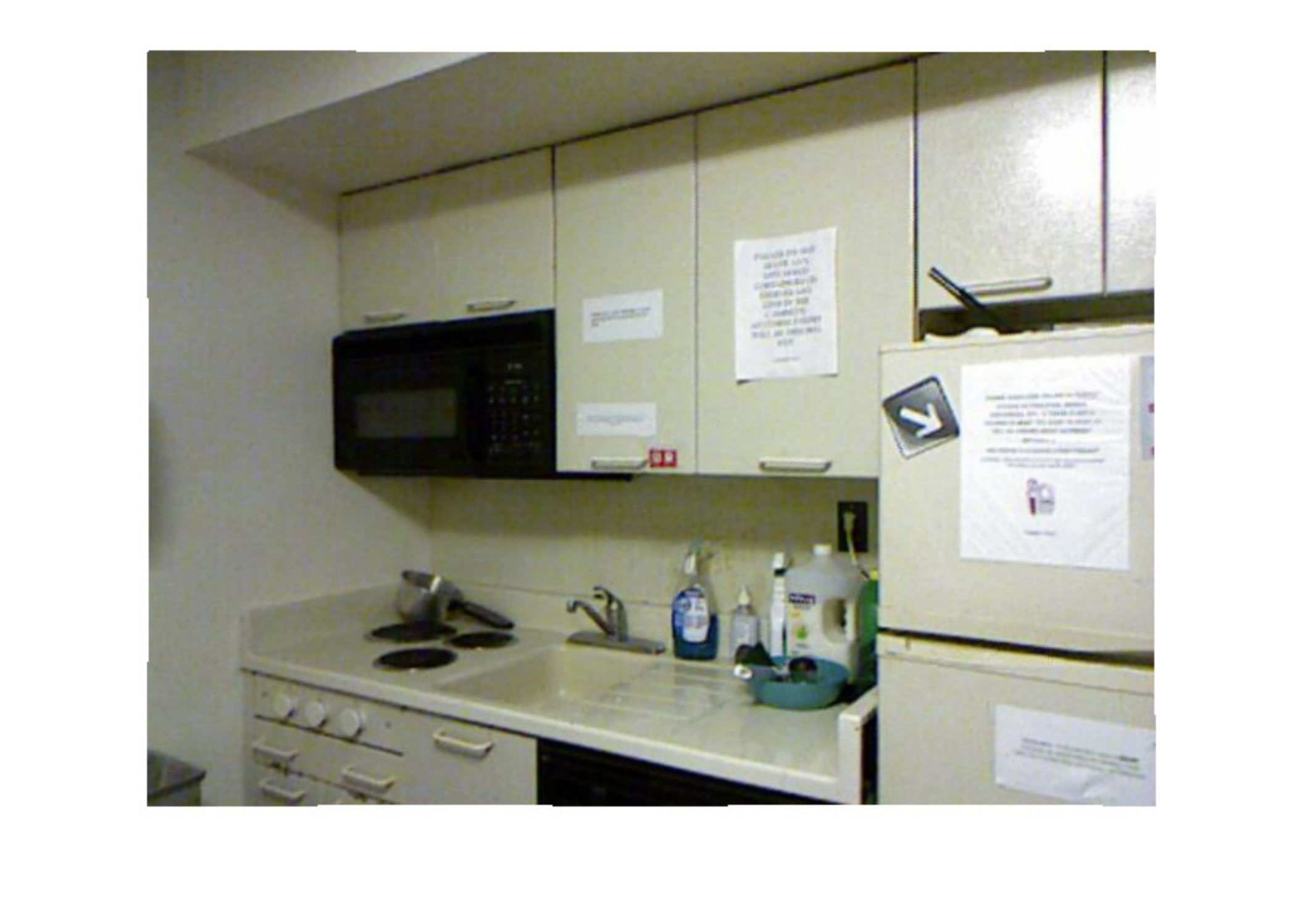}}			\subfigure[annotation]{\label{figure:IMAGE2_AnnotationGC0}\includegraphics[width=0.3\linewidth]{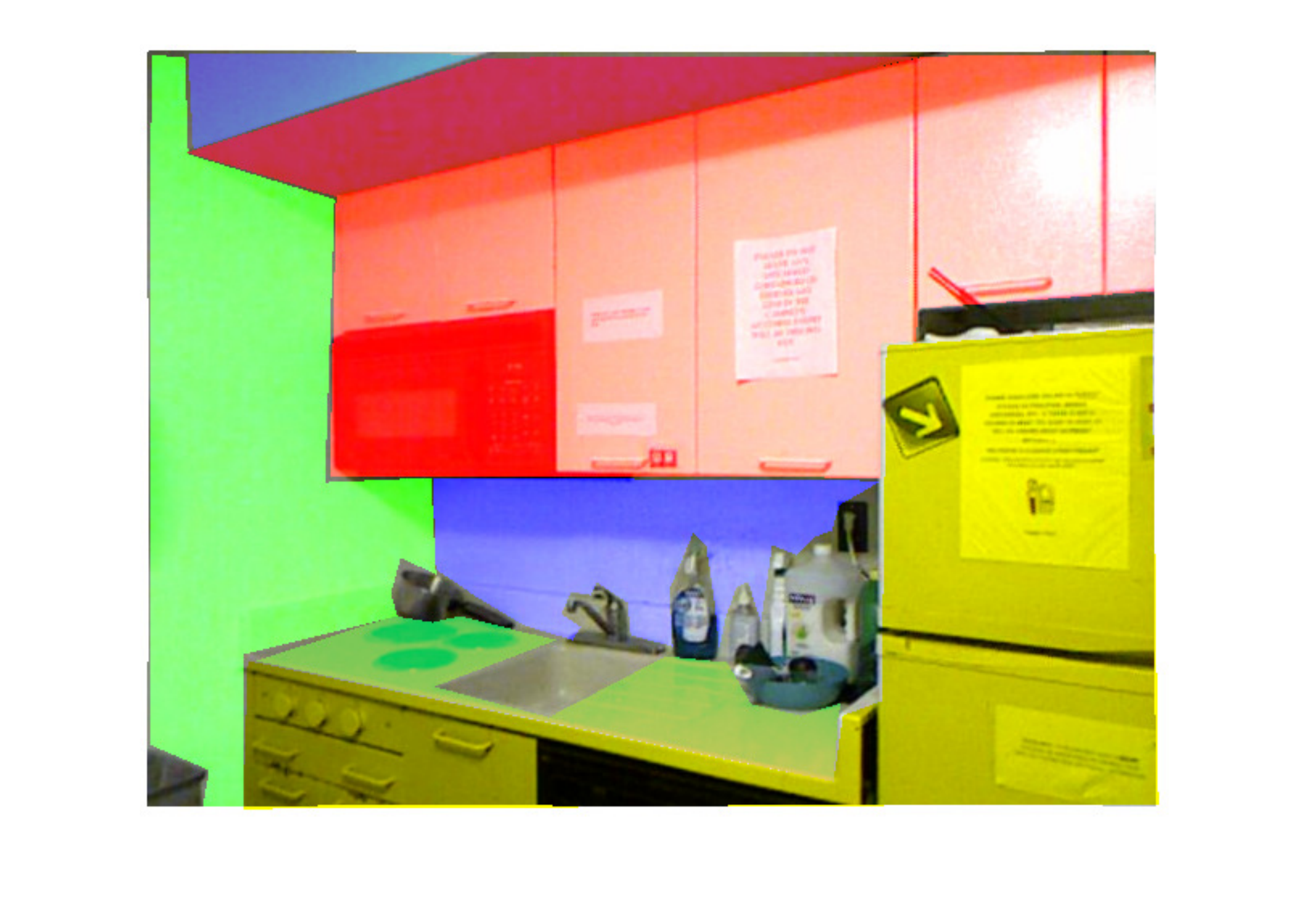}}	\subfigure[IHL-ASC-RANSAC ($23.6\%$)]{\label{figure:I2_ASC_RANSAC_GC0}\includegraphics[width=0.3\linewidth]{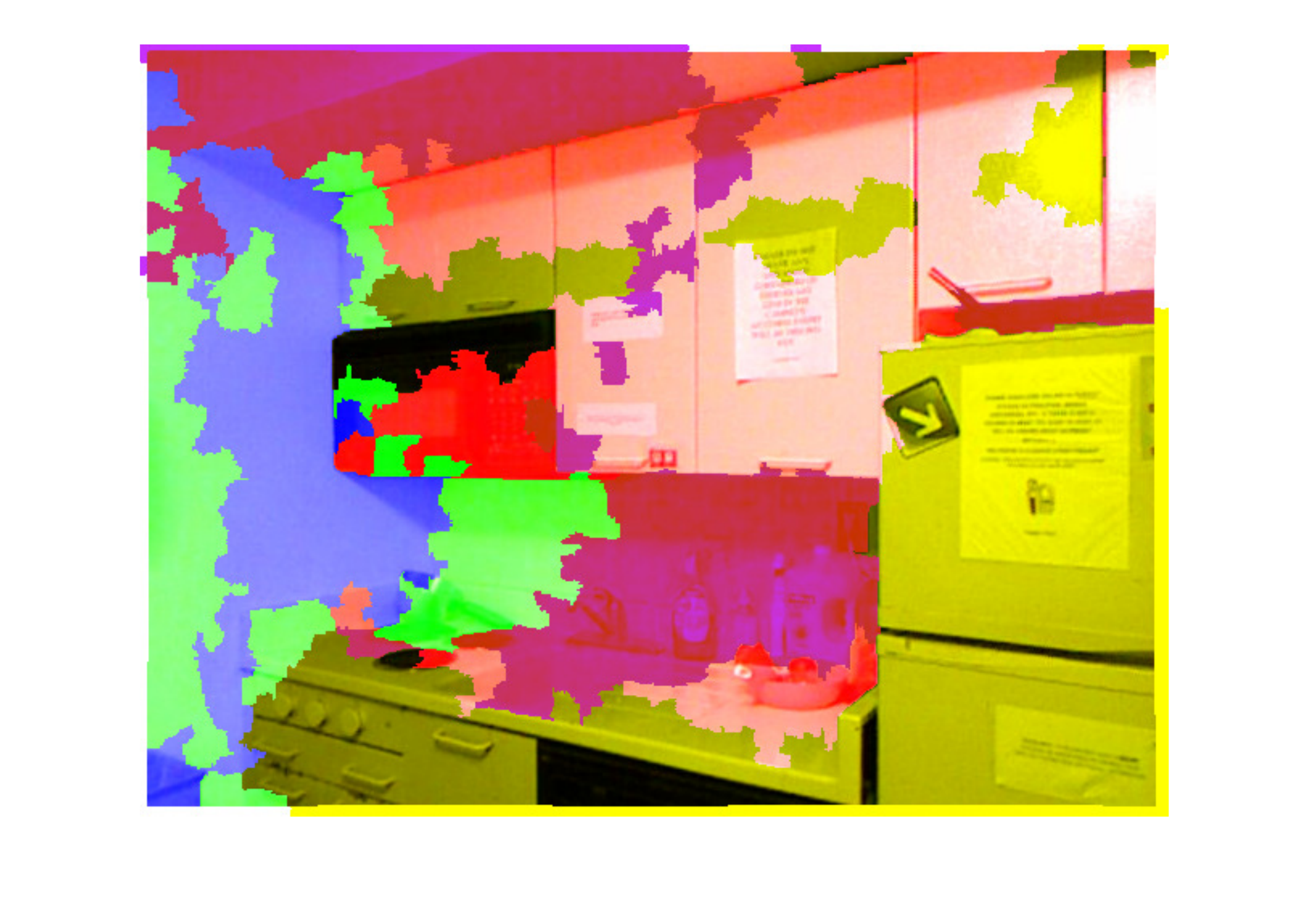}}\subfigure[SHL-RANSAC ($7.9\%$)]{\label{figure:I2_RANSAC_RANSAC_GC0}\includegraphics[width=0.3\linewidth]{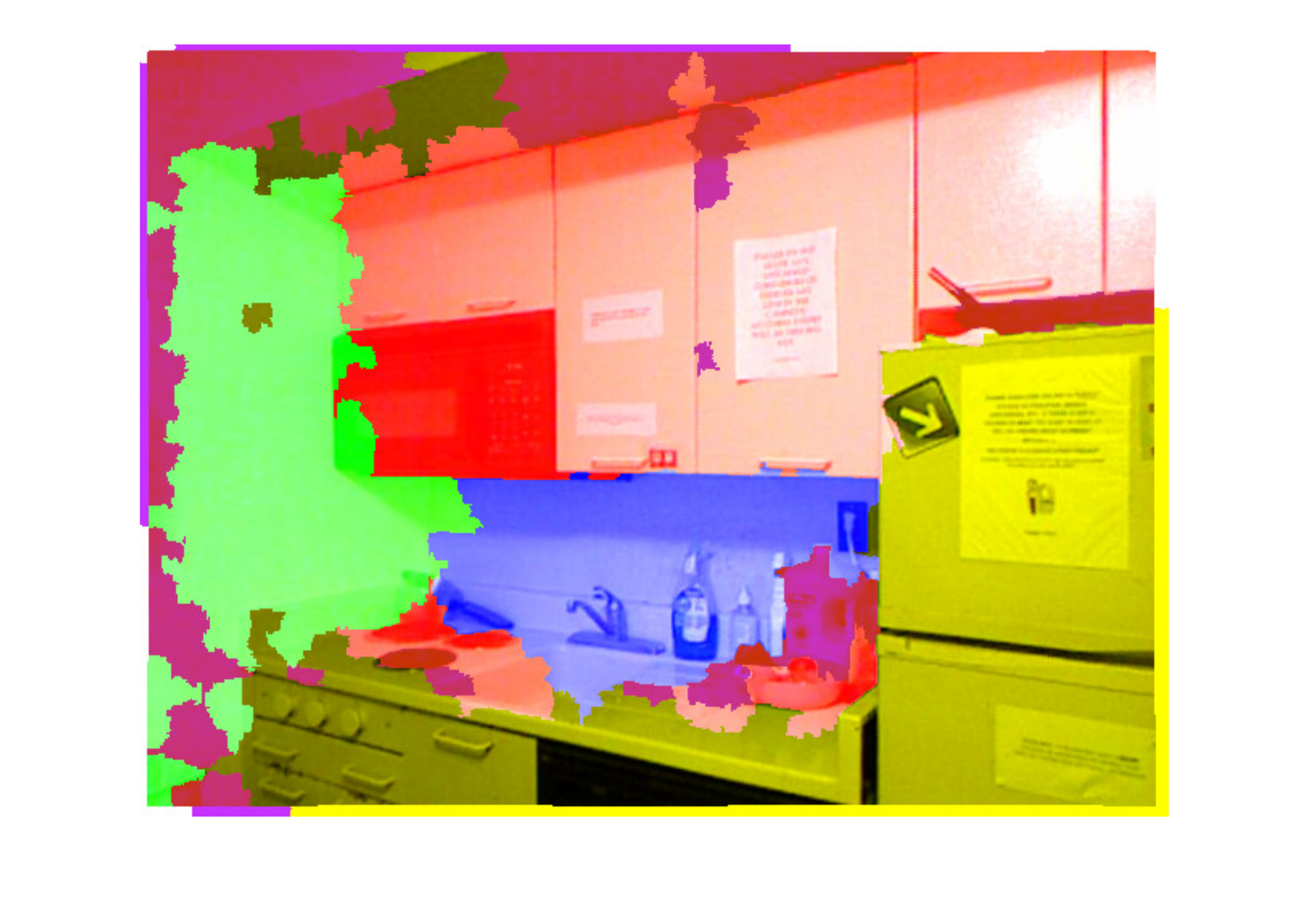}}
	\subfigure[IHL(1)-RANSAC ($8.6\%$)]{\label{figure:I2_KH_RANSAC_GC0}\includegraphics[width=0.3\linewidth]{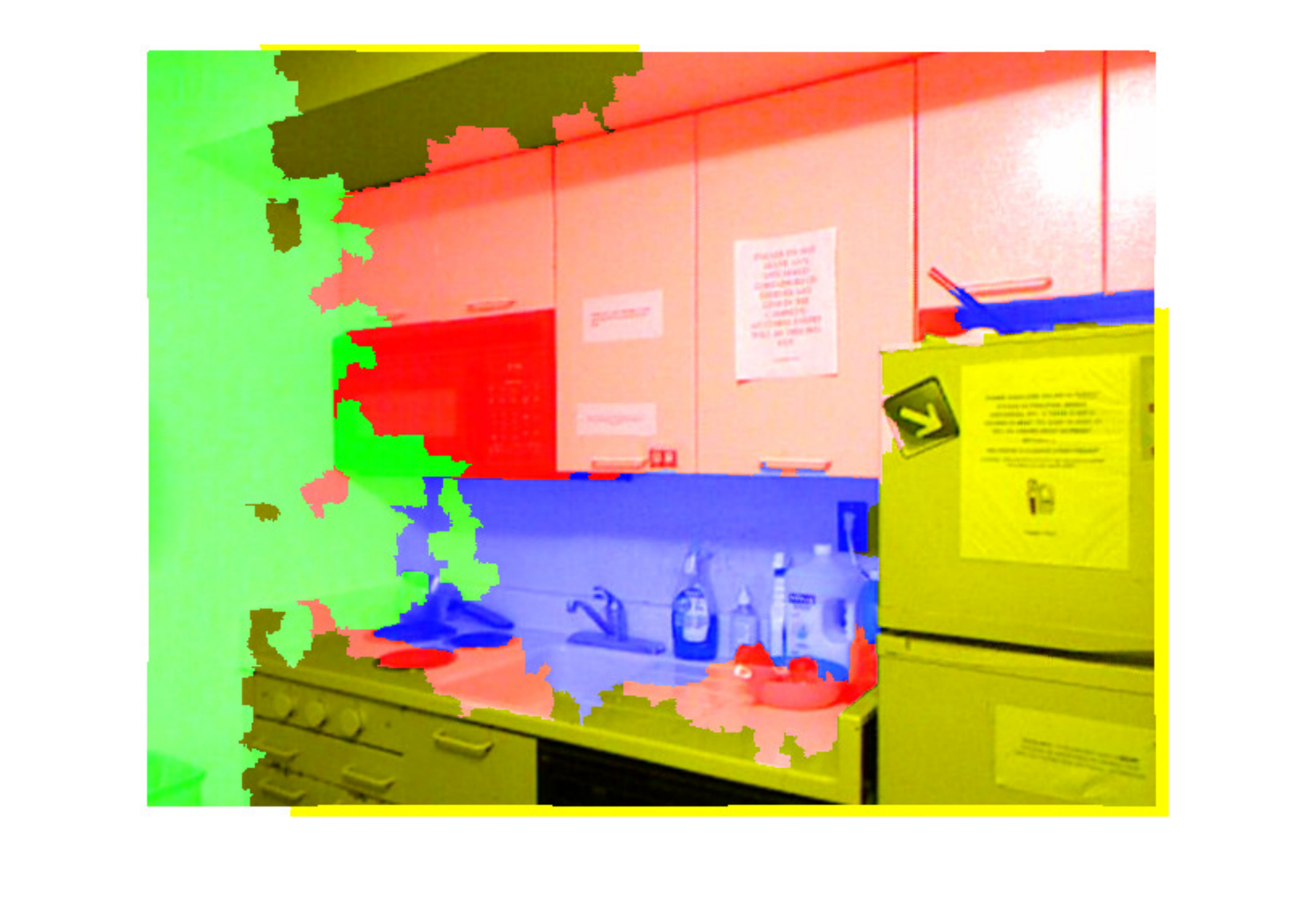}}	
	\subfigure[IHL(1)-DPCP-r-d ($12.24\%$)]{\label{figure:I2_KH_DPCP_GC0}\includegraphics[width=0.3\linewidth]{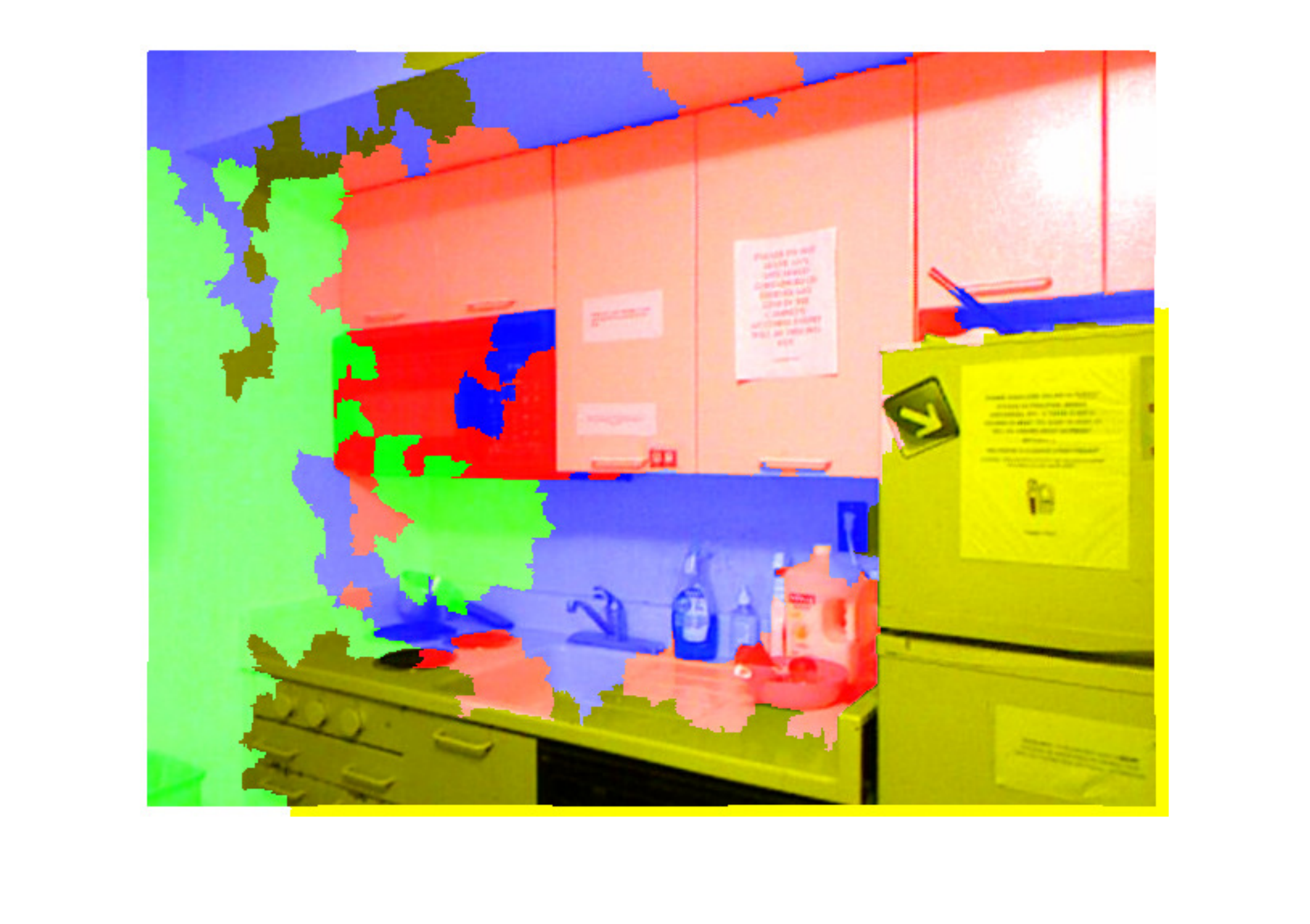}}
		\subfigure[IHL(1)-DPCP-d ($12.73\%$)]{\label{figure:I2_KH_ADM_GC0}\includegraphics[width=0.3\linewidth]{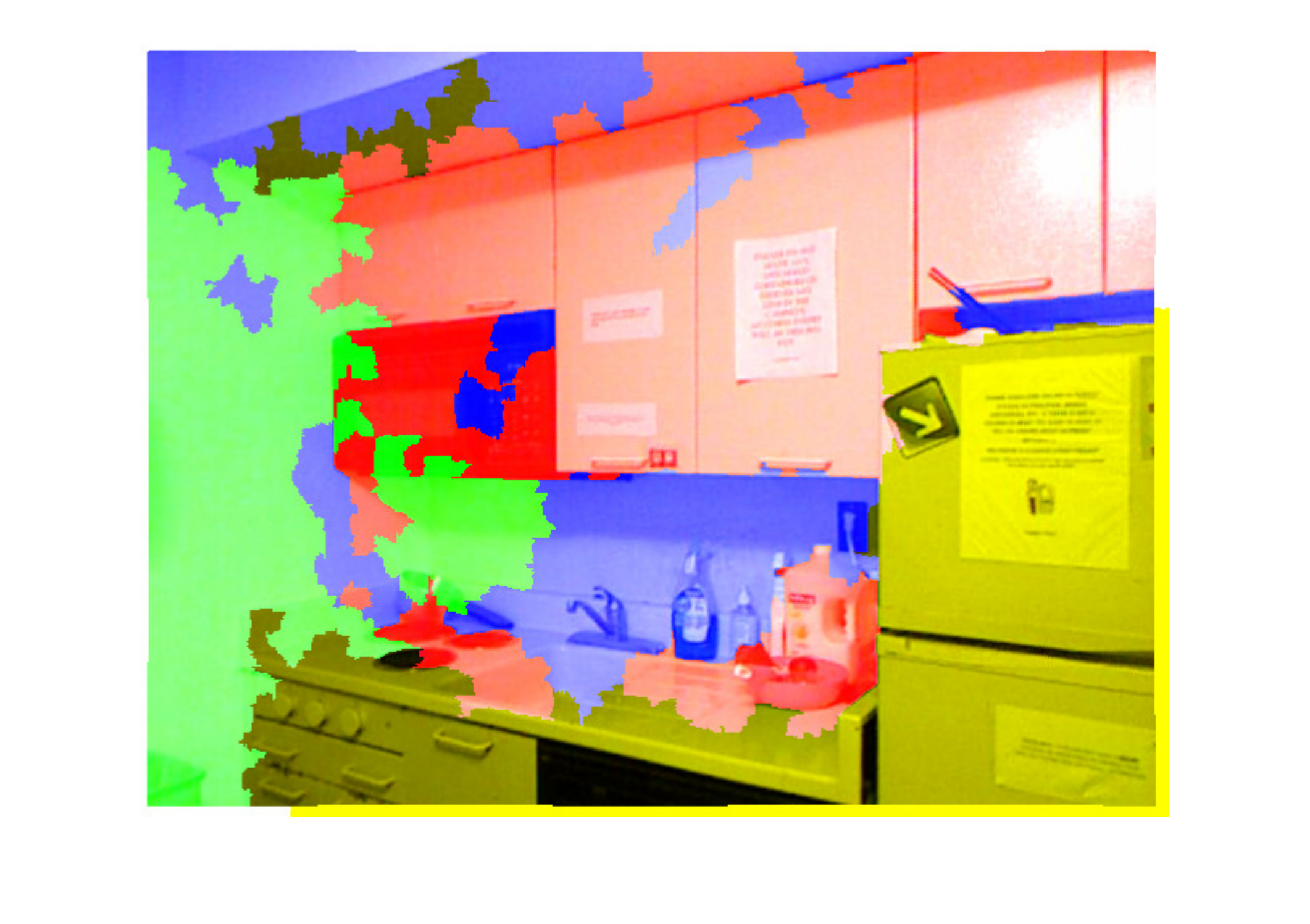}}		
	\subfigure[IHL(2)-SVD ($43.0\%$)]{\label{figure:I2_KH_SVD_GC0}\includegraphics[width=0.3\linewidth]{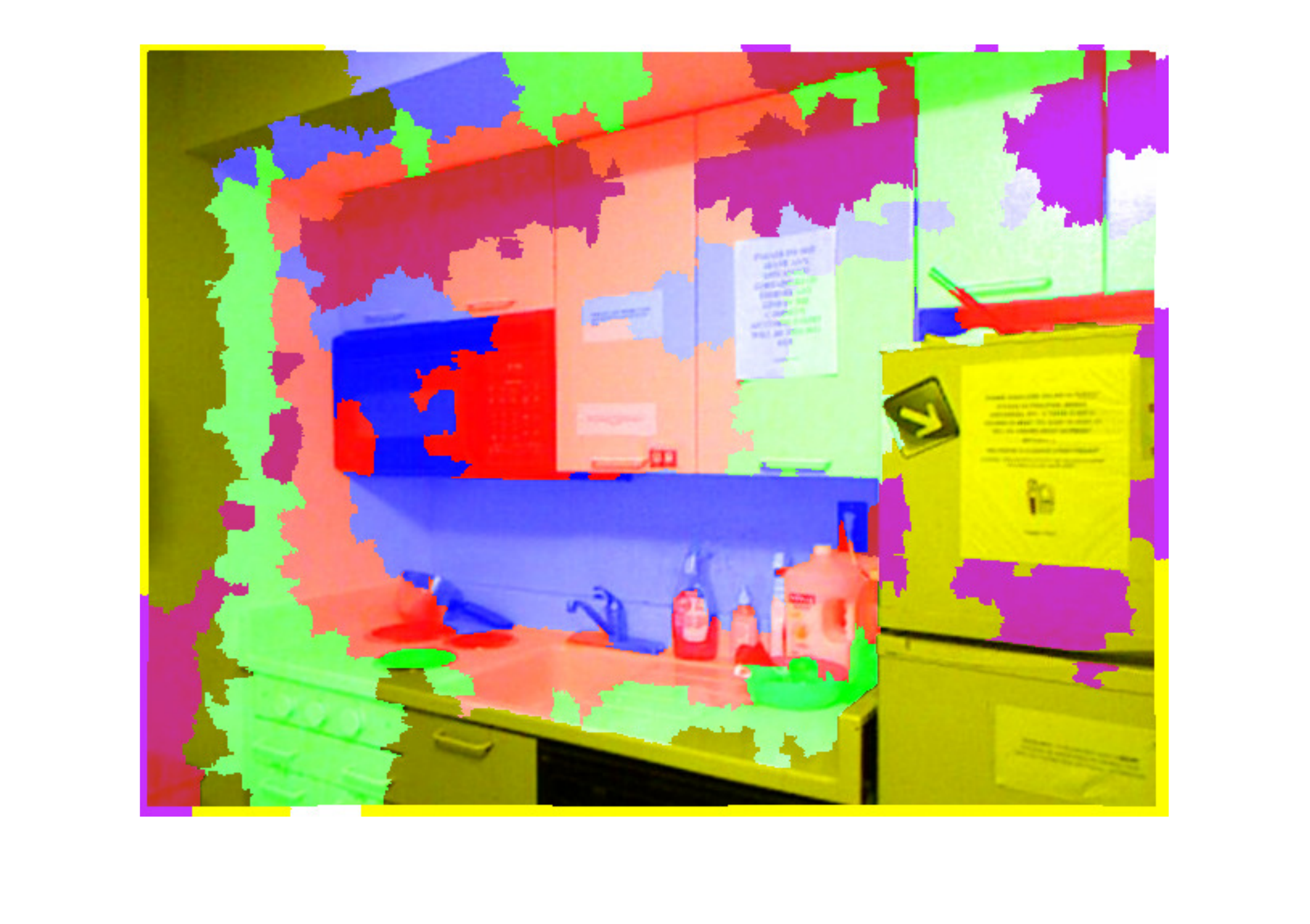}}	
	\subfigure[IHL(1)-REAPER ($22.82\%$)]{\label{figure:I2_KH_REAPER_GC0}\includegraphics[width=0.3\linewidth]{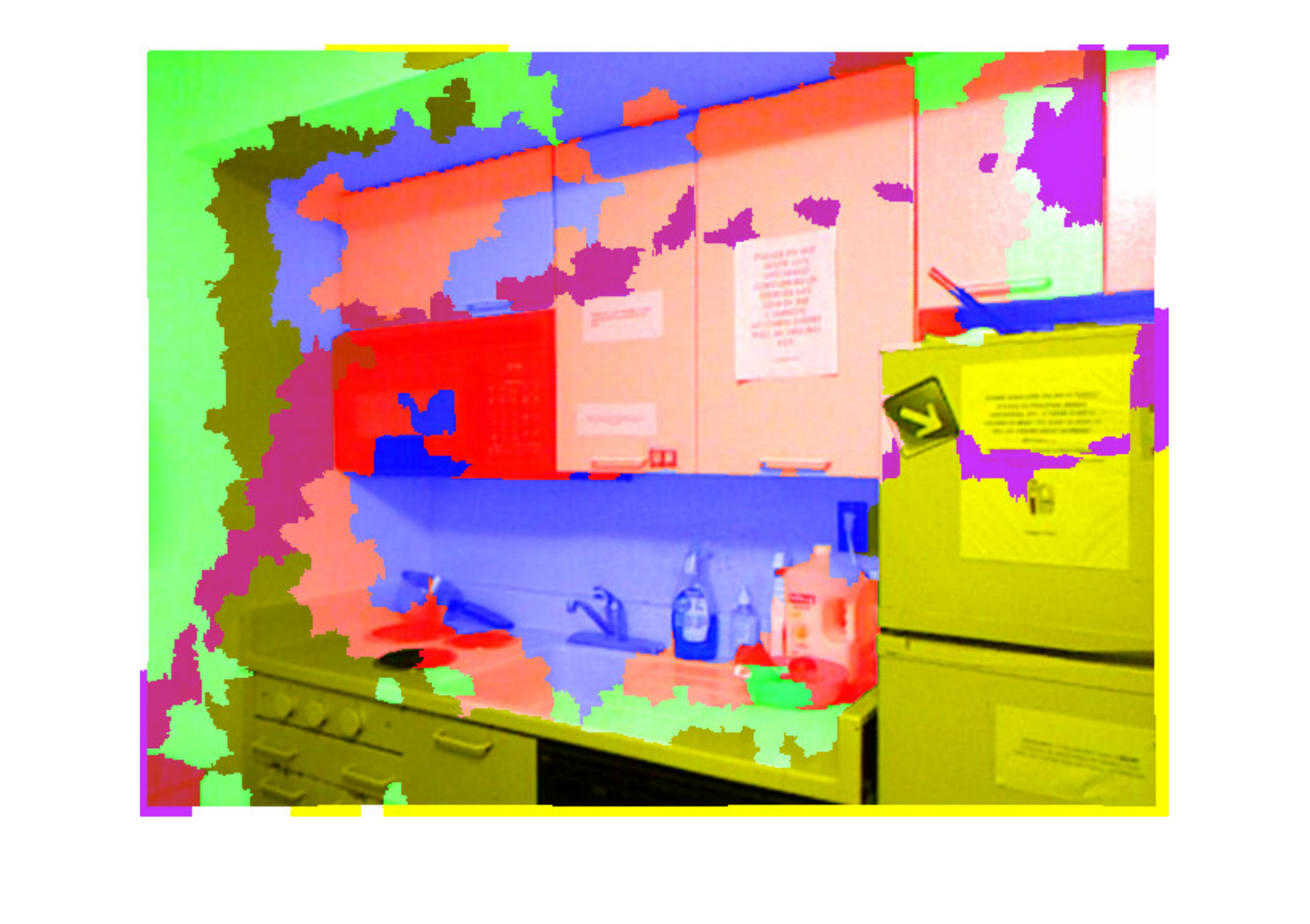}}
	\subfigure[IHL(1)-DPCP-IRLS ($13.9\%$)]{\label{figure:I2_KH_L12_IRLS_GC0}\includegraphics[width=0.3\linewidth]{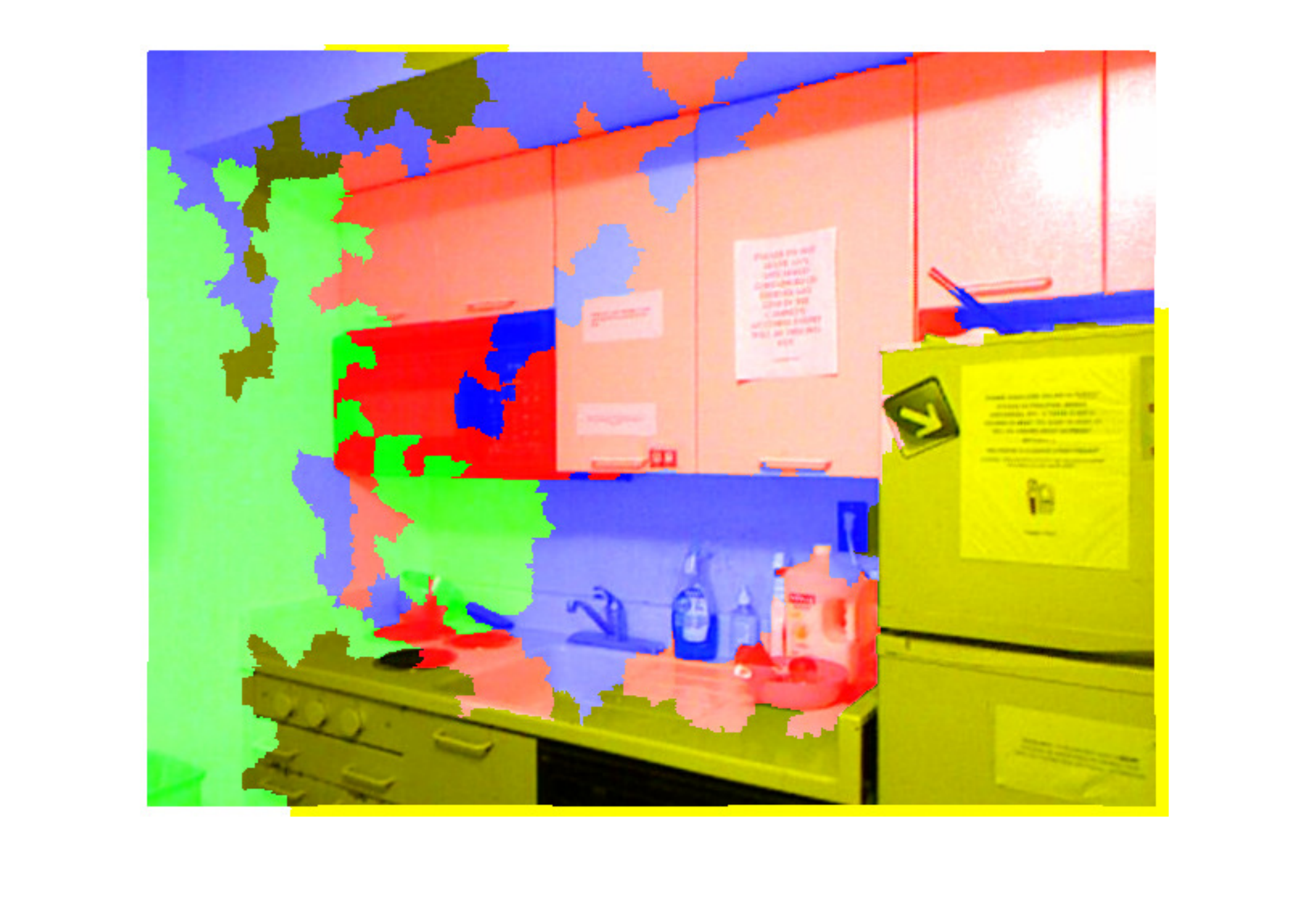}}			
	\caption{Segmentation into planes of image $2$ in dataset NYUdepthV2 without spatial smoothing. Numbers are segmentation errors.}\label{figure:IMAGE2-GC0}
\end{figure}

\begin{figure}[t!]
	\centering
	\subfigure[original image]{\label{figure:OriginalImage}\includegraphics[width=0.3\linewidth]{IMAGE2-eps-converted-to.pdf}}			\subfigure[annotation]{\label{figure:IMAGE2_Annotation}\includegraphics[width=0.3\linewidth]{IMAGE2_Annotation-eps-converted-to.pdf}}\\	\subfigure[IHL-ASC-RANSAC ($12.3\%$)]{\label{figure:I2_ASC_RANSAC_GC1}\includegraphics[width=0.3\linewidth]{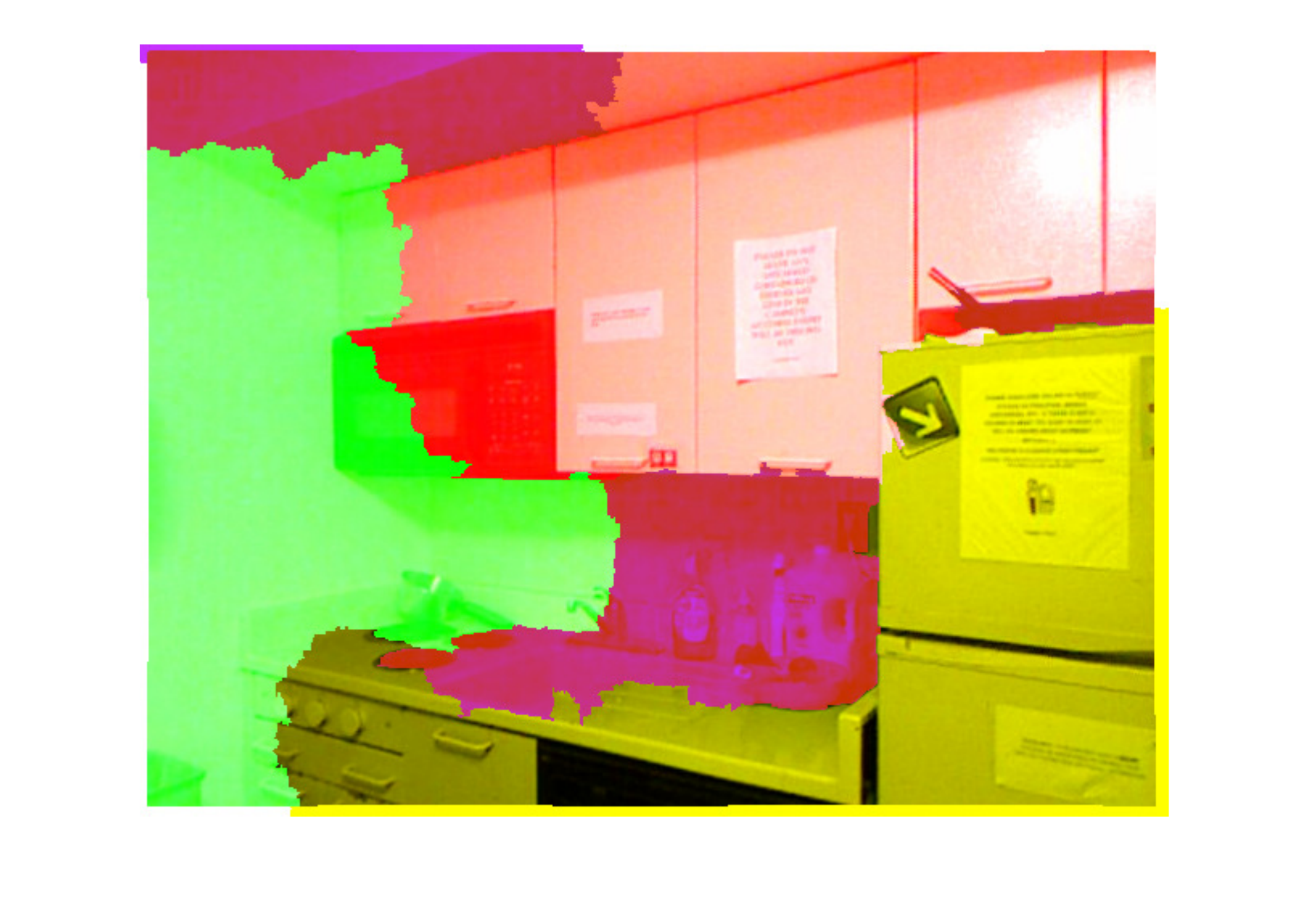}}
	\subfigure[SHL-RANSAC ($5.91\%$)]{\label{figure:I2_RANSAC_RANSAC_GC1}\includegraphics[width=0.3\linewidth]{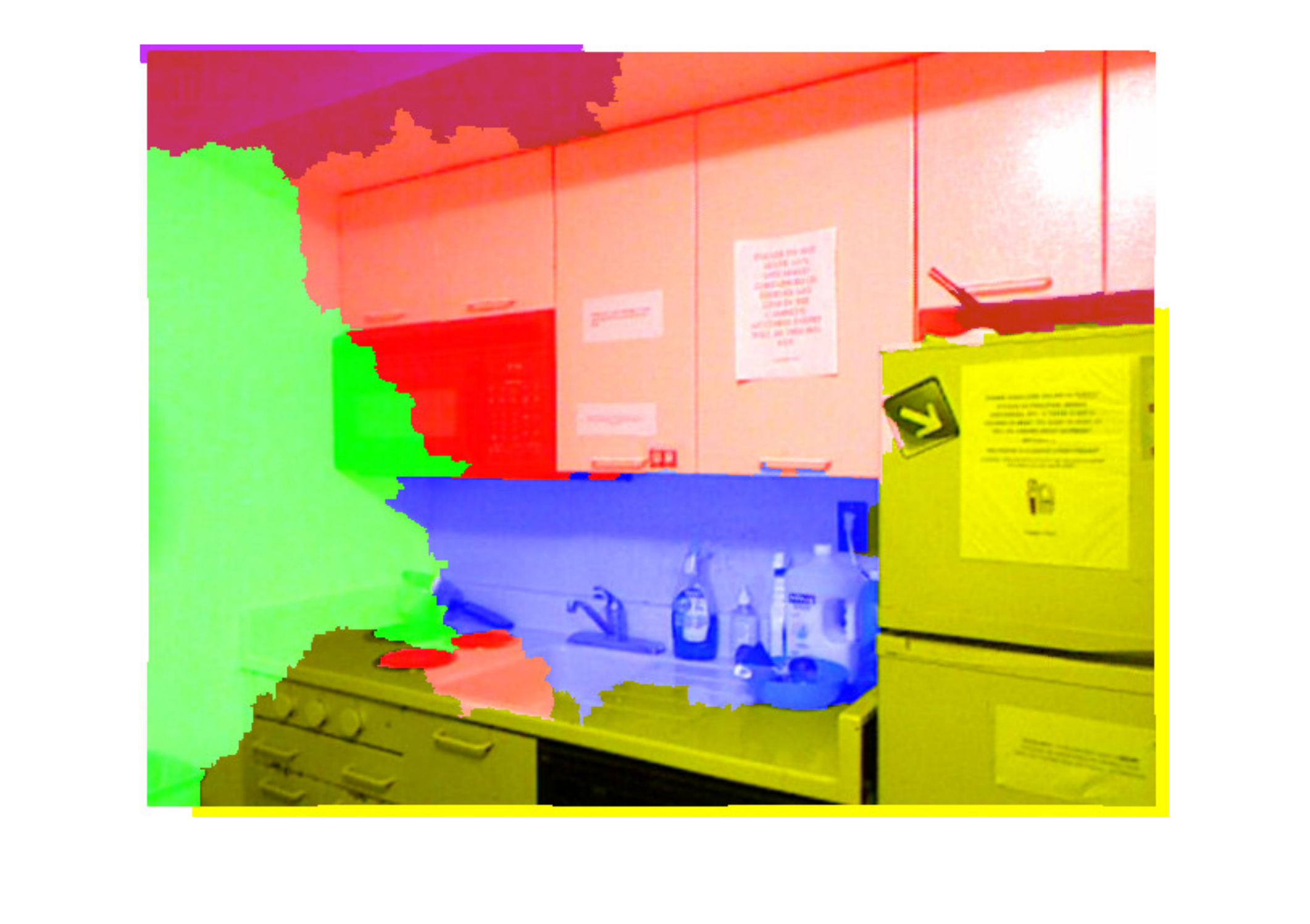}}
	\subfigure[IHL(1)-RANSAC ($9.39\%$)]{\label{figure:I2_KH_RANSAC_GC1}\includegraphics[width=0.3\linewidth]{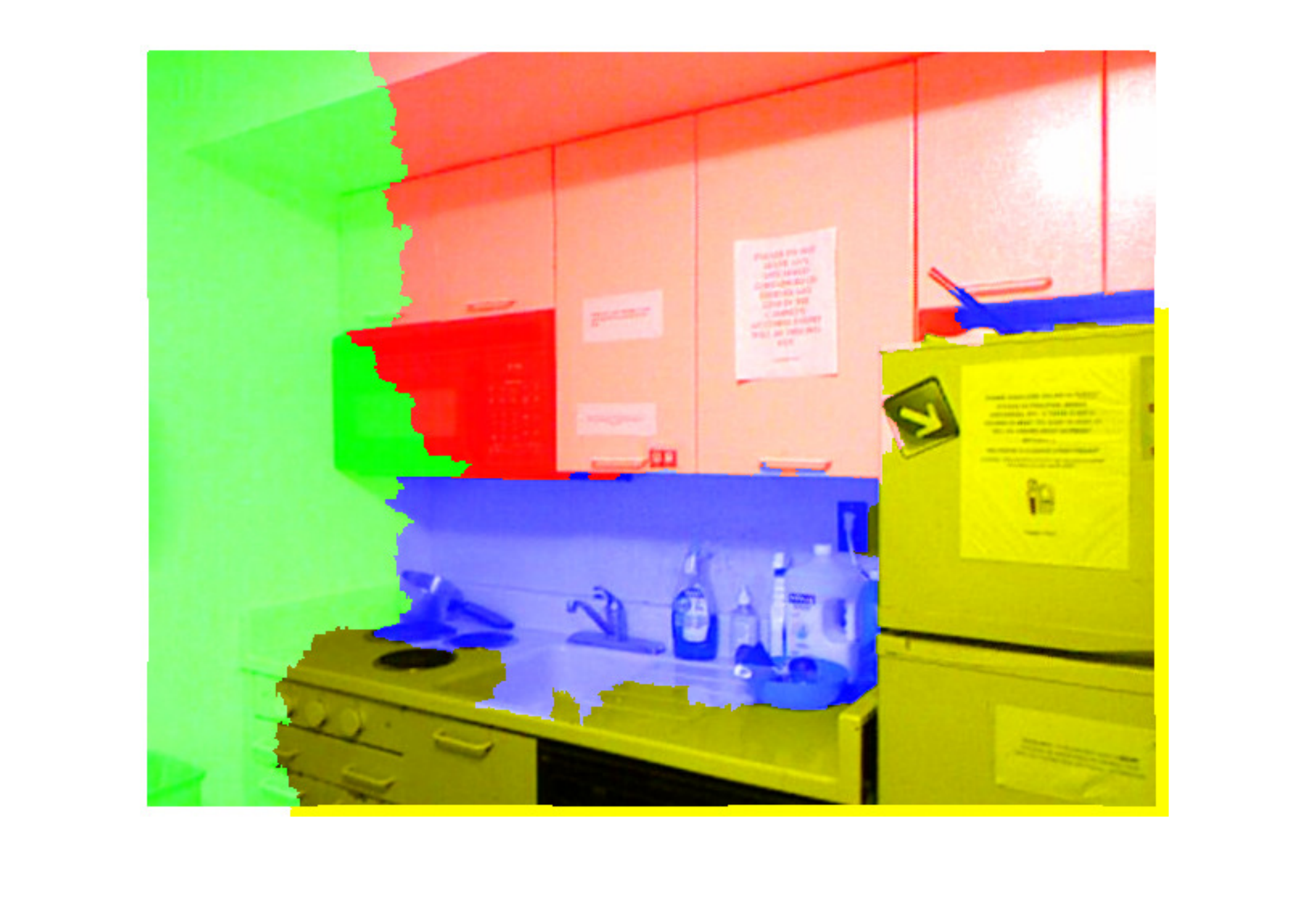}}	
	\subfigure[IHL(1)-DPCP-r-d ($10.05\%$)]{\label{figure:I2_KH_DPCP_GC1}\includegraphics[width=0.3\linewidth]{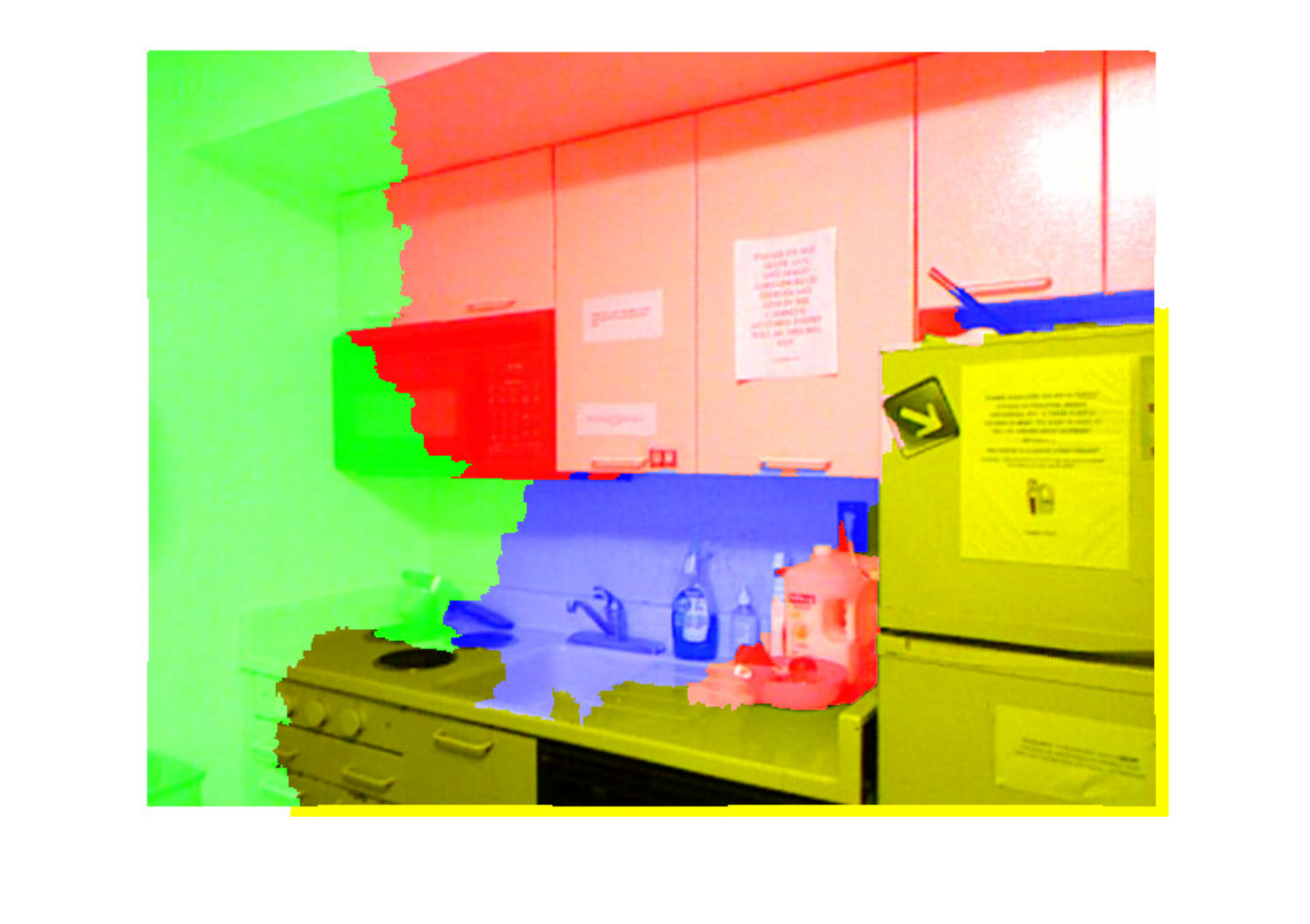}}	
	\subfigure[IHL(1)-DPCP-d ($10.05\%$)]{\label{figure:I2_KH_ADM_GC1}\includegraphics[width=0.3\linewidth]{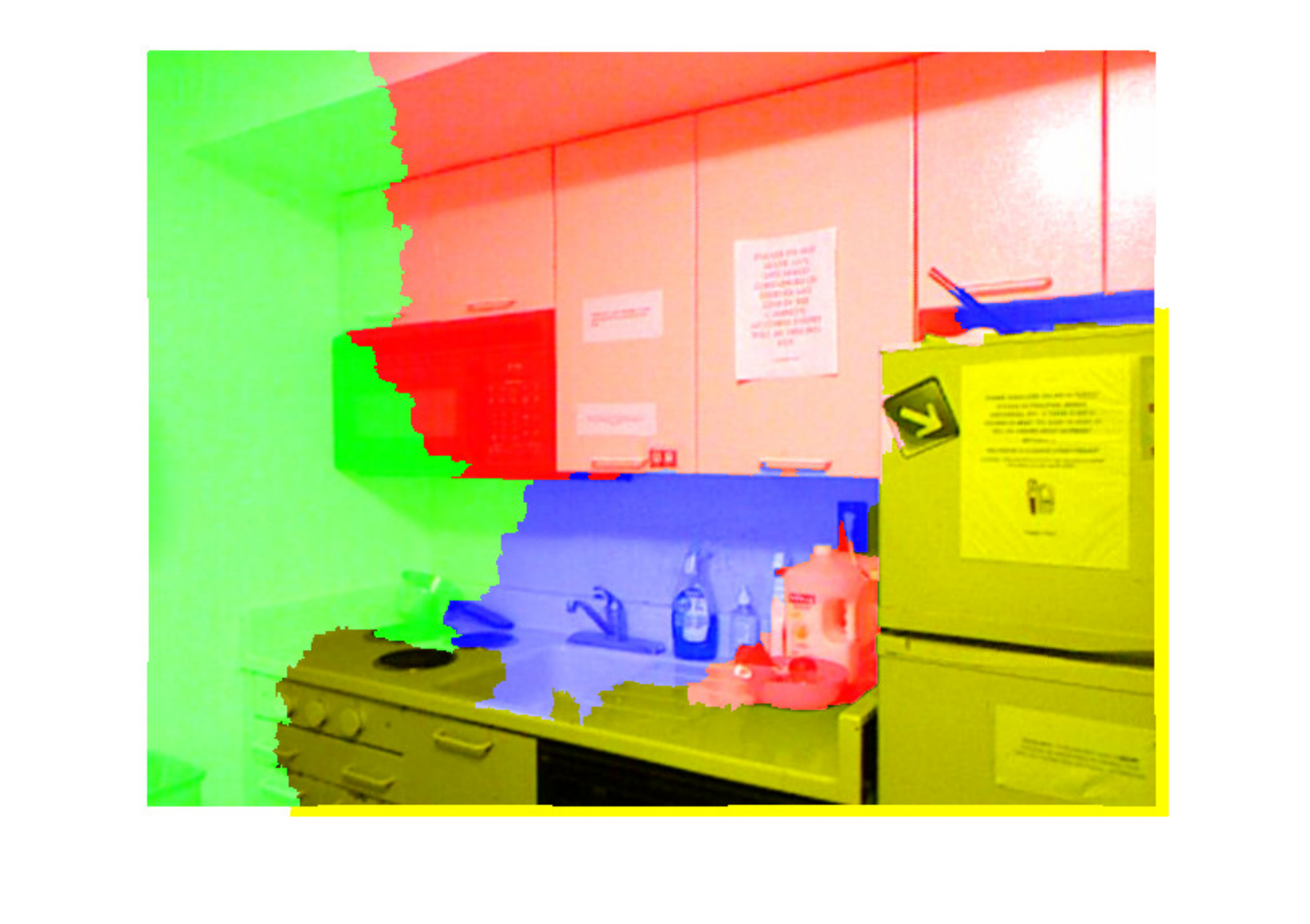}}		
	\subfigure[IHL(2)-SVD ($32.37\%$)]{\label{figure:I2_KH_SVD_GC1}\includegraphics[width=0.3\linewidth]{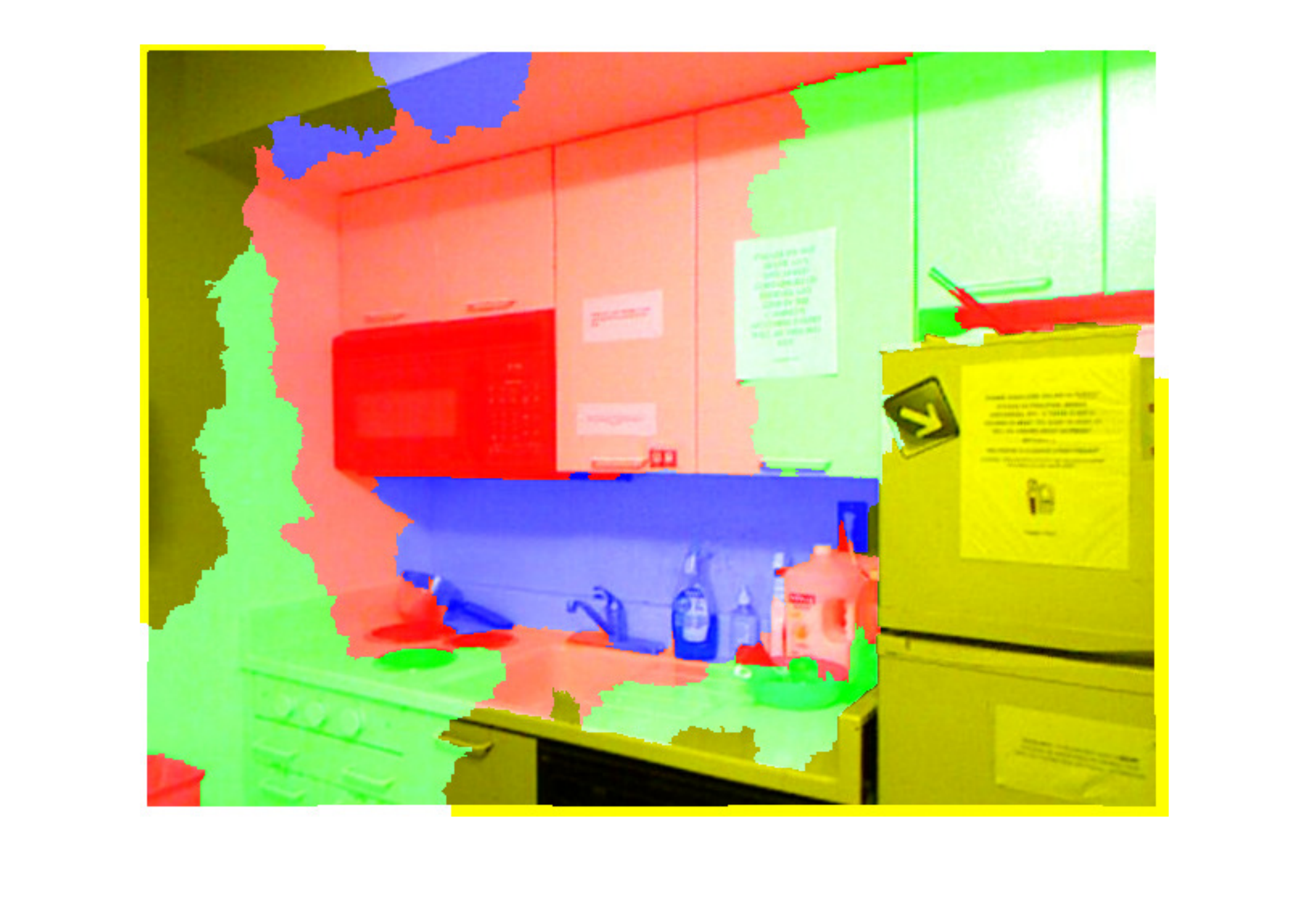}}	
	\subfigure[IHL(1)-REAPER ($13.70\%$)]{\label{figure:I2_KH_REAPER_GC1}\includegraphics[width=0.3\linewidth]{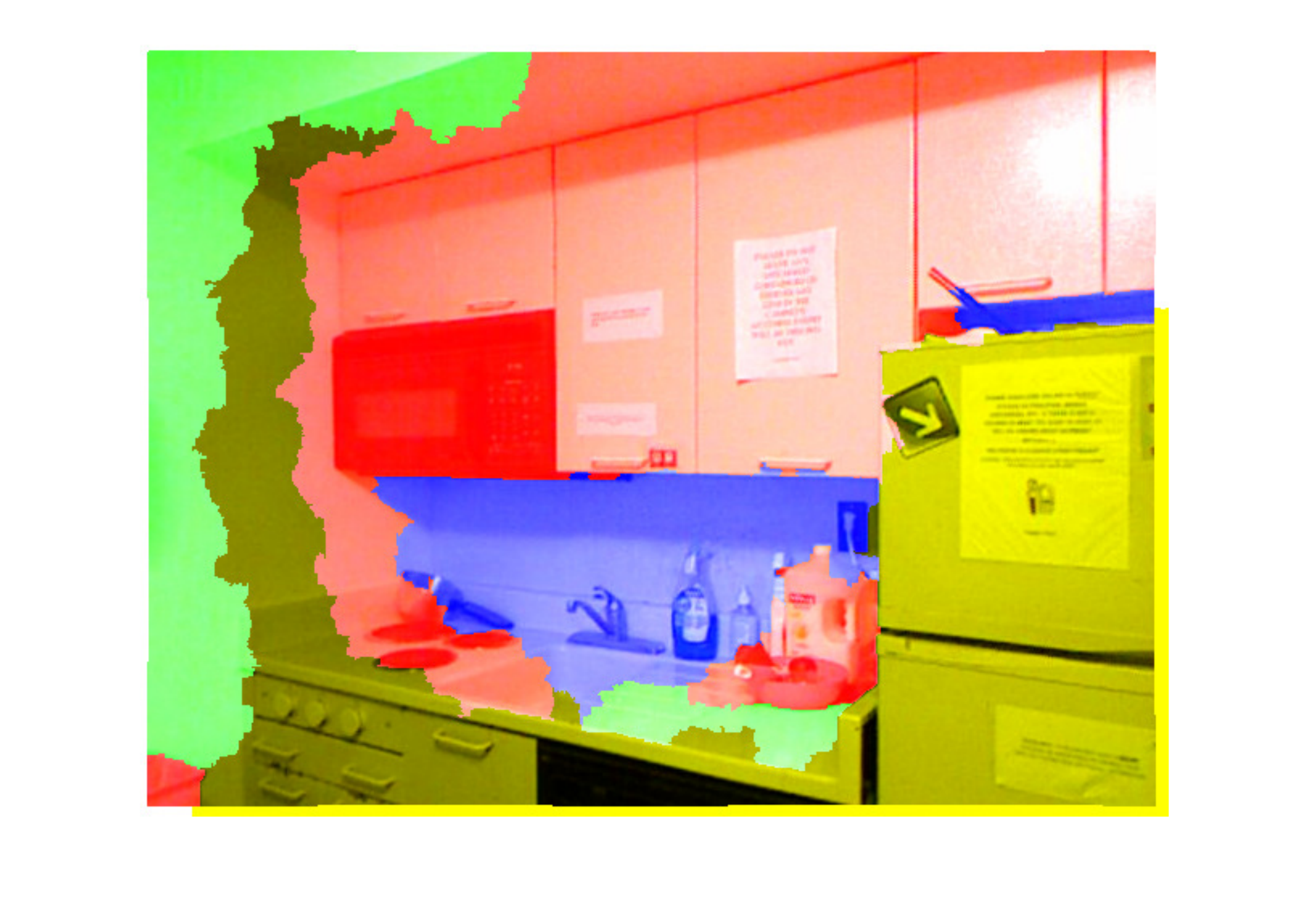}}
	\subfigure[IHL(1)-DPCP-IRLS ($10.0\%$)]{\label{figure:I2_KH_L12_IRLS_GC1}\includegraphics[width=0.3\linewidth]{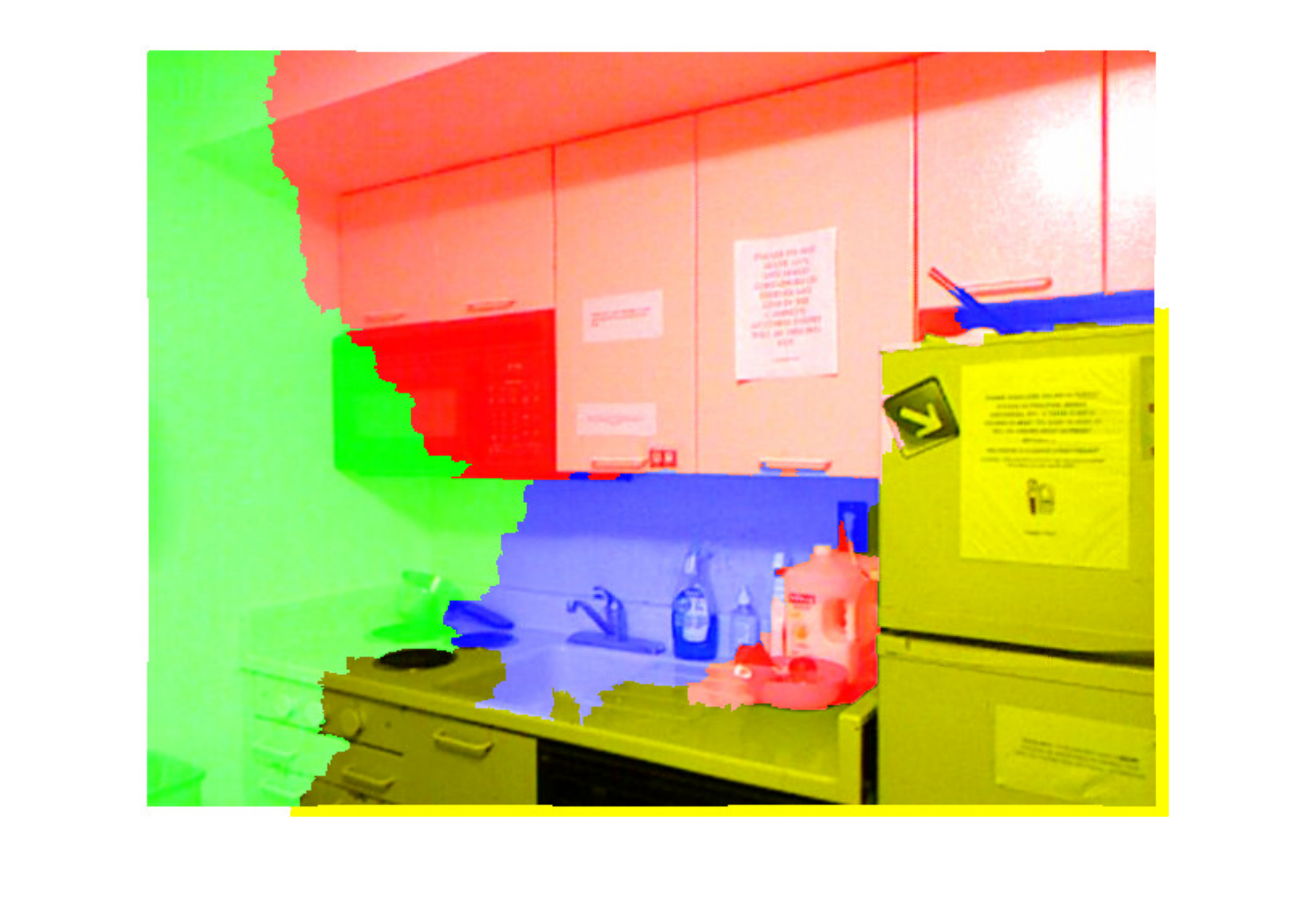}}			
	\caption{Segmentation into planes of image $2$ in dataset NYUdepthV2 with spatial smoothing. Numbers are segmentation errors.}\label{figure:IMAGE2-GC1}
\end{figure} 

\indent \myparagraph{Post-processing} The algorithms described above, are generic hyperplane clustering 
algorithms. On the other hand, we know that nearby points in a $3D$ point cloud have a high 
chance of lying in the same plane, simply because indoor scenes are spatially coherent. Thus
to associate a spatially smooth image segmentation to each algorithm, we use the normal vectors $\b_1,\dots,\b_n$
that the algorithm produced to minimize a Conditional-Random-Field 
\citep{SuttonMcCallum:MITPress06} type of energy function, given by 
\begin{align}
E(y_1,\dots,y_N) := \sum_{j=1}^N d(\b_{y_j},\x_j) + \lambda \sum_{k \in \N_j} w(\x_j,\x_k) \delta (y_j \neq y_k).
\label{eq:CRF}
\end{align} In \eqref{eq:CRF} $y_j \in \left\{1,\dots,n \right\}$ is the plane label of point $\x_j$, 
 $d(\b_{\y_j},\x_j)$ is a unary term that measures the cost of assigning 3D point $\x_j$ to the plane with normal $\b_{\y_j}$, $w(\x_j,\x_k) $ is a pairwise term that measures the similarity between points $\x_j$ and $\x_k$, $\lambda > 0$ is a chosen parameter, $\N_j$ indexes the neighbors of $\x_j$, and $\delta(\cdot)$ is the indicator function. The unary term is defined as $d(\b_{y_j},\x_j) = |\b_{y_j}^\transpose \x_j |$, which is the Euclidean distance from point $\x_j$ to the plane with normal $\b_{y_j}$, and the pairwise term is defined as
\begin{align}
w(\x_j,\x_k) := \text{CB}_{j,k} \, \exp\left({-\frac{\|\x_j - \x_k \|_2^2}{2 \sigma_d^2}}\right),
 \end{align} where $\left\|\x_j - \x_k \right\|_2$ is the Euclidean distance from $\x_j$ to $\x_k$, and CB$_{j,k}$ is the length of the common boundary between superpixels $j$ and $k$. The minimization of the energy function is done via Graph-Cuts \citep{Boykov:PAMI01}.
 
\indent \myparagraph{Parameters} 
For the thresholding parameter of SHL-RANSAC, denoted by $\tau$,
we test the values $0.1,0.01,0.001$. For the parameter $\tau$ 
of IHL(1)-DPCP-d and IHL(1)-DPCP-r-d we test the values $0.1,0.01,0.001$. We also use the same values for the thresholding parameter of SHL-RANSAC, which we also denote by $\tau$. The rest of the parameters of DPCP and REAPER are set as in Section \ref{subsection:ExperimentsSynthetic}.
The convergence accuracy of the IHL algorithms is set to $10^{-3}$.
Moreover, the IHL algorithms are configured to allow for a maximal number
of $10$ random restarts and $100$ iterations per restart, but the overall running time of each IHL algorithm should not exceed $5$ seconds; this latter constraint is also enforced to  SHL-RANSAC and IHL-ASC-RANSAC.

The parameter $\sigma_d$ in \eqref{eq:CRF} is set to the mean distance between $3D$ points representing neighboring superpixels. The parameter $\lambda$ in \eqref{eq:CRF} is set to the inverse of twice the maximal 
row-sum of the pairwise matrix $\left\{w(\x_j,\x_k) \right\}$; this is to achieve a balance between unary and pairwise terms. 

\indent \myparagraph{Evaluation} Recall that none of the algorithms considered in this section is explicitly configured to detect outliers, rather it assigns each and every point to some plane. 
Thus we compute the clustering error as follows. First, we restrict the output labels of each algorithm to the indices of the dominant ground-truth cluster, and measure how far are these restricted labels from being identical (identical labels would signify that the algorithm identified perfectly well the plane); this is done by computing the ratio of the restricted labels that are different from the dominant label. Then the dominant
label is \emph{disabled} and a similar error is computed for the second dominant ground-truth plane, and so on. Finally the clustering error is taken to be the weighted sum of the errors associated with each dominant plane, with the weights proportional to the size of the ground-truth cluster. 

We evaluate the algorithms in several different settings. First, we test how well the algorithms can cluster
the data into the first $n$ dominant planes, where $n$ is $2,4$ or equal to the total number of annotated
planes for each scene. Second, we report the clustering error before spatial smoothing, i.e., without refining
the clustering by minimizing \eqref{eq:CRF}, and after spatial smoothing. The former case is denoted by
GC(0), indicating that no graph-cuts takes place, while the latter is indicated by GC(1). Finally,
to account for the randomness in RANSAC as well as the random initialization of IHL, we average the clustering errors over $10$ independent experiments.

\indent \myparagraph{Results}  The results are reported in Table \ref{table:NYU}, where the clustering error of 
the methods that depend on $\tau$ is shown for each value of $\tau$ individually, as well as 
averaged over all three values. 

Notice that spatial smoothing improves the clustering accuracy considerably (GC(0) vs GC(1)); e.g., the clustering error of the traditional IHL(2)-SVD for all ground-truth planes drops from $26.22\%$  to $16.71\%$, when spatial smoothing is employed. Moreover, as it is intuitively expected, the clustering error increases when fitting more planes (larger $n$) is required; e.g., for the GC(1) case, the error of IHL(2)-SVD increases from $9.96\%$ for $n=2$ to $16.71\%$ for all planes ($n \approx 9$).

Next, we note the remarkable insensitivity of the DPCP-based methods IHL(1)-DPCP-d and IHL(1)-DPCP-r-d to variations of the parameter $\tau$. In sharp contrast, SHL-RANSAC is very sensitive to $\tau$;  e.g., for $\tau=0.01$ and $n=2$, SHL-RANSAC is the best method with $6.27\%$, while for $\tau=0.1, 0.001$ its error increases to $16.01\%$ and $15.26\%$ respectively. Interestingly, the hybrid IHL(1)-RANSAC is significantly more robust; in fact, in terms of clustering error it is the best method. On the other hand, 
by looking at the lower part of Table \ref{table:NYU}, we conclude that on average the rest of the methods have very similar behavior.

Figs. \ref{figure:IMAGE5-GC0}-\ref{figure:IMAGE2-GC1} show some segmentation results for two scenes,
with and without spatial smoothing. It is remarkable that, even though the segmentation in
Fig. \ref{figure:IMAGE5-GC0} contains artifacts, which are expected due to the lack of spatial smoothing, its quality is actually very good, in that most of the dominant planes have been correctly identified. Indeed, applying spatial smoothing  (Fig. \ref{figure:IMAGE5-GC1}) further drops the error for most methods only by about $1\%$.


\section{Conclusions}

We studied theoretically and algorithmically the application of the recently proposed single subspace learning
method \emph{Dual Principal Component Pursuit (DPCP)} to the problem of clustering data that lie close to
a union of hyperplanes. We gave theoretical conditions under which the non-convex cosparse problem 
associated with DPCP admits a unique (up to sign) global solution equal to the normal vector of 
the underlying dominant hyperplane. We proposed sequential and parallel hyperplane clustering methods,
which on synthetic data dramatically improved upon state-of-the-art methods such as RANSAC or REAPER, while were competitive to the latter in the case of learning unions of $3D$ planes from real Kinect data. 
Future research directions include analysis in the presence of noise, generalizations to unions of subspaces of arbitrary dimensions, even more scalable algorithms, and applications to deep networks.

\section*{Acknowledgement}
This work was supported by grants NSF 1447822 and NSF 1618637. The authors thank Prof. Daniel P. Robinson of the Applied Mathematics and Statistics department of the Johns Hopkins University for many useful conversations, as well as for his many comments that helped improve this manuscript. 

\appendix
\section{Results on Problems \eqref{eq:ell1} and \eqref{eq:ConvexRelaxations} following \cite{Spath:Numerische87}} \label{appendix:Spath}

In this Section we state three results that are important for our mathematical
analysis, already known in \cite{Spath:Numerische87}; detailed proofs can be
found in \cite{Tsakiris:DPCP-ArXiv17}. Let $\bY$ be a $D \times N$ matrix of full rank $D$. Then we have the following.

\begin{lem} \label{lem:NonConvexMaximalInterpolation}
	Any global solution $\b^*$ to $\min_{\b^\transpose \b=1} \left\|\bY^\transpose \b \right\|_1$, must be orthogonal to $(D-1)$ linearly independent points of $\bY$.
\end{lem}

\begin{lem} \label{lem:LPMaximalInterpolation}
	Problem $\min_{\b^\transpose \hat{\bn}_k=1} \left\|\bY^\transpose \b \right\|_1$ admits a computable solution $\bn_{k+1}$ that is orthogonal to $(D-1)$ linearly independent points of $\bY$. 
\end{lem}

\begin{lem} \label{lem:LPconvergence}
	Suppose that for each problem $\min_{\b^\transpose \hat{\bn}_k=1} \left\|\bY^\transpose \b \right\|_1$, a solution $\bn_{k+1}$ is chosen such that $\bn_{k+1}$ is orthogonal to $D-1$ linearly independent points of $\bY$, in accordance with Lemma \ref{lem:LPMaximalInterpolation}. Then the sequence $\left\{\bn_k\right\}$ converges to a critical point of problem $\min_{\b^\transpose \b=1} \left\|\bY^\transpose \b \right\|_1$ in a finite number of steps.
\end{lem}


\bibliography{DPCP-HC-ArXiv17.bbl}


\end{document}